%% file: contrastive_od.tex
\documentclass{article}

\usepackage{microtype}
\usepackage{graphicx}
\usepackage{wrapfig}
\usepackage{subcaption}
\usepackage{booktabs} 

\usepackage{hyperref}
\usepackage{mathtools}
\usepackage{cleveref}
\usepackage{amsfonts}
\usepackage{amsthm}
\usepackage{gensymb}

\usepackage{amsmath,amssymb}
\usepackage{glossaries}
\usepackage{comment}
\usepackage{enumitem}
\usepackage{breqn}
\usepackage{dsfont}
\usepackage[percent]{overpic}

\DeclareMathOperator*{\argmin}{arg\,min}
\newcommand{\parhead}[1]{{\bf{#1}}\hspace{3pt}}
\newcommand{\at}[2][]{#1|_{#2}}
\newacronym{ntl}{NeuTraL AD}{neural transformation learning for anomaly detection}
\newacronym{ntlC}{NeuTraL AD}{Neural Transformation Learning for Anomaly Detection}
\newacronym{dcl}{DCL}{deterministic contrastive loss}
\newacronym{ocsvm}{OC-SVM}{one-class SVM}
\newacronym{if}{IF}{isolation forest}
\newacronym{lof}{LOF}{local outlier factor}
\newacronym{svdd}{Deep SVDD}{}
\newacronym{dagmm}{DAGMM}{}
\newacronym{goad}{GOAD}{}
\newacronym{drocc}{DROCC}{}
\newacronym{mi}{MI}{mutual information}
\newacronym{rnn}{RNN}{RNN-based model}
\newacronym{lstm}{LSTM-ED}{}
\newacronym{sad}{SAD}{spoken Arabic digits}
\newacronym{natops}{NATOPS}{naval air training and operating procedures standardization}
\newacronym{ct}{CT}{character trajectories}
\newacronym{rs}{RS}{Racket sports}
\newacronym{epilepsy}{EPSY}{Epilepsy}
\glsunset{goad}
\glsunset{lstm}
\glsunset{svdd}
\glsunset{dagmm}
\glsunset{drocc}
\newtheorem{claim}{Proposition}
\newtheorem{prop}{Proposition}
\newtheorem{prt}{Proposition 3, Part}
\newtheorem{requirement}{Requirement}
\crefname{requirement}{Requirement}{Requirements}
\crefname{claim}{Proposition}{Propositions}

\usepackage[accepted]{icml2021}

\icmltitlerunning{\acrfull{ntlC}}

\begin{document}

\twocolumn[
\icmltitle{Neural Transformation Learning for Deep Anomaly Detection Beyond Images}



\icmlsetsymbol{equal}{*}

\begin{icmlauthorlist}
\icmlauthor{Chen Qiu}{bcai,tuk}
\icmlauthor{Timo Pfrommer}{bcai}
\icmlauthor{Marius Kloft}{tuk}
\icmlauthor{Stephan Mandt}{uci}
\icmlauthor{Maja Rudolph}{bcai}
\end{icmlauthorlist}

\icmlaffiliation{bcai}{Bosch Center for AI}
\icmlaffiliation{tuk}{TU Kaiserslautern}
\icmlaffiliation{uci}{UC Irvine}

\icmlcorrespondingauthor{Maja Rudolph}{majarita.rudolph@de.bosch.com}

\icmlkeywords{Machine Learning, anomaly detection, deep learning, contrastive losses, transformation learning}

\vskip 0.3in
]



\printAffiliationsAndNotice{} 

\begin{abstract}
Data transformations (e.g. rotations, reflections, and cropping) play an important role in self-supervised learning. Typically, images are transformed into different views, and neural networks trained on tasks involving these views produce useful feature representations for downstream tasks, including anomaly detection. However, for anomaly detection beyond image data, it is often unclear which transformations to use. 
Here we present a simple end-to-end procedure for anomaly detection with \emph{learnable} transformations. The key idea is to embed the transformed data into a semantic space such that the transformed data still resemble their untransformed form, while different transformations are easily distinguishable. 
Extensive experiments on time series show that our proposed method outperforms existing approaches in the one-vs.-rest setting and is competitive in the more challenging $n$-vs.-rest anomaly-detection task.
On medical and cyber-security tabular data, our method learns domain-specific transformations and detects anomalies more accurately than previous work.
\end{abstract}

\input{introduction}

\input{related}

\input{method}
\input{experiments}

\input{conclusion}

\input{Acknowledgements}

\bibliography{refs}
\bibliographystyle{icml2021}
\newpage
\appendix
\onecolumn
\input{appendix}
\end{document}

%% file: introduction.tex
\section{Introduction}
Many recent advances in anomaly detection rely on the paradigm of data augmentation. In the self-supervised setting, especially for image data, predefined transformations such as rotations, reflections, and cropping are used to generate varying {\em views} of the data. This idea has led to strong anomaly detectors based on either transformation prediction \citep{golan2018deep,wang2019effective,hendrycks2019using} or using representations learned using these views \citep{chen2020simple} for downstream anomaly detection tasks  \citep{sohn2020learning, tack2020csi}.

Unfortunately, for data other than images, such as time series or tabular data, it is much less well known which transformations are useful, and it is hard to design these transformations manually. This paper studies self-supervised anomaly detection for data types beyond images. We develop \gls{ntl}: a simple end-to-end procedure for anomaly detection with {\em learnable} transformations. Instead of manually designing data transformations to construct auxiliary prediction tasks that can be used for anomaly detection, we derive a single objective function for jointly learning useful data transformations and anomaly thresholding.  
As detailed below, the idea is to learn a variety of transformations such that the transformed samples share semantic information with their untransformed form, while different views are easily distinguishable.  

\Gls{ntl} has only two components: a fixed set of learnable transformations and an encoder model. Both elements are jointly trained on a noise-free, \gls{dcl} designed to learn faithful transformations. Our \gls{dcl} is different from other contrastive losses in representation learning \citep{gutmann2010noise,gutmann2012noise,mnih2013learning,oord2018representation,bamler2020extreme,chen2020simple} and image anomaly detection \citep{tack2020csi,sohn2020learning}, all of which use negative samples from a noise distribution. In contrast, our approach leads to a non-stochastic objective that neither needs any additional regularization or adversarial training  \citep{tamkin2020viewmaker} and can be directly used as the anomaly score.

Our approach leads to a new state of the art in anomaly detection beyond images. For time series and tabular data, \gls{ntl} significantly improves the anomaly detection accuracy. For example, in an epilepsy time series dataset, we raised the state-of-the-art from an AUC of {\bf$82.6\%$} to {\bf$92.6\%$} ($+10\%$). On an Arrhythmia tabular dataset, we raised the F1 score by $+2.9$ percentage points to an accuracy of $60.3$.



Our paper is structured as follows. We first discuss related work (\Cref{sec:related}) and present \gls{ntl} in \Cref{sec:ntl}. In \Cref{sec:ntl_obs}, we discuss its advantages for neural transformation learning in comparison to other self-supervised learning objectives. Experimental results are presented in \Cref{sec:empirical}.  \Cref{sec:conclusion} concludes this paper.

\input{figures/schema}

%% file: figures/schema.tex
\begin{figure*}[t!]
\begin{center}
	\captionsetup[subfigure]{labelformat=empty}
	\centering
	\resizebox{0.8\textwidth}{!}{
	\LARGE
	\begin{tabular}{@{}c@{}c@{}l@{}l@{}c@{}}
	\includegraphics[height=0.4\textheight]{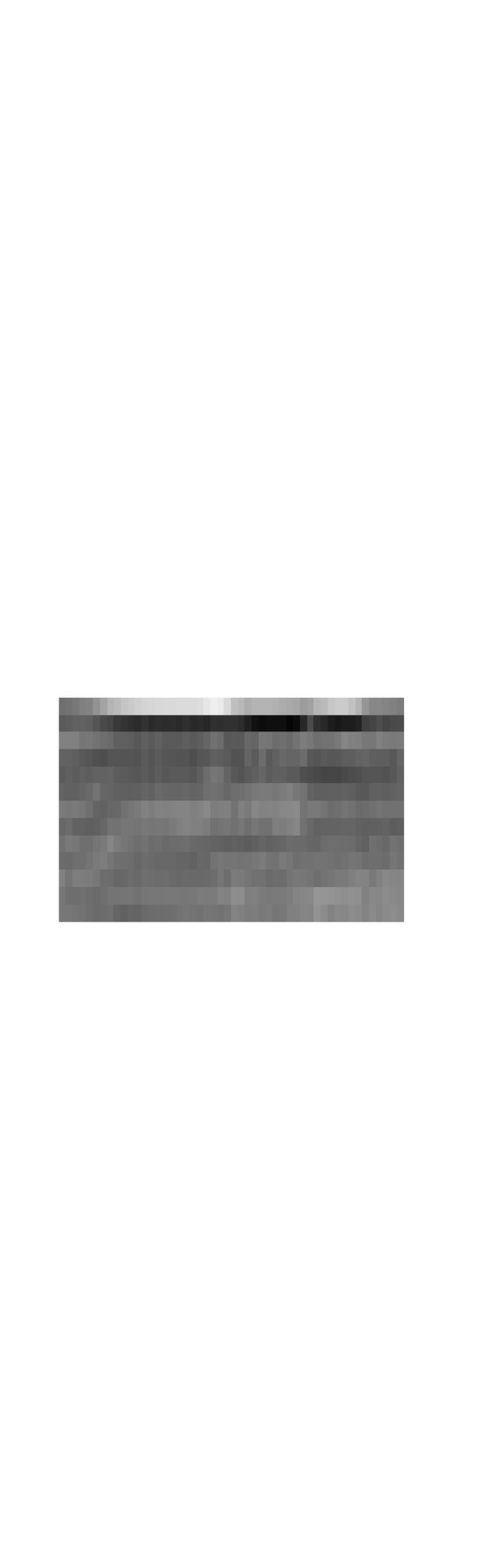}&
	\includegraphics[height=0.4\textheight]{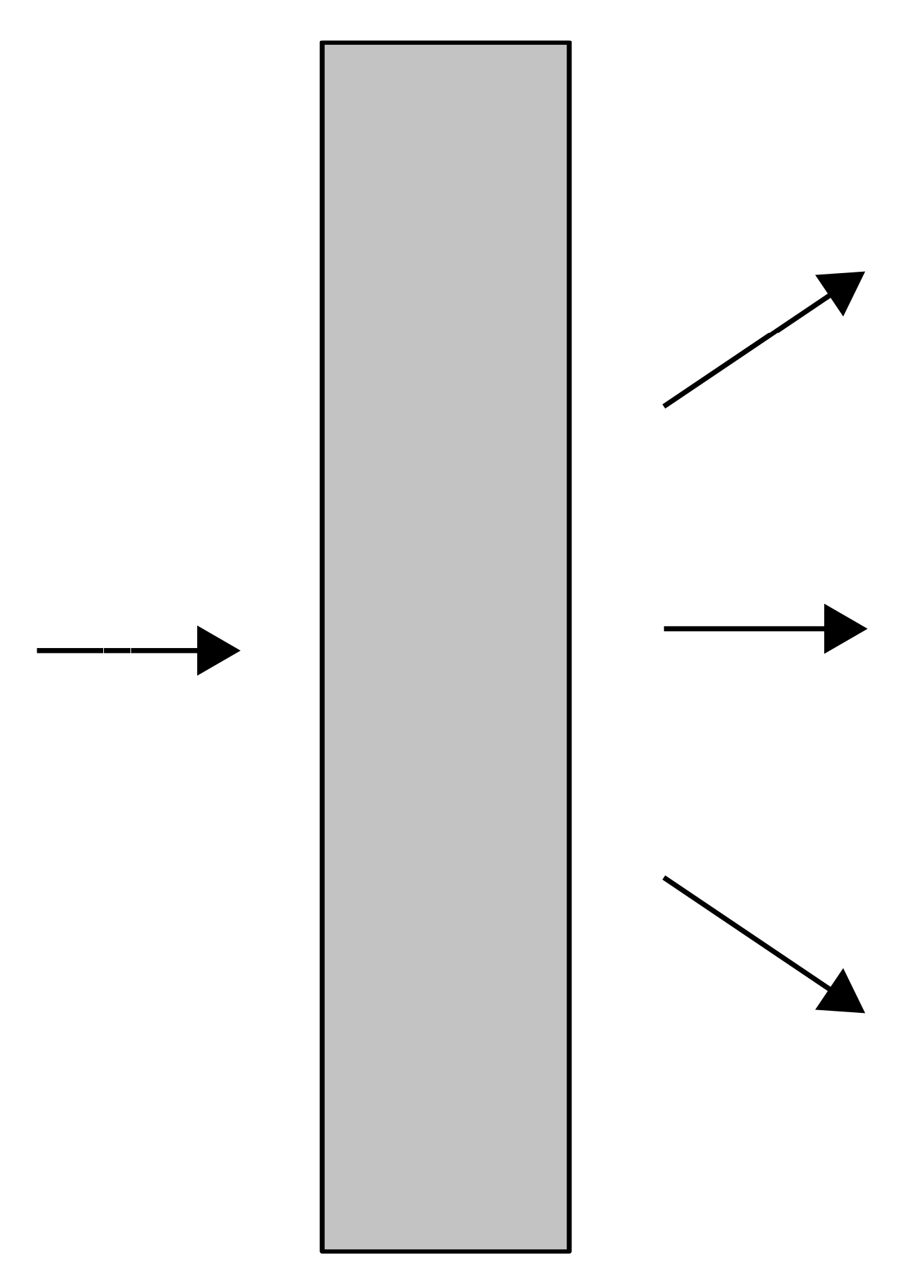}&
	\includegraphics[height=0.4\textheight]{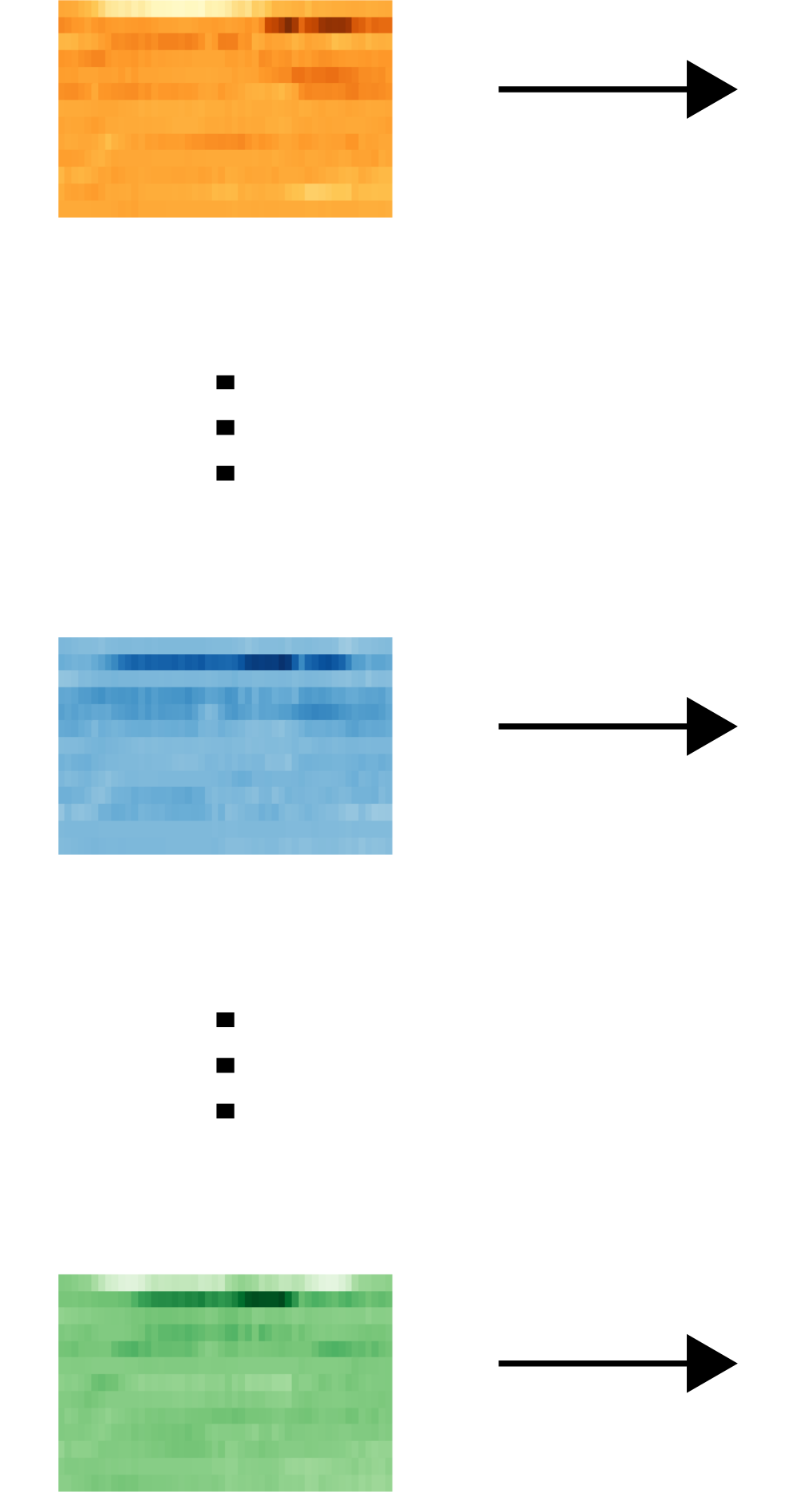}&
	\includegraphics[height=0.4\textheight]{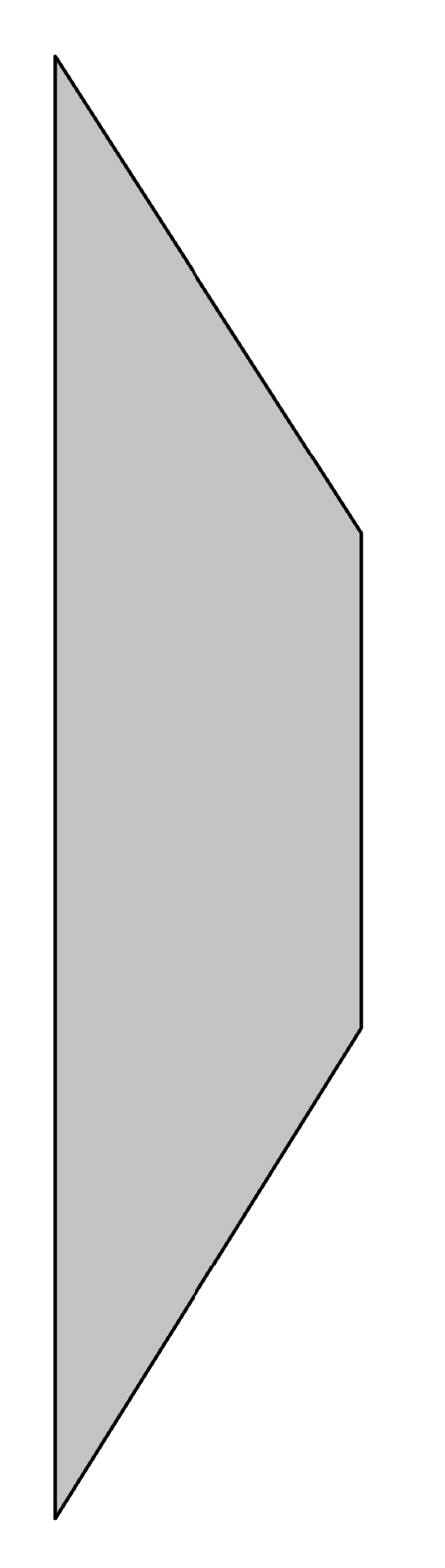}&
	\includegraphics[height=0.4\textheight]{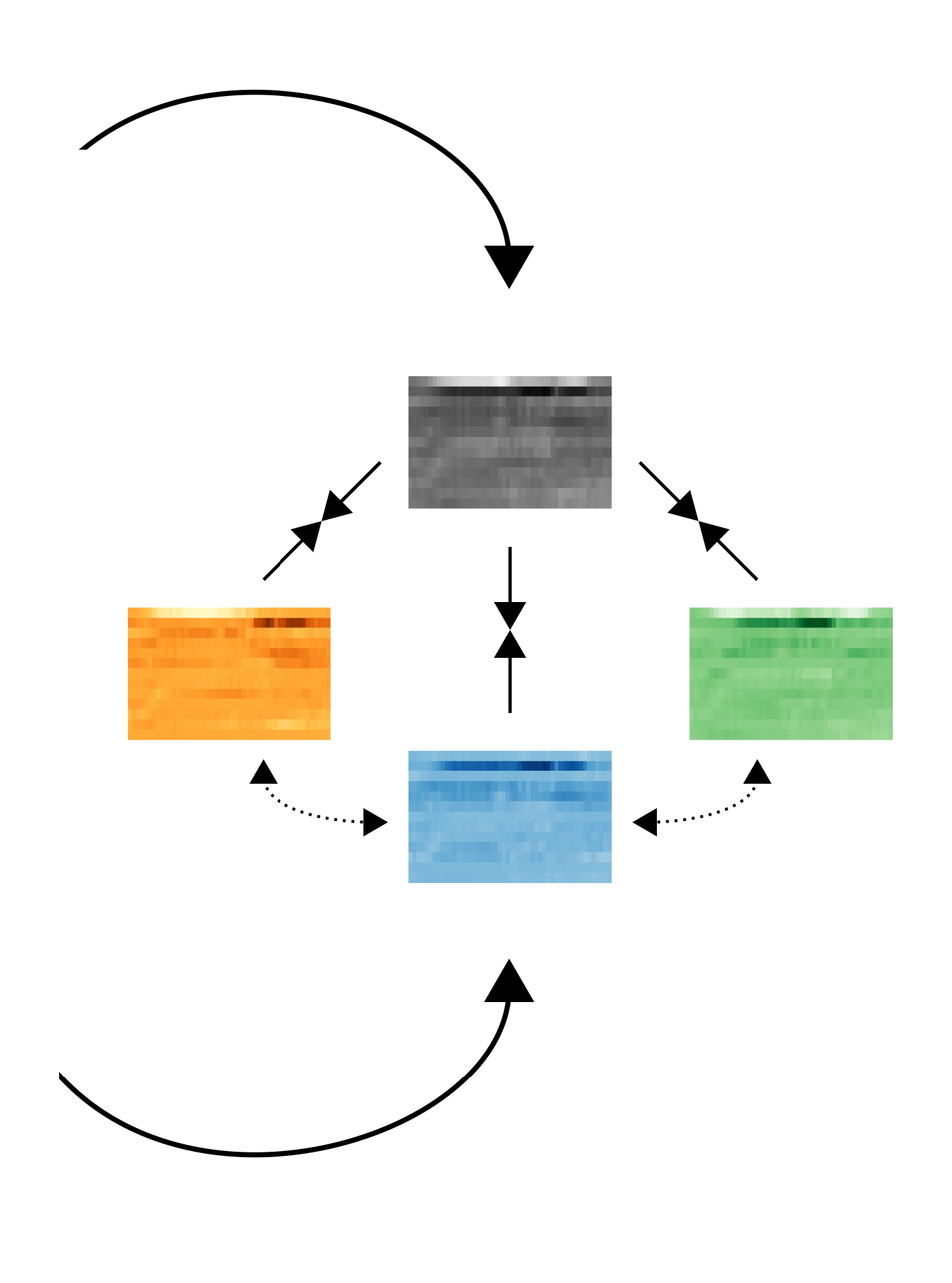}\\
	data & neural & transformed & encoder & deterministic contrastive \\
	(e.g. spectograms) & transformations &data (views) &  & loss/anomaly score
\end{tabular}}
		\caption{\gls{ntl} is a end-to-end procedure for self-supervised anomaly detection with {\em learnable} neural transformations. Each sample is transformed by a set of learned transformations and then embedded into a semantic space. The transformations and the encoder are trained jointly on a contrastive objective (\Cref{eqn:contrastive_loss}), which is also used to score anomalies. }  
		\label{fig:schema}
\end{center}
\end{figure*}

%% file: related.tex
\section{Related Work}
\label{sec:related}
\parhead{Deep Anomaly Detection.} Recently, there has been a rapidly growing interest in developing deep-learning approaches for anomaly detection \citep{ruff2021unifying}. While deep learning---by removing the burden of manual feature engineering for complex problems---has brought about tremendous technological advances, its application to anomaly detection is rather recent. Related work on deep anomaly detection includes deep autoencoder variants \citep{principi2017acoustic,zhou2017anomaly,chen2018unsupervised}, deep one-class classification \citep{erfani2016high, ruff2019deep,ruff2019self}, deep generative models \citep{schlegl2017unsupervised, deecke2018image}, and outlier exposure \citep{hendrycks2018deep,goyal2020drocc}.  

Self-supervised anomaly detection has led to drastic improvements in detection accuracy \citep{golan2018deep,hendrycks2019using,wang2019effective,sohn2020learning,tack2020csi}. 
For instance, \citet{golan2018deep} and \citet{wang2019effective} augment the data and learn to predict which transformation was applied. 
After training, the resulting classifier is used for anomaly detection. 

An alternative approach to self-supervised anomaly detection is to train a classifier on a contrastive loss to tell if two views are of the same original image. This leads to strong representations \citep{chen2020simple}, which can be used for anomaly detection \citep{sohn2020learning,tack2020csi}.

\citet{bergman2020classification} study how to extend self-supervised anomaly detection to other domains beyond images.
Similar to \citet{golan2018deep} and \citet{wang2019effective}, their approach is based on transformation prediction, but they consider the open-set setting.
For tabular data, they use random affine transformations. We study the same datasets, but our method \emph{learns} the transformations and achieves consistently higher performance (see \Cref{sec:empirical}).

\parhead{Self-Supervised Learning.}
Self-supervised learning typically relies on data augmentation  for auxiliary tasks \citep{doersch2015unsupervised,noroozi2016unsupervised,misra2016shuffle,zhang2017split,gidaris2018unsupervised}.
The networks trained on these auxiliary tasks (e.g. patch prediction \citep{doersch2015unsupervised}, solving jigsaw-puzzles \citep{noroozi2016unsupervised}, cross-channel prediction \citep{zhang2017split}, or rotation prediction \citep{gidaris2018unsupervised}) are used as feature extractors for downstream tasks. While many of these methods are developed for images, \citet{misra2016shuffle} propose temporal order verification as an auxiliary task for self-supervised learning of time series representations. %

\parhead{Contrastive Representation Learning.}
Many recent self-supervised methods have relied on the InfoMax principle \citep{linsker1988self,hjelm2018learning}. These methods are trained on the task to maximize the \gls{mi} between the data and their context \citep{oord2018representation} or between different ``views'' of the data \citep{bachman2019learning}. 
Computing the mutual information in these settings is often intractable and various approximation schemes and bounds have been introduced \citep{tschannen2019mutual}. By using noise contrastive estimation \citep{gutmann2010noise,gutmann2012noise} to bound \gls{mi}, \citet{oord2018representation} bridge the gap between contrastive losses for \gls{mi}-based representation learning and the use of contrastive losses in discriminative methods for representation learning \citep{hadsell2006dimensionality,mnih2013learning,dosovitskiy2015discriminative,bachman2019learning,chen2020simple}.
We also use a contrastive loss. But while the contrastive loss of \citet{chen2020simple} (which is used for anomaly detection of images in \citet{sohn2020learning,tack2020csi},) 
contrast two views of the same sample with views of other samples in the minibatch, \gls{ntl} is tasked with determining the original version of a sample from different views of the same sample. The dependence on only a single sample is advantageous for scoring anomalies at test time and enables us to learn the data transformations (discussed further in \Cref{sec:ntl_obs}).

\parhead{Learning Data Augmentation Schemes.}
The idea of learning data augmentation schemes is not new. ``AutoAugmentation'' has usually relied on composing hand-crafted data augmentations \citep{ratner2017learning, Cubuk_2019_CVPR, zhang2019adversarial, ho2019population,lim2019fast}.
\citet{tran2017bayesian} learn Bayesian augmentation schemes for neural networks, and \citet{wong2020learning} learn perturbation sets for adversarial robustness. Though their setting and approach are different, our work is most closely related to \citet{tamkin2020viewmaker}, who study how to generate views for representation learning in the framework of \citet{chen2020simple}. They parametrize their ``viewmakers'' as residual perturbations, which are trained adversarially to avoid trivial solutions where the views share no semantic information with the original sample (discussed in \Cref{sec:ntl_obs}). 



\gls{ntl} falls into the area of deep, self-supervised anomaly detection, with the core novelty of learning the transformations so that we can effectively use them for anomaly detection beyond images. Our method receives whole time series or tabular data as input. For time series, this is a remarkable difference to prevalent work on anomaly detection {\em within}  time series \citep[e.g.][]{shen2020timeseries}, which output anomaly scores per time point, but not for the sequence as a whole. Additional approaches to time series anomaly detection include shallow \citep{hyndman2015large, baragona2007outliers} and deep-learning approaches based on modeling \cite{munir2018deepant}, on autoencoders \citep{kieu2018outlier}, or one-class classification \citep{shen2020timeseries}.

%% file: method.tex
\glsresetall
\section{Neural Transformation Learning for Deep Anomaly Detection}
\label{sec:ntl}
We develop \gls{ntl}, a deep anomaly detection method based on contrastive learning for general data types. We first describe the approach in Section~\ref{sec:neutralad}. In Section~\ref{sec:ntl_obs}, we provide theoretical arguments why alternative contrastive loss functions are less suited for transformation learning.


\subsection{Proposed Method: \acrshort{ntl}}
\label{sec:neutralad}
Our method, \Gls{ntl}, is a simple pipeline with two components: a set of learnable transformations, and an encoder. Both are trained jointly on a \gls{dcl}. The objective has two purposes. During training, it is optimized to find the parameters of the encoder and the transformations. During testing, the \gls{dcl} is also used to score each sample as either an inlier or an anomaly.


\parhead{Learnable Data Transformations.}
We consider a data space $\mathcal{X}$ with samples $\mathcal{D}=\{x^{(i)} \sim \mathcal{X} \}_{i=1}^{N}$. 
We also consider $K$ transformations $\mathcal{T} :=  \{T_1,...,T_K | T_k : \mathcal{X} \rightarrow \mathcal{X} \}$. We assume here that the transformations are learnable, i.e., they can be modeled by any parameterized function whose parameters are accessible to gradient-based optimization and we denote the parameters of transformation $T_k$ by $\theta_k$. In our experiments, we use feed-forward neural networks for $T_k$. Note that in Section~\ref{sec:ntl_obs}, we use the same notation also for fixed transformations (such as rotations and cropping). 
 
 \parhead{Deterministic Contrastive Loss (DCL).}
A key ingredient of \gls{ntl} is a new objective. The \gls{dcl} encourages each transformed sample $x_k = T_k(x)$ to be similar to its original sample $x$, while encouraging it to be dissimilar from other transformed versions of the same sample, $x_l=T_l(x)$ with $l\neq k$.
We define the score function of two (transformed) samples as
\begin{align}
\label{eqn:score}
    h(x_k, x_l) = \exp (\mathrm{sim}(f_{\phi}(T_k(x)), f_{\phi}(T_l(x)))/ \tau),
\end{align}
where $\tau$ denotes a temperature parameter, and the similarity is defined as the cosine similarity $\mathrm{sim}(z,z') := z^T z'/ \|z\| \|z'\|$ in an embedding space $\mathcal{Z}$. The encoder $f_{\phi}(\cdot):\mathcal{X}\rightarrow \mathcal{Z}$ serves as a features extractor. 
The \gls{dcl} is
\begin{align}
\label{eqn:contrastive_loss}
    \mathcal{L}:=&\mathbb{E}_{x \sim \mathcal{D}}\left[-\sum_{k=1}^K \log \frac{h(x_k, x)}{h(x_k, x) + \sum_{l\neq k}h(x_k, x_l)}\right].
\end{align}
The term in the nominator of \Cref{eqn:contrastive_loss} pulls the embedding of each transformed sample close
to the embedding of the original sample. This encourages the transformations to preserve relevant semantic information. The denominator pushes all the embeddings of the transformed samples away
from each other, thereby encouraging diverse transformations.
The parameters of \gls{ntl} $\theta=[\phi, \theta_{1:K}]$ consist of the parameters $\phi$ of the encoder, and the parameters $\theta_{1:K}$ of the learnable transformations. All parameters $\theta$ are optimized jointly by minimizing \Cref{eqn:contrastive_loss}.

A schematic of \gls{ntl} is in \Cref{fig:schema}. Each sample is transformed by a set of learnable transformations and then embedded into a semantic space. The transformations and the encoder are trained jointly on the \gls{dcl} (\Cref{eqn:contrastive_loss}), which is also used to score anomalies. 

\parhead{Anomaly Score.} One advantage of our approach over other methods is that our training loss is also our anomaly score. Based on \Cref{eqn:contrastive_loss}, we define an anomaly score $S(x)$ as
\begin{align}
\label{eqn:anomaly_score}
    S(x) =-\sum_{k=1}^K \log \frac{h(x_k, x)}{h(x_k, x) + \sum_{l\neq k}h(x_k, x_l)}.
\end{align}
Since the score is deterministic, it can be straightforwardly evaluated for new data points $x$ without negative samples. 
By minimizing the \gls{dcl} (\Cref{eqn:contrastive_loss}), we minimize the score (\Cref{eqn:anomaly_score}) for training examples (inliers). The higher the anomaly score, the more likely that a sample is an anomaly.

This concludes the proposed loss function. We stress that it is extremely simple and does not need any additional regularization. However, it is not trivial to see why other proposed self-supervised approaches are less well suited for anomaly detection without imposing constraints on the types of transformations. To this end, we establish theoretical requirements and desiderata for neural transformation learning. 

\subsection{A Theory of Neural Transformation Learning}
\label{sec:ntl_obs}
To learn transformations for self-supervised anomaly detection we pose two requirements.

\theoremstyle{remark}
\begin{requirement}[Semantics]
\label{req:sem}
The transformations should produce views that share relevant semantic information with the original data.
\end{requirement}

\theoremstyle{remark}
\begin{requirement}[Diversity]
\label{req:div}
The transformations should produce diverse views of each sample.
\end{requirement}

A valid loss function for neural transformation learning should avoid solutions that violate either of these requirements. 
There are plenty of transformations that would violate \Cref{req:sem} or \Cref{req:div}. For example, a constant transformation $T_k(x)=c_k$, where $c_k$ is a constant that does not depend on $x$, would violate the semantic requirement, whereas the identity $T_{1}(x) = \cdots = T_{K}(x) = x$ violates the diversity requirement. 

\input{figures/score_hist}

We argue thus that for self-supervised anomaly detection, the learned transformations need to negotiate the trade-off between semantics and diversity with the two examples as edge-cases on a spectrum of possibilities. 
Without semantics, i.e. without dependence on the input data, an anomaly detection method can not decide whether a new sample is normal or an anomaly. 
And without variability in learning transformations, the self-supervised learning goal is not met. 

We now put the benefits of \gls{ntl} into perspective by comparing the approach with two other approaches that use data transformations for anomaly detection: 
\begin{enumerate}[leftmargin=*,topsep=0pt,noitemsep]
\item 
The first approach is the transformation prediction approach by \citet{wang2019effective}. Here, $f_\phi(\cdot): \mathcal{X} \rightarrow \mathbb{R}^K$ is a deep classifier\footnote{even though here $f_{\phi}$ is a classifier, we refer to it as the encoder in the discussion below.} that outputs $K$ values $f_\phi(x)_{1}\cdots f_\phi(x)_{K}$ proportional to the log-probabilities of the transformations. The transformation prediction loss is a softmax classification loss,
\begin{align}
    \mathcal{L}_P :=&\mathbb{E}_{x \sim \mathcal{D}}\left[  -\sum_{k=1}^K \log \frac{\exp{f_\phi(x_k)_k}}{\sum_{l=1:K}\exp{f_\phi(x_k)_l}}\right]. \label{eqn:lp}
\end{align}
\item We also consider \citep{chen2020simple}, who define a contrastive loss on each minibatch of data $\mathcal{M} \subset \mathcal{D}$ of size $N = |\mathcal{M}|$. For each gradient step, they sample a minibatch and two transformations $T_1, T_2 \sim \mathcal{T}$ from the family of transformations, which are applied to all the samples to produce $x^{(i)}_k = T_k(x^{(i)})$. 
The loss function is given by $\mathcal{L}_C(\mathcal{M})  := \sum_{i=1}^N  \mathcal{L}_C (x_1^{(i)}, x_2^{(i)}) +  \mathcal{L}_C  (x_2^{(i)}, x_1^{(i)})$, where
\begin{align}
\label{eqn:lc}
      &\mathcal{L}_C (x_1^{(i)}, x_2^{(i)}) := - \log h(x_1^{(i)}, x_2^{(i)}) \\\nonumber + & \log \left[ \sum_{j=1}^N h(x_1^{(i)}, x_2^{(j)})+\sum_{j=1}^N\mathds{1}_{[j\neq i]} h(x_1^{(i)}, x_1^{(j)}) \right].
\end{align}


\end{enumerate}

With hand-crafted, fixed transformations, these losses produce excellent anomaly detectors for images \citep{golan2018deep,wang2019effective,sohn2020learning,tack2020csi}. Since it is not always easy to design transformations for new application domains, we study their suitability for \emph{learning} data transformations. 

We argue that $\mathcal{L}_P$ and $\mathcal{L}_C$ are less well suited for transformation learning than $\mathcal{L}$: 

\begin{claim}
\label{c1}
 The `constant' edge-case $f_\phi(T_k(x)) = Cc_k$, where $c_k$ is a one-hot vector encoding the $k^{th}$ position (i.e. $c_{kk} = 1$), tends towards the minimum of $\mathcal{L}_P$ (\Cref{eqn:lp}) as the constant $C$ goes to infinity.
\end{claim}

\begin{claim}
\label{c2}
The `identity' edge-case $T_k(x) = x$ with adequate encoder $f_\phi$ is a minimizer of $\mathcal{L}_C$ (\Cref{eqn:lc}).
\end{claim}

The proof of these propositions is in \Cref{sec:appendix_proof}. The intuition is simple. Transformations that only encode which transformation was used make transformation prediction easy (\Cref{c1}), whereas the identity makes any two views of a sample identical, which can then be easily recognized as a positive pair by $\mathcal{L}_C$ (\Cref{c2}). 

The propositions highlight a serious issue with using $\mathcal{L}_P$ or $\mathcal{L}_C$ for transformation learning and anomaly detection. Should the optimization reach the edge-cases of \Cref{c1,c2}, $\mathcal{L}_P$ and $\mathcal{L}_C$ incur the same loss irrespective of whether the inputs are normal or abnormal data.  
Here are three remedies that can help avoid the trivial edge-cases:
Through careful {\em parametrization}, one can define $T_k(\cdot;\theta_k)$ in a way that explicitly excludes the edge cases. Beware that the transformation family might contain other members that violate  \Cref{req:sem} or \Cref{req:div}. The second potential remedy is {\em regularization} that explicitly encourages  \Cref{req:sem,req:div}, e.g. based on the InfoMax principle \citep{linsker1988self} or norm constraints. Finally, one can resort to {\em adversarial training}. 

\citet{tamkin2020viewmaker} use all three of these remedies to learn ``viewmakers'' using the contrastive loss $\mathcal{L}_C$; They parametrize the transformations as residual perturbations, which are regularized to the $\ell_p$ ball and trained adversarially. In contrast, under \gls{ntl} there are no restrictions on the architecture of the transformations, as long as \Cref{eqn:contrastive_loss} can be optimized (i.e. the gradient is well defined). The \gls{dcl} is an adequate objective for training the encoder and transformations jointly as it manages the trade-off between \Cref{req:sem,req:div}.

\begin{claim}
\label{c3}
The edge-cases of \Cref{c1,c2} do not minimize $\mathcal{L}$ (\gls{dcl}, \Cref{eqn:contrastive_loss}).
\end{claim}

The proof is in \Cref{sec:appendix_proof}. The numerator of \Cref{eqn:contrastive_loss} encourages transformed samples to resemble their original version (i.e. the semantic requirement) and the denominator encourages the diversity of transformations. The result of our well-balanced objective is a heterogeneous set of transformations that model various relevant aspects of the data. The transformations and the encoder need to highlight salient features of the data such that a low loss can be achieved. After training, samples from the normal class have a low anomaly score while anomalies are handled less well by the model and as a result, have a high score.   

\Cref{fig:hist} shows empirical evidence for this. We observe that, while the histogram of anomaly scores (computed using \Cref{eqn:anomaly_score}) is similar for inliers and anomalies before training, this changes drastically after training, and held-out inliers and anomalies become easily distinguishable.


There's another advantage of using the \gls{dcl} for self-supervised anomaly detection. Unlike most other contrastive losses, the ``negative samples'' are not drawn from a noise distribution (e.g. other samples in the minibatch) but constructed deterministically from $x$. Dependence on the minibatch for negative samples would need to be accounted for at test time. In contrast, the deterministic nature of \Cref{eqn:anomaly_score} makes it a simple choice for anomaly detection.

%% file: figures/score_hist.tex
\begin{figure}[t!]
	\centering
	\begin{subfigure}[b]{0.49\linewidth}
		\includegraphics[width=\linewidth]{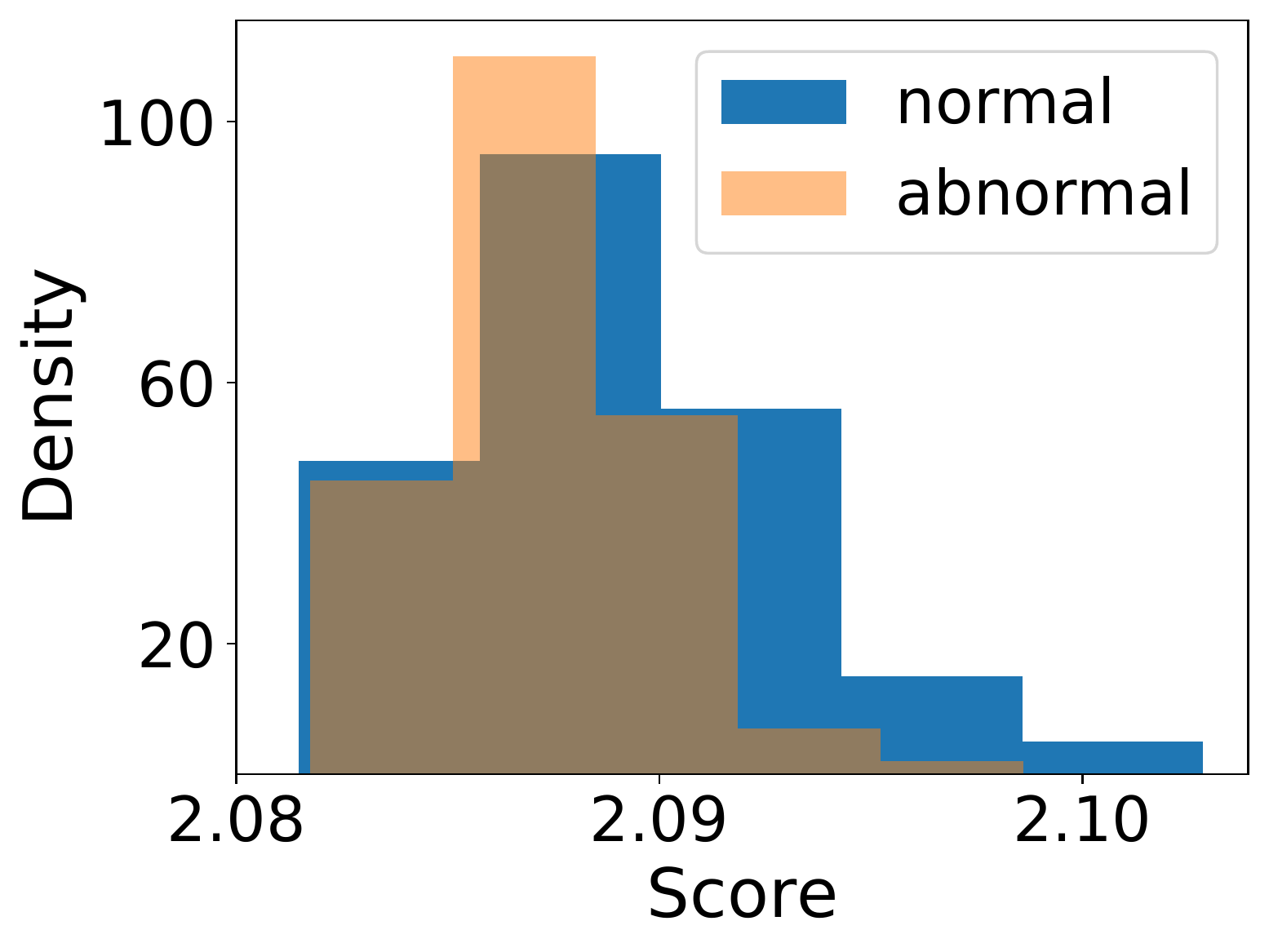}
		\caption{histogram before training}
		\label{fig:hist_initial}
	\end{subfigure}
	\begin{subfigure}[b]{0.46\linewidth}
		\includegraphics[width=\linewidth]{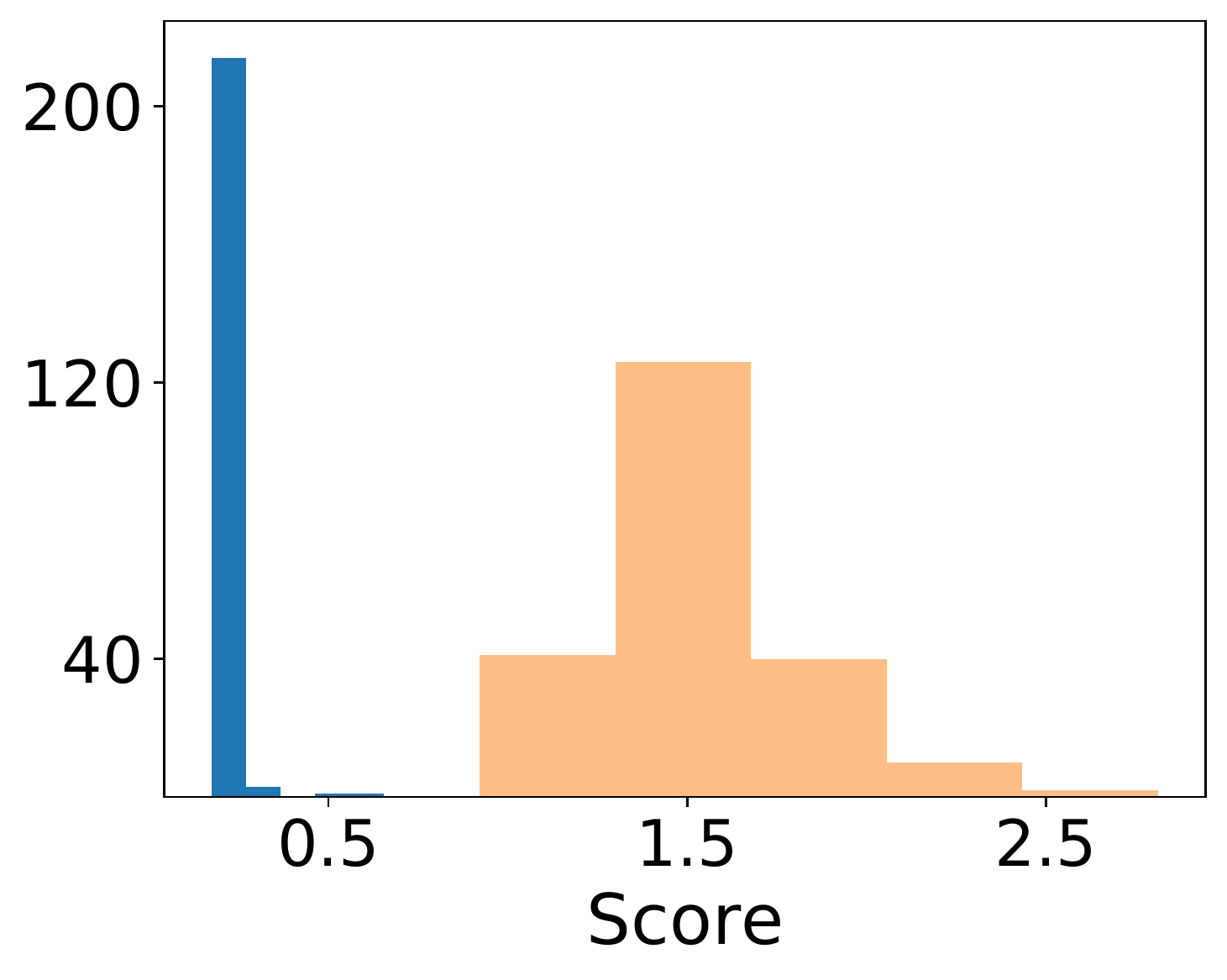}
		\caption{histogram after training}
		\label{fig:hist_trained}
	\end{subfigure}
	\caption{While the histogram of anomaly scores (computed using \Cref{eqn:anomaly_score}) is similar for inliers (blue) and anomalies (orange) before training, this changes drastically after training, and held-out inliers and anomalies become easily distinguishable. The data come from the \gls{sad} experiments described in \Cref{sec:empirical}}
	\label{fig:hist}
\end{figure}

%% file: experiments.tex
\glsunset{goad}
\glsunset{lstm}
\glsunset{svdd}
\glsunset{dagmm}
\glsunset{drocc}
\section{Empirical Study: Deep Anomaly Detection of Time Series and Tabular Data}
\label{sec:empirical}
We developed \gls{ntl} for deep anomaly detection beyond images, so we consider various application domains involving multiple data types.
For image data, strong self-supervised baselines exist that benefit from hand-crafted transformations. We do not expect any benefit from using \gls{ntl} there. Our focus here is on time series and tabular data, which are important in many application domains of anomaly detection. Our study finds that \gls{ntl} improves detection accuracy over the state of the art.

\subsection{Evaluation Protocol}
\label{sec:eval}
We compare \gls{ntl} to prevalent shallow and deep anomaly-detection baselines using two evaluation protocols: the standard `one-vs.-rest' and the more challenging `$n$-vs.-rest' evaluation protocol. 
Both settings turn a classification dataset into a quantifiable anomaly-detection benchmark.

\parhead{one-vs-rest.} This evaluation setup has been used in virtually all papers on deep anomaly detection published at top-tier venues \citep[e.g.][]{ruff2019deep, hendrycks2019using, ruff2018deep, golan2018deep, deecke2018image, akcay2018ganomaly, abati2019latent, perera2019ocgan, wang2019multivariate, bergman2020classification, kim2019rapp}. For `one-vs.-rest', the dataset is split by the $N$ class labels, creating $N$ one class classification tasks; the models are trained on data from one class and tested on a test set with examples from all classes. The samples from other classes should be detected as anomalies.

\parhead{n-vs-rest.} We also evaluate on the more challenging $n$-vs.-rest protocol, where $n$ classes (for $1<n< N$) are treated as normal and the remaining classes provide the anomalies in the test and validation set. By increasing the variability of what is considered normal data, one-class classification becomes more challenging. 

\subsection{Shallow and Deep Anomaly Detection Baselines}
\label{sec:baselines}
We study \gls{ntl} in comparison to a number of unsupervised and self-supervised anomaly detection methods.

\parhead{Traditional Anomaly Detection Baselines.} 
We chose three popular anomaly detection baselines: The \gls{ocsvm}, a kernel-based method, \gls{if}, a tree-based model which aims to isolate anomalies \citep{liu2008isolation}, and \gls{lof}, which uses density estimation with $k$-nearest neighbors.

\parhead{Deep Anomaly Detection Baselines.} 
Next, we include three deep anomaly detection methods,
\acrshort{svdd} \citep{ruff2018deep}, which fits a one-class SVM in the feature space of a neural net, \acrshort{drocc} \citep{goyal2020drocc}, which fits a one-class classifier with artificial outlier exposure, and \acrshort{dagmm} \citep{zong2018deep}, which estimates the density in the latent space of an autoencoder.

\parhead{Self-Supervised Anomaly Detection Baselines.} 
We also choose two self-supervised baselines, which are technically also deep anomaly detection methods.
\acrshort{goad} \citep{bergman2020classification} is a distance-based classification method based on random affine transformations. 
\citet{wang2019effective} is a softmax-based classification method based on hand-crafted transformations, which show impressive performance on images. We adopt their pipeline to time series here by crafting specific time series transformations (fixed Ts, described in \Cref{sec:appendix_implement}).

\input{tables/ts_results}
\parhead{Anomaly Detection Baselines for Time Series.}
Finally, we also include two baselines that are specifically designed for time series data:
The \gls{rnn} directly models the data distribution $p(x_{1:T})=\prod p(x_t|x_{<t})$ and uses the log-likelihood as the anomaly score. Details on the architecture are in \Cref{sec:appendix_implement}. 
\acrshort{lstm} \citep{malhotra2016lstm} is an encoder-decoder time series model where anomaly score is based on the reconstruction error.

\subsection{Anomaly Detection of Time Series}
Our goal is to detect abnormal time series on a {\em whole-sequence level}. This is a different set-up than novelty detection within time series, but also very important in practice.

For example, one might want to detect abnormal sound or find production quality issues by detecting abnormal sensor measurements recorded over the duration of producing a batch. Other applications are sports and health monitoring; an abnormal movement pattern during sports can be indicative of fatigue or injury; whereas anomalies in health data can point to more serious issues.

We study \gls{ntl} on a selection of datasets that are representative of these varying domains. The datasets come from the UEA multivariate time series classification archive\footnote{from which we selected datasets on which supervised multi-class classification methods achieve strong results \citep{ruiz2020great}. Only datasets with separable classes can be repurposed for one-class classification} \citep{bagnall2018uea}.

\parhead{Time Series Datasets}
\begin{itemize}[leftmargin=*,topsep=0pt]
	\setlength\itemsep{.1em}
	\item \Gls{sad}: Sound of ten Arabic digits, spoken by 88 speakers. The dataset has 8800 samples, which are stored as 13 Mel Frequency Cepstral Coefficients (MFCCs). We select sequences with the length between 20 and 50. The sequences that are shorter than 50 are zero padded to have the length of 50.
	\item \Gls{natops}: The data is originally from a motion detection competition of various movement patterns used to control planes in naval air training. The data has six classes of distinct actions. Each sample is a sequence of x, y, z coordinates for eight body parts of length 51.
	\item \Gls{ct}: The data consists of 2858 character samples from 20 classes, captured using a WACOM tablet. Each instance is a 3-dimensional pen tip velocity trajectory. The data is truncated to the length of the shortest, which is 182.
	\item \Gls{epilepsy}: The data was generated with healthy participants simulating four different activities: walking, running, sawing with a saw, and seizure mimicking whilst seated. The data has 275 cases in total, each being a 3-dimensional sequence of length 203.
	\item \Gls{rs}: The data is a record of university students playing badminton or squash whilst wearing a smart watch, which measures the x, y, z coordinates for both the gyroscope and accelerometer. Sport and stroke types separate the data into four classes. Each sample is a 6-d sequence with a length of 30.
\end{itemize}

We compare \gls{ntl} to all baselines from \Cref{sec:baselines} on these datasets under the one-vs-rest setting. Additionally, we study how the methods adapt to increased variability of inliers by exploring \gls{sad} and \gls{natops} under the $n$-vs-rest setting for a varying number of classes $n$ considered normal.


\parhead{Implementation Details}
We consider the following parametrizations of the learnable transformations:
{\em feed forward} $T_k(x):=M_k(x)$,
{\em residual} $T_k(x):=M_k(x) + x$, and
{\em multiplicative} $T_k(x):=M_k(x) \odot x$, which differ in how they combine the learnable masks $M_k(\cdot)$ with the data\footnote{We use $10\%$ of the test set as the validation set to allow parameterization selection.}.
The masks $M_k$ are each a stack of three residual blocks of 1d convolutional layers with instance normalization layers and ReLU activations, as well as one 1d convolutional layer on the top. For the multiplicative parameterization, a sigmoid activation is applied to the masks. All bias terms are fixed as zero, and the learnable affine parameters of the instance normalization layers are frozen. For a fair comparison, we use the same number of $12$ transformations in \gls{ntl}, \acrshort{goad}, and the classification-based method (fixed Ts) for which we manually designed appropriate transformations. 
In \Cref{sec:ablation} we make a more detailed comparison of various design choices for \gls{ntl} and one of our findings is that its anomaly detection results are robust to the number of learnable transformations.

The same encoder architecture is used for \gls{ntl}, \gls{svdd}, \gls{drocc}, and with slight modification to achieve the appropriate number of outputs for \gls{dagmm} and transformation prediction with fixed Ts. The encoder is a stack of residual blocks of 1d convolutional layers. The number of blocks depends on the dimensionality of the data and is detailed in \Cref{sec:appendix_implement}. 

\parhead{Results.}
The results of \gls{ntl} in comparison to the baselines from \Cref{sec:baselines} on time series datasets from various fields are reported in \Cref{tab:ts_one-vs-all}. \gls{ntl} outperforms all shallow baselines in all experiments and outperforms the deep learning baselines in 4 out of 5 experiments. 
Only on the \gls{rs} data, it is outperformed by transformation prediction with fixed transformations, which we designed to understand the value of learning transformations with \gls{ntl} vs using hand-crafted transformations. The results confirm that designing the transformations only succeeds sometimes, whereas with \gls{ntl} we can learn the appropriate transformations.
The learned transformations also give \gls{ntl} a competitive advantage over the other self-supervised baseline \gls{goad} which uses random affine transformations.
The performance of the shallow anomaly detection baselines hints at the difficulty of each anomaly detection task; the shallow methods perform well on \gls{sad} and \gls{ct}, but perform worse than the deep learning based methods on other data. 

\input{figures/z_space}
\parhead{What does \gls{ntl} learn?} For visualization purposes, we train \gls{ntl} with 4 learnable transformations on the \gls{sad} data. 
\Cref{fig:hypershere} shows the structure in the data space $\mathcal{X}$ and the embedding space of the encoder $\mathcal{Z}$ after training. Held-out data samples (blue) are transformed by each of the learned transformations $T_k(x) = M_k(x) \odot x$ to produce $K=4$ different views of each sample (the transformations are color-coded by the other colors). Projection to three principal components with PCA allows for visualization in 3D. In \Cref{fig:normal_x,fig:abnormal_x}, we can see that the transformations already cluster together in the data space, but only with the help of the encoder, the different views of inliers are separated from each other \Cref{fig:normal_z}. 
In comparison, the anomalies and their transformations are less structured in $\mathcal{Z}$ (\Cref{fig:abnormal_z}), visually explaining why they incur a higher anomaly score and can be detected as anomalies.
\input{figures/sad_masks}

The learned masks $M_{1:4}(x)$ of one inlier $x$ are visualized in \Cref{fig:masks}. We can see that the four masks are dissimilar from each other, and have learned to focus on different aspects of the spectrogram. The masks take values between $0$ and $1$, with dark areas corresponding to values close to $0$ that are zeroed out by the masks, while light colors correspond to the areas of the spectrogram that are not masked out. Interestingly, in $M_1$, $M_2$, and $M_3$ we can see `black lines' where they mask out entire frequency bands at least for part of the sequence. In contrast, $M_4$ has a bright spot in the middle left part of the spectrogram; it creates views that focus on the content of the intermediate frequencies at the first half of the recording.

\parhead{How do the methods cope with an increased variability of inliers?}
\input{tables/ts_n_vs_rest}
To study this question empirically, we increase the number of classes $n$ considered to be inliers. We test all methods on \gls{sad} and \gls{natops} under the $n$-vs-rest setting with varying $n$. Since there are too many combinations of normal classes when $n$ approaches $N-1$, we only consider combinations of $n$ consecutive classes. From \Cref{fig:n-vs-all} we can observe that the performance of all methods drops as the number of classes included in the normal data increases. This shows that the increased variance in the nominal data makes the task more challenging. \gls{ntl} outperforms all baselines on \gls{natops} and all deep-learning baselines on \gls{sad}.
It is interesting that \gls{lof}, a method based on $k$-nearest neighbors, performs better than our method (and all other baselines) on \gls{sad} when $n$ is larger than three.
\input{figures/n_vs_rest}
We also include quantitative results for $n = N-1 $ under the $n$-vs-rest setting for all time series datasets, where only one class is considered abnormal, and the remaining $N-1$ classes are normal. The results are reported in \Cref{tab:ts_rest-vs-all}. We can see, the performance of all deep learning based methods drops as the variability of normal data increases. Our method outperforms other deep learning methods on 4 out of 5 datasets. On \gls{rs}, it is outperformed by transformation prediction with hand-crafted transformations. The results are consistent with the experiments under one-vs.-rest setting in \Cref{tab:ts_one-vs-all}. The traditional method \gls{lof} performs better than deep learning methods on \gls{ct} and \gls{sad}.

\subsection{Anomaly Detection of Tabular Data}
Tabular data is another important application area of anomaly detection.
For example, many types of health data come in tabular form.
To unleash the power of self-supervised anomaly detection for these domains, \citet{bergman2020classification} suggest using random affine transformation. Here we study the benefit of {\em learning} the transformations with \gls{ntl}. 
We base the empirical study on tabular datasets used in previous work \citep{zong2018deep,bergman2020classification} and follow their precedent of reporting results in terms of F1-scores. 

\parhead{Tabular Datasets.}
We study the four tabular datasets from the empirical studies of \citet{zong2018deep,bergman2020classification}. The datasets include the small-scale medical datasets Arrhythmia and Thyroid as well as the large-scale cyber intrusion detection datasets KDD and KDDRev (see \Cref{sec:appendix_tab_data} for all relevant details). We follow the configuration of \citep{zong2018deep} to train all models on half of the normal data, and test on the rest of the normal data as well as the anomalies.

\parhead{Baseline Models.}
We compare \gls{ntl} to shallow and deep baselines outlined in \Cref{sec:baselines}, namely \gls{ocsvm}, \gls{if}, \gls{lof}, and the deep anomaly detection methods \gls{svdd}, \gls{dagmm}, \gls{goad}, and \gls{drocc}. 

\parhead{Implementation details.} 
The implementation details of \gls{ocsvm}, \gls{lof}, \gls{dagmm}, and \gls{goad} are replicated from \citet{bergman2020classification}, as we report their results. The implementation of \gls{drocc} is from their official code.

For neural transformations, we use the residual parametrization $T_k(x) = M_k(x) +x$ on Thyroid and Arrhythmia, and use the multiplicative parametrization $T_k(x) = M_k(x) \odot x$ on KDD, and KDDRev.
The masks $M_k$ consists of two bias-free linear layers with an intermediate ReLU activation. When using the multiplicative parametrization, it has a sigmoid activation on the final layer.
The encoder consists of five bias-free linear layers with ReLU activations. The output dimensions of the encoder are $24$ for Thyroid and $32$ for the other datasets.
We set the number of neural transformations $K=11$.


\parhead{Results.}
The results of \gls{ocsvm}, \gls{lof}, \gls{dagmm}, and \gls{goad} are taken from \citep{bergman2020classification}. The results of \gls{drocc} were provided by \citet{goyal2020drocc} \footnote{Their empirical study uses a different experimental setting while the results reported here are consistent with prior works.}. \gls{ntl} outperforms all baselines on all datasets. Compared with the self-supervised anomaly detection baseline \gls{goad}, we use much fewer transformations.
\input{tables/tab_results}

\subsection{Design Choices for the Transformations}
\label{sec:ablation}
In \Cref{sec:ntl_obs}, we have discussed the advantages of \gls{ntl} for neural transformation learning; no regularization or restrictions on the transformation family are necessary to ensure the transformations fulfill the semantics and diversity requirements defined in \Cref{sec:ntl_obs}.

In this section, we study the performance of \gls{ntl} under various design choices for the learnable transformations, including their parametrization, and their number $K$. We consider the following parametrizations:
{\em feed forward} $T_k(x):=M_k(x)$,
{\em residual} $T_k(x):=M_k(x) + x$, and
{\em multiplicative} $T_k(x):=M_k(x) \odot x$, which differ in how they combine the learnable masks $M_k(\cdot)$ with the data. 

In \Cref{fig:ablation_fig} we show the anomaly detection accuracy achieved with each parametrization, as $K$ varies from 2 to 15 on the time series data \gls{natops} and the tabular data Arrhythmia. For large enough $K$, \gls{ntl} is robust to the different parametrizations, since \gls{dcl} ensures the learned transformations satisfy the semantic requirement and the diversity requirement. The performance of \gls{ntl} improves as the number $k$ increases, and becomes stable when $K$ is large enough. When $K\leq 4$, the performance has a larger variance, since the learned transformations are not guaranteed to be useful for anomaly detection without the guidance of any labels. When $K$ is large enough, the learned transformations contain with high likelihood some transformations that are useful for anomaly detection.
The transformation based methods (including
\gls{ntl}) have roughly $K$ times the memory requirement as other deep learning methods (e.g. \gls{svdd}). However, the results in \Cref{fig:ablation_fig} show that even with small $K$ \gls{ntl} achieves competitive results.
\input{figures/ablation_fig}

%% file: tables/ts_results.tex
\begin{table*}[t!]
	\caption{Average AUC with standard deviation for one-vs-rest anomaly detection on time series datasets.}
	\label{tab:ts_one-vs-all}
	\centering
	\resizebox{\linewidth}{!}{
	\begin{tabular}{l|ccc|ccccccc|c}
        \hline
        &\acrshort{ocsvm}&\acrshort{if}&\acrshort{lof}&\acrshort{rnn}&\acrshort{lstm} &\acrshort{svdd} &\acrshort{dagmm} &\acrshort{goad} &\acrshort{drocc}&fixed Ts&\acrshort{ntl}\\
		\hline
\acrshort{sad}  &95.3 & 88.2&98.3 &81.5$\pm$0.4 &93.1$\pm$0.5 &86.0$\pm$0.1 &80.9$\pm$1.2 & 94.7$\pm$0.1 &85.8$\pm$0.8 &96.7$\pm$0.1 &\textbf{98.9}$\pm$0.1\\
\acrshort{natops} &86.0&85.4&89.2&89.5$\pm$0.4&91.5$\pm$0.3 &88.6$\pm$0.8&78.9$\pm$3.2 &87.1$\pm$1.1&87.2$\pm$1.4 &78.4$\pm$0.4& \textbf{94.5}$\pm$0.8\\
\acrshort{ct}   &97.4 &94.3 &97.8 &96.3$\pm$0.2 &79.0$\pm$1.1 &95.7$\pm$0.5 &89.8$\pm$0.7 &97.7$\pm$0.1&95.3$\pm$0.3  &97.9$\pm$0.1 & \textbf{99.3}$\pm$0.1\\
\acrshort{epilepsy}  &61.1 & 67.7 &56.1 &80.4$\pm$1.8 &82.6$\pm$1.7 &57.6$\pm$0.7 & 72.2$\pm$1.6 &76.7$\pm$0.4&85.8$\pm$2.1  &80.4$\pm$2.2  &\textbf{92.6}$\pm$1.7 \\
\acrshort{rs} &70.0&69.3&57.4&84.7$\pm$0.7&65.4 $\pm$2.1&77.4$\pm$0.7&51.0$\pm$4.2&79.9$\pm$0.6&80.0$\pm$1.0  &\textbf{87.7}$\pm$0.8&86.5$\pm$0.6\\
        \hline
	\end{tabular}}
\end{table*}

%% file: figures/z_space.tex
\begin{figure}[t!]
	\centering
	\begin{subfigure}[b]{0.4\linewidth}
		\includegraphics[width=\linewidth]{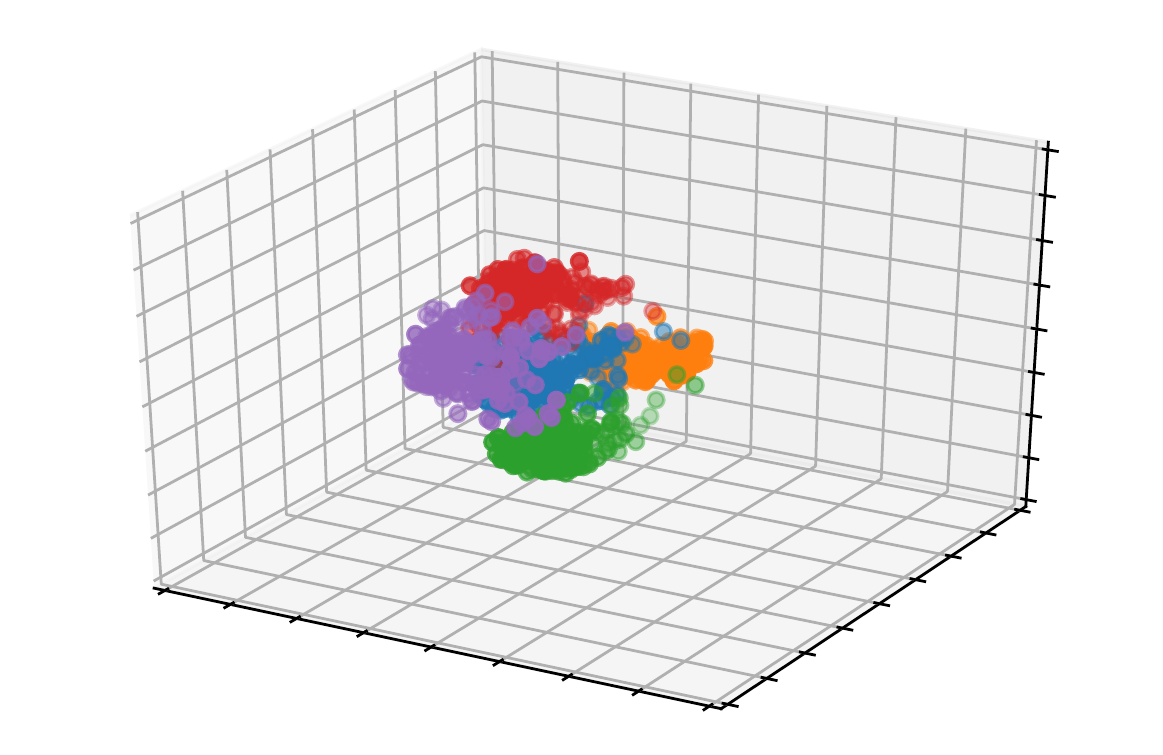}
		\caption{normal data on $\mathcal{X}$}
		\label{fig:normal_x}
	\end{subfigure}
	\begin{subfigure}[b]{0.4\linewidth}
		\includegraphics[width=\linewidth]{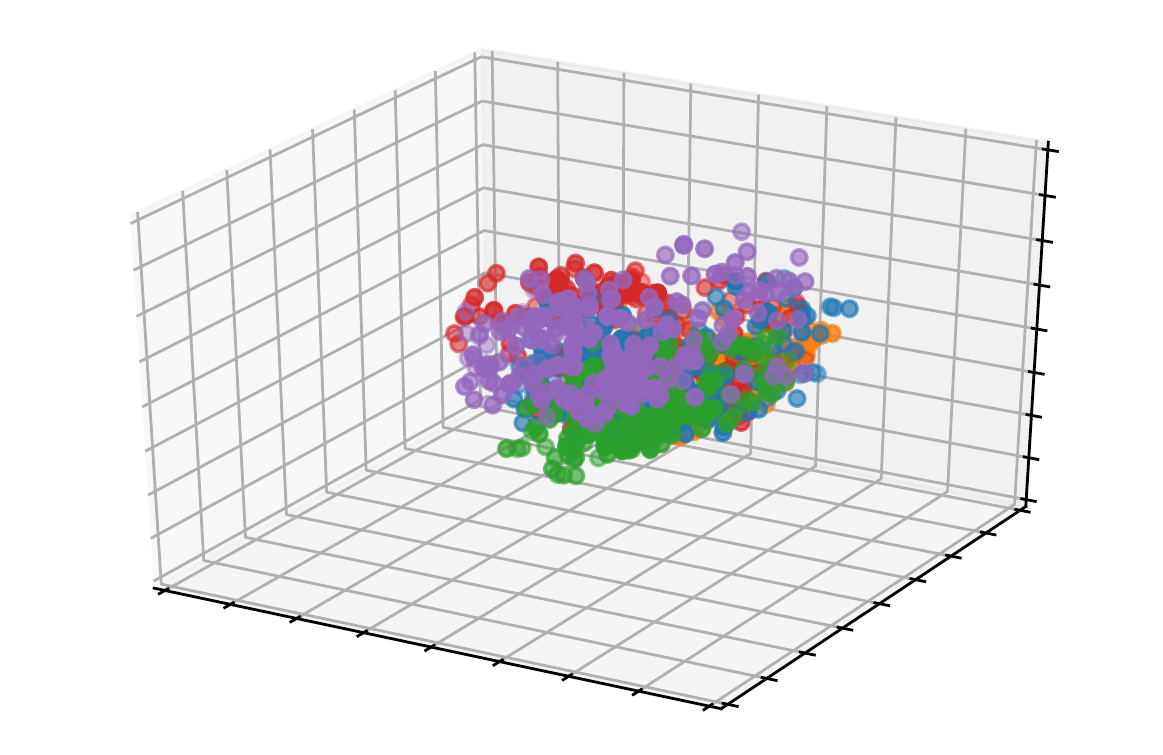}
		\caption{anomalies on $\mathcal{X}$}
		\label{fig:abnormal_x}
	\end{subfigure}\\
		\begin{subfigure}[b]{0.4\linewidth}
		\includegraphics[width=\linewidth]{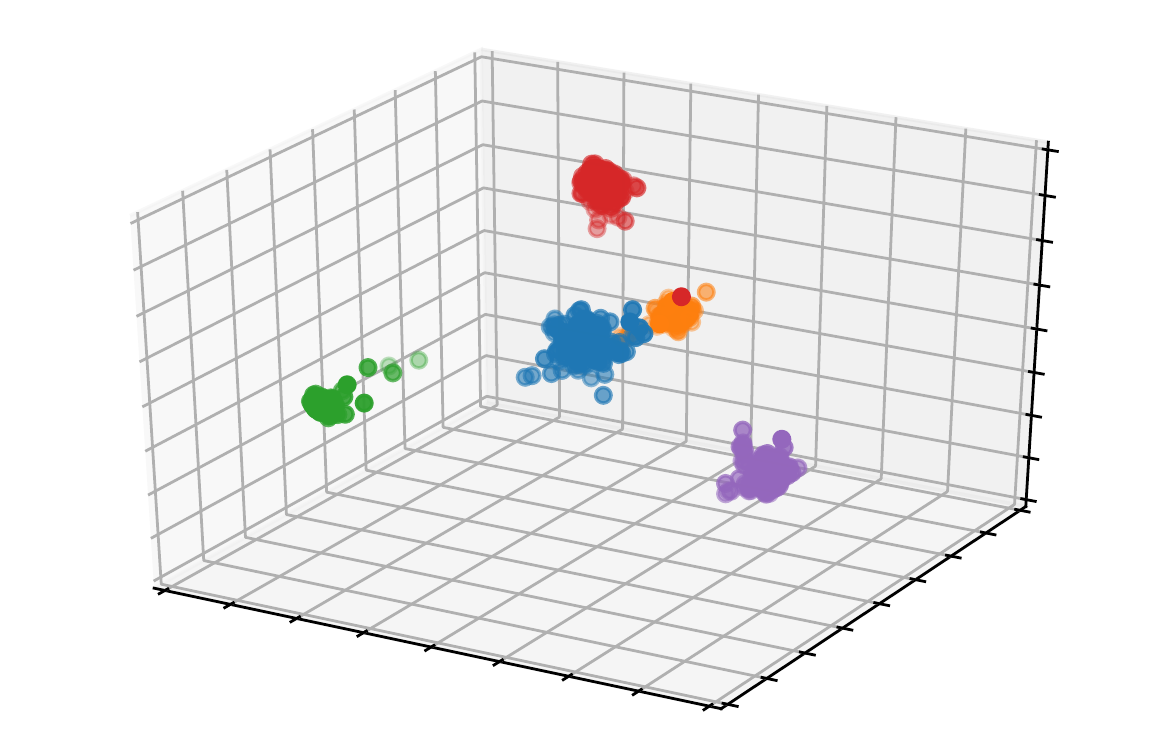}
		\caption{normal data on $\mathcal{Z}$}
		\label{fig:normal_z}
	\end{subfigure}
	\begin{subfigure}[b]{0.4\linewidth}
		\includegraphics[width=\linewidth]{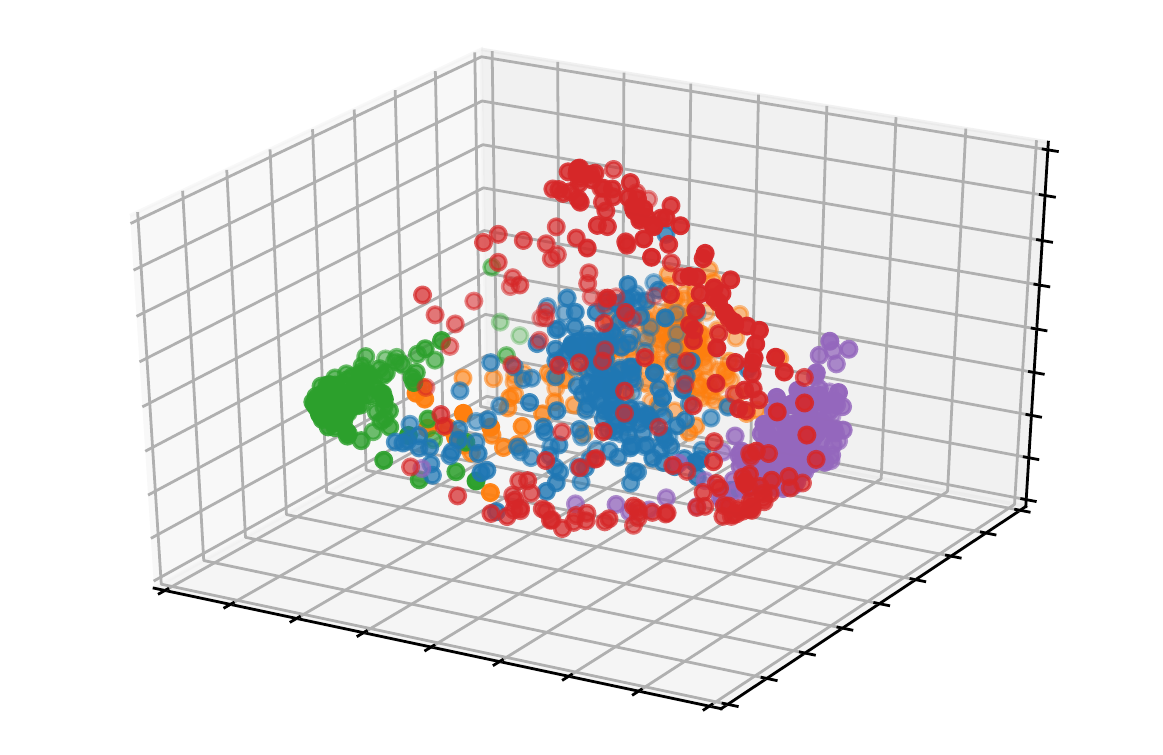}
		\caption{anomalies on $\mathcal{Z}$}
		\label{fig:abnormal_z}
	\end{subfigure}
	\caption{3D visualization (projected using PCA) of how the original samples (blue) from the \gls{sad} dataset and the different views created by the neural transformations of \gls{ntl} (one color per transformation type) cluster in data space (\Cref{fig:normal_x,fig:abnormal_x}) and in the embedding space of the encoder (\Cref{fig:normal_z,fig:abnormal_z}). The crisp separation of the different transformations of held-out inliers (\Cref{fig:normal_z}) in contrast to the overlap between transformed anomalies (\Cref{fig:abnormal_z}) visualizes how \gls{ntl} is able to detect anomalies. 
	}
	\label{fig:hypershere}
\end{figure}

%% file: figures/sad_masks.tex
\begin{figure}[t]
	\captionsetup[subfigure]{labelformat=empty}
	\centering
	\resizebox{\linewidth}{!}{
	\begin{tabular}{@{}c@{}c@{}c@{}c@{}}
	$M_1(x)$&$M_2(x)$&$M_3(x)$&$M_4(x)$\\
			\begin{subfigure}[b]{0.4\linewidth}
	\includegraphics[width=\linewidth]{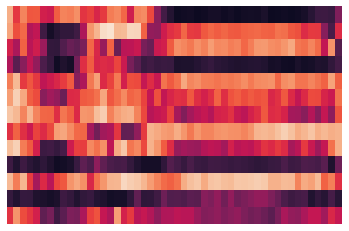}
	\end{subfigure}&
	\begin{subfigure}[b]{0.4\linewidth}
	\includegraphics[width=\linewidth]{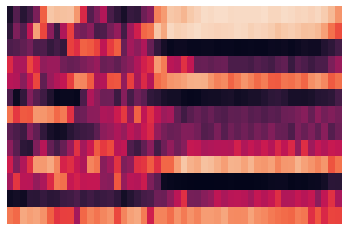}
	\end{subfigure}&
	\begin{subfigure}[b]{0.4\linewidth}
	\includegraphics[width=\linewidth]{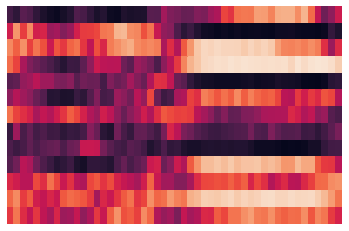}
	\end{subfigure}&
	\begin{subfigure}[b]{0.4\linewidth}
	\includegraphics[width=\linewidth]{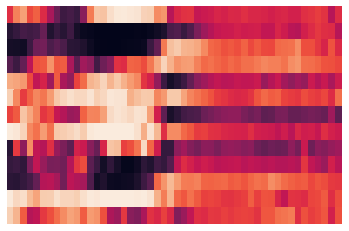}
	\end{subfigure}
\end{tabular}}
	\caption{\gls{ntl} learns dissimilar masks for \gls{sad} spectrograms. Dark horizontal lines indicate where $M_1$ and $M_2$ mask out frequency bands almost entirely, while the bright spot in the middle left part of $M_4$ indicates that this mask brings the intermediate frequencies in the first half of the recording into focus.}
		\label{fig:masks}
\end{figure}

%% file: tables/ts_n_vs_rest.tex
\begin{table*}[t!]
	\caption{Average AUC with standard deviation for $n$-vs-rest ($n=N-1$) anomaly detection on time series datasets}
	\label{tab:ts_rest-vs-all}
	\centering
	\resizebox{\linewidth}{!}{
	\begin{tabular}{l|ccc|ccccccc|c}
        \hline
        &\acrshort{ocsvm}&\acrshort{if}&\acrshort{lof}&\acrshort{rnn}&\acrshort{lstm} &\acrshort{svdd} &\acrshort{dagmm} &\acrshort{goad} &\acrshort{drocc}&fixed Ts&\acrshort{ntl}\\
		\hline
\acrshort{sad}  &60.2 & 56.9&\textbf{93.1} &53.0$\pm$0.1 &58.9$\pm$0.5 &59.7$\pm$0.5 &49.3$\pm$0.8 & 70.5$\pm$1.4 &58.8$\pm$0.5 &74.8$\pm$1.3 &85.1$\pm$0.3\\
\acrshort{natops} &57.6& 56.0 & 71.2 & 65.6$\pm$0.4 &56.9$\pm$0.7 &59.2$\pm$0.8&53.2$\pm$0.8 &61.5$\pm$0.7 &60.7$\pm$1.6 &70.8$\pm$1.3& \textbf{74.8}$\pm$0.9\\
\acrshort{ct}   &57.8 &57.9 &\textbf{90.3} &55.7$\pm$0.8 &50.9$\pm$1.2 &54.4$\pm$0.7 &47.5$\pm$2.5 &81.1$\pm$0.1&57.6$\pm$1.5 &63.0$\pm$0.6 & 87.4$\pm$0.2\\
\acrshort{epilepsy}  &50.2 & 55.3 &54.7 &74.9$\pm$1.5 &56.8$\pm$2.1 &52.9$\pm$1.4 & 52.0$\pm$1.0 &62.7$\pm$0.9&55.5$\pm$1.9  &69.8$\pm$1.6  &\textbf{80.5}$\pm$1.0 \\
\acrshort{rs} &55.9&58.4&59.4&75.8$\pm$0.9&63.1 $\pm$0.6&62.2$\pm$2.1&47.8$\pm$3.5&68.2$\pm$0.9&60.9$\pm$0.2  &\textbf{81.6}$\pm$1.2&80.0$\pm$0.4\\
        \hline
	\end{tabular}}
\end{table*}

%% file: figures/n_vs_rest.tex
\begin{figure}
    \centering
	\begin{subfigure}[b]{\linewidth}
	\captionsetup[subfigure]{labelformat=empty}
		\includegraphics[width=\linewidth]{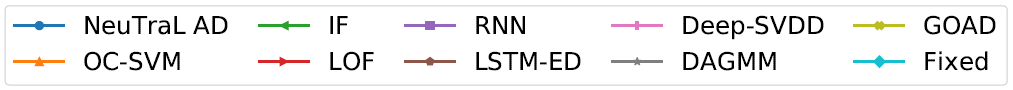}
	\end{subfigure}\\

	\begin{subfigure}[b]{0.52\linewidth}
		\includegraphics[width=\linewidth]{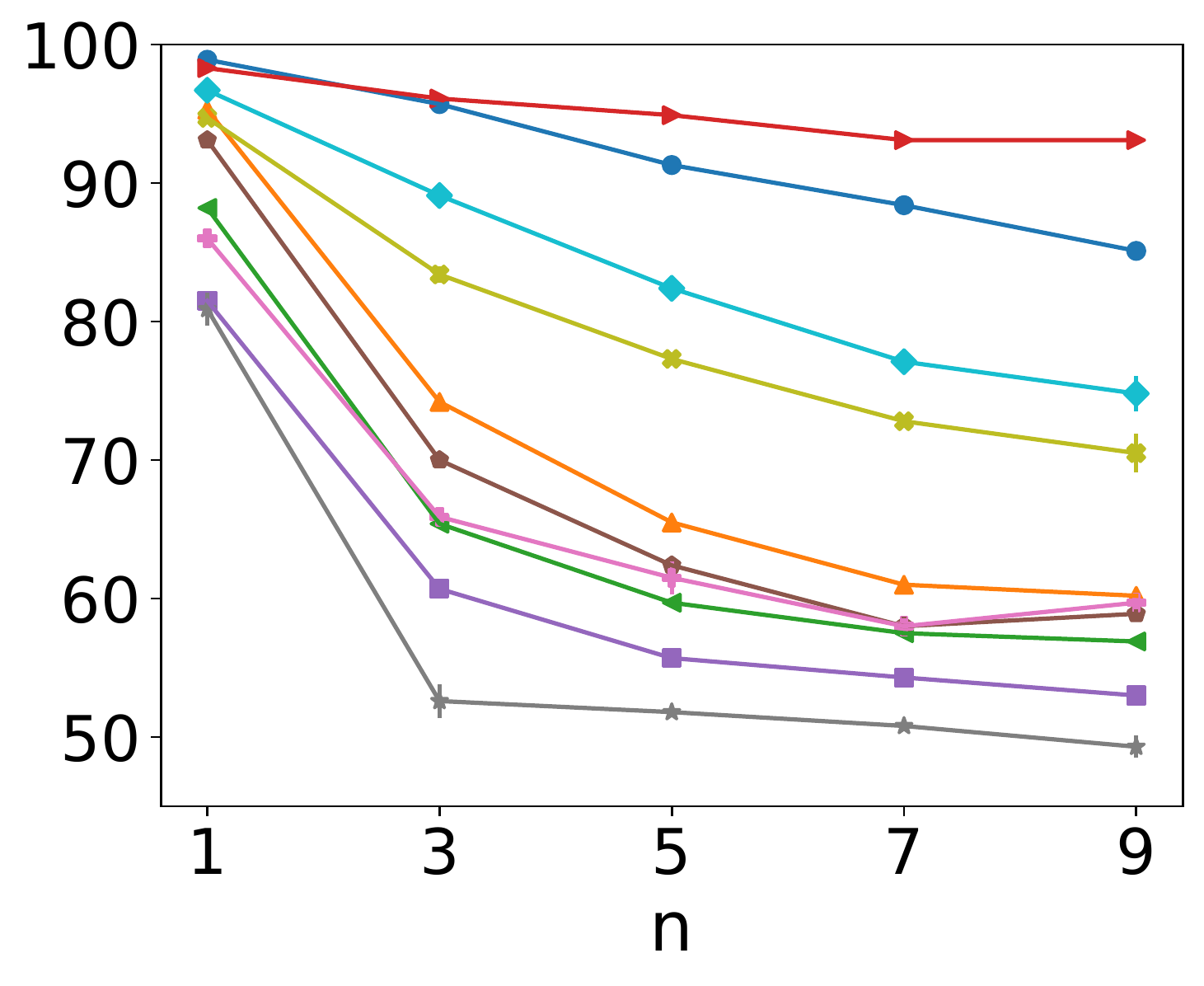}
    \caption{AUCs on SAD}
	\end{subfigure}
	\begin{subfigure}[b]{0.465\linewidth}
		\includegraphics[width=\linewidth]{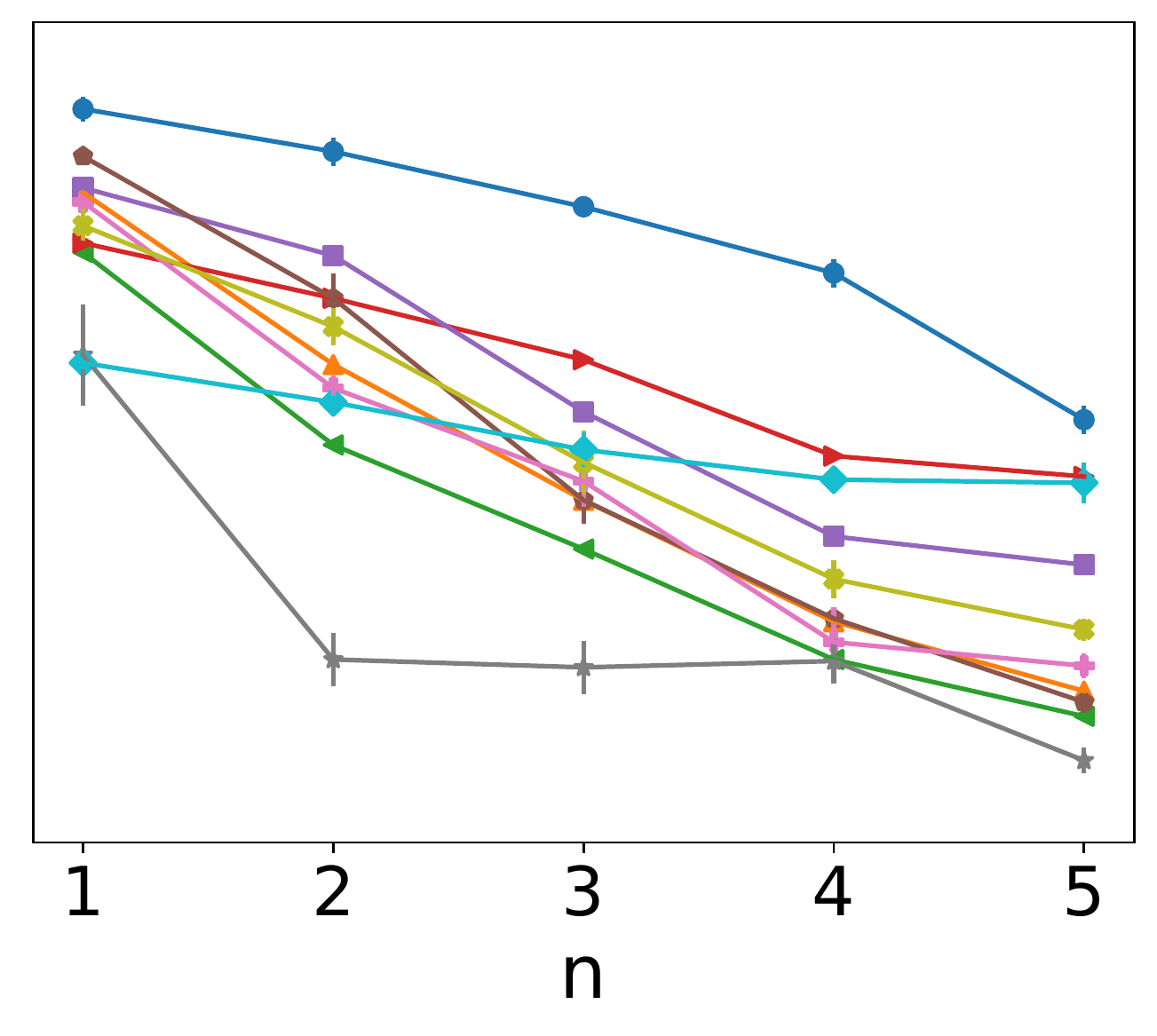}
    \caption{AUCs on NATOPS}
	\end{subfigure}

    \caption{AUC result of $n$-vs-all experiments on \gls{sad} and \gls{natops} with error bars (barely visible due to significance). \gls{ntl} outperforms all baselines on \gls{natops} and all deep learning baselines on \gls{sad}. \gls{lof}, a method based on $k$-nearest neighbors, outperforms \gls{ntl}, when $n > 3$ on \gls{sad}.}
    \label{fig:n-vs-all}
    \vspace{-5pt}
\end{figure}

%% file: tables/tab_results.tex
\begin{table}[t!]
\vspace{-6pt}
	\caption{F1-score ($\%$) with standard deviation for anomaly detection on tabular datasets (choice of F1-score consistent with prior work).}
	\label{tab:tab_one-vs-all}
	\centering
	\resizebox{\linewidth}{!}{
	\begin{tabular}{l|cccc}
        \hline
		     & Arrhythmia   & Thyroid   &KDD &KDDRev \\
		\hline
        \acrshort{ocsvm} & 45.8 & 38.9 & 79.5 &83.2\\
        \acrshort{if} &57.4&46.9&90.7&90.6 \\
        \acrshort{lof} & 50.0 &52.7 &83.8 &81.6\\
        \hline
        \acrshort{svdd} &53.9$\pm$3.1 &70.8$\pm$1.8 &99.0$\pm$0.1 & 98.6$\pm$0.2 \\        
        \acrshort{dagmm}  &49.8 & 47.8 &93.7&93.8 \\
        \acrshort{goad}   &52.0$\pm$2.3 &74.5$\pm$1.1 &98.4$\pm$0.2 &98.9$\pm$0.3\\
        \acrshort{drocc}& 46&27&-&-\\
        \hline
        \acrshort{ntl} &\textbf{60.3}$\pm$1.1 &\textbf{76.8}$\pm$1.9 &\textbf{99.3}$\pm$0.1 &\textbf{99.1}$\pm$0.1  \\
        \hline
	\end{tabular}}
	\vspace{-5pt}
\end{table}

%% file: figures/ablation_fig.tex
\begin{figure}[t!]
    \centering
	\begin{subfigure}[b]{0.49\linewidth}
		\includegraphics[width=\linewidth]{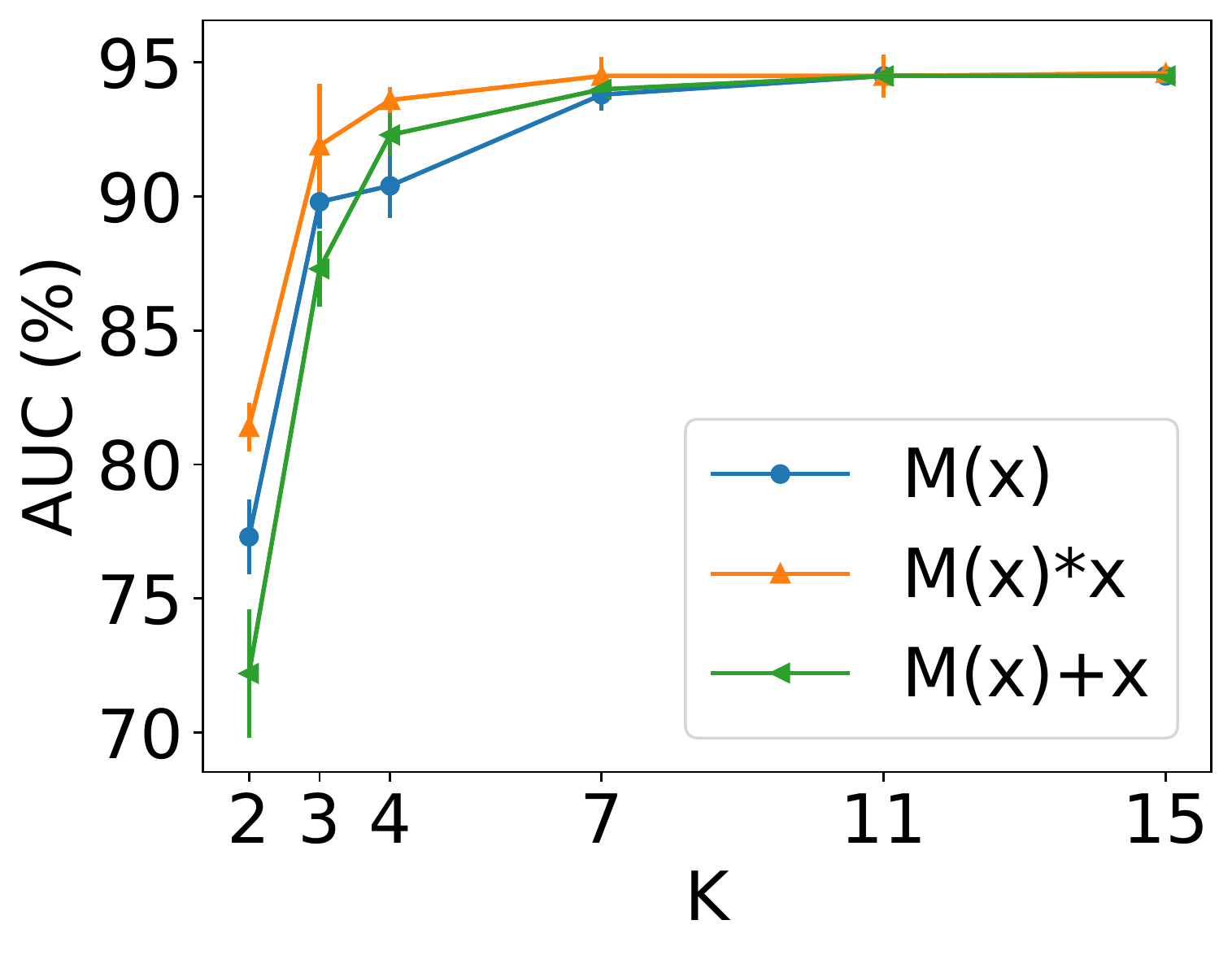}
	\caption{AUC on \gls{natops}}
	\end{subfigure}
	\begin{subfigure}[b]{0.49\linewidth}
		\includegraphics[width=\linewidth]{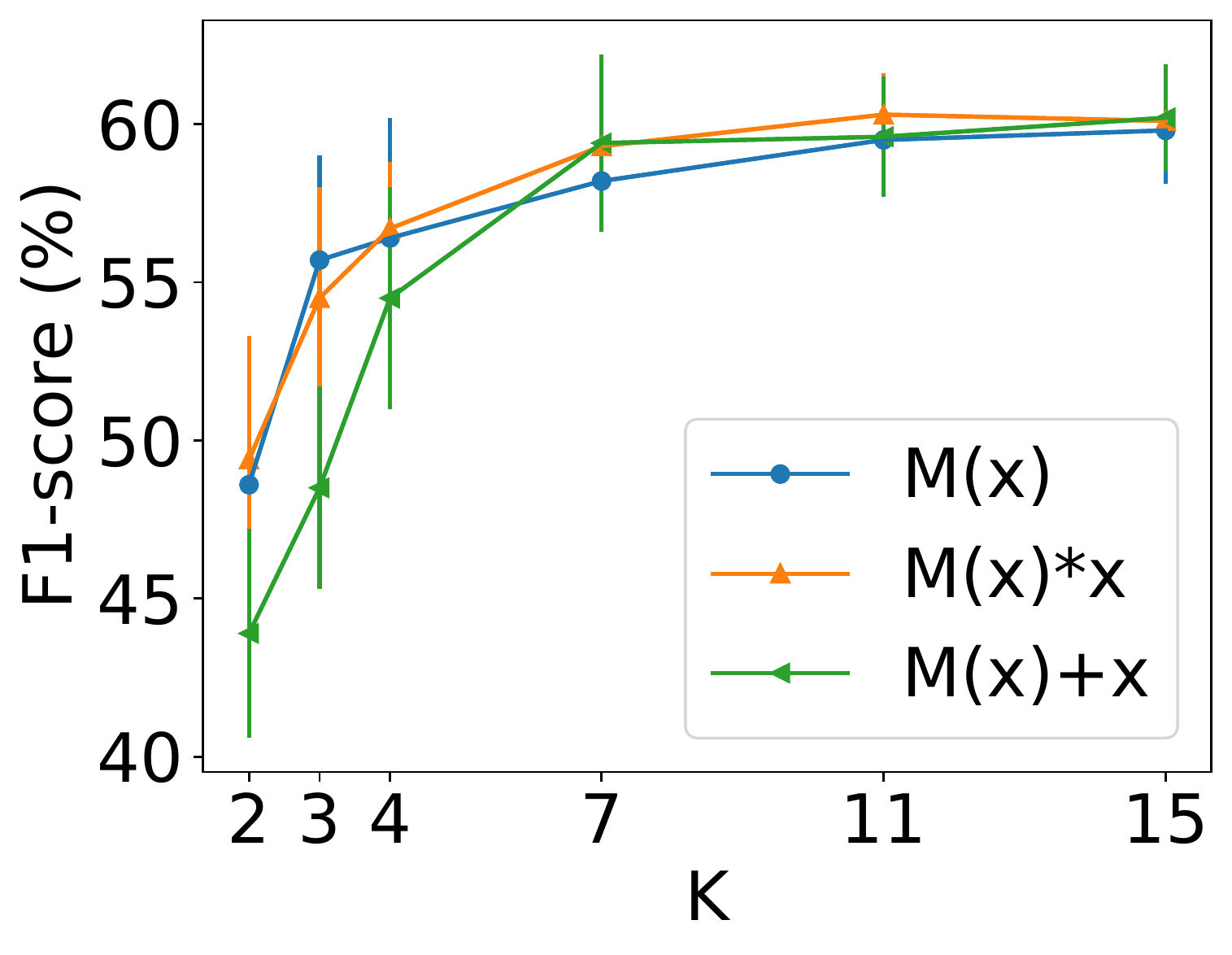}
    \caption{F1-score on Arrhythmia}
	\end{subfigure}
\vspace{-10pt}
    \caption{The outlier detection accuracy in terms of AUC of \gls{ntl} on \gls{natops} and in terms of F1-score of \gls{ntl} on Arrhythmia increases as the number of transformations $K$ increases, but stabilizes when a certain threshold is reached ($K>\approx 10$). With enough transformations, \gls{ntl} is robust to the parametrization of the transformations.}
    \label{fig:ablation_fig}
\vspace{-5pt}
\end{figure}

%% file: conclusion.tex
\section{Conclusion}\label{sec:conclusion}
We propose a self-supervised anomaly detection method with learnable transformations. The key ingredient is a novel training objective based on a \acrlong{dcl}, which encourages the learned transformations to produce diverse views that each share semantic information with the original sample, while being dissimilar. This unleashes the power of self-supervised anomaly detection to various data types including time series and tabular data. Our extensive empirical study finds, that on these data types, learning transformations and detecting outliers with \gls{ntl} improves over the state of the art.

%% file: Acknowledgements.tex
\section*{Acknowledgements}
MK acknowledges support by the Carl-Zeiss Foundation, by the German Research Foundation (DFG) award KL 2698/2-1, and by the Federal Ministry of Science and Education (BMBF) awards 01IS18051A and 031B0770E.

SM was supported by the National Science Foundation (NSF) under Grants 1928718, 2003237 and 2007719, by Qualcomm, and by the Defense Advanced Research Projects Agency (DARPA) under Contract No. HR001120C0021. Any opinions, findings and conclusions or recommendations expressed in this material are those of the author(s) and do not necessarily reflect the views of DARPA.

We acknowledge helpful discussions with Dennis Wagner and Luis Augusto Weber Mercado’s contribution to \Cref{fig:schema} and thank Sachin Goyal and Prateek Jain for being impressively responsive by email and for updating their experimental results.

%% file: appendix.tex
\section{Proofs for \Cref{sec:ntl_obs}}
\label{sec:appendix_proof}
In this Appendix, we prove the propositions in \Cref{sec:ntl_obs}.
The point is to compare various potential loss functions for neural transformation learning in their ability to produce useful transformations for self-supervised anomaly detection.

\Cref{req:sem,req:div} formalize what we consider useful transformations. The learned transformations should produce diverse views the share semantic information with the original sample.

We give two example edge-cases that violate the requirements, the `constant' edge-case in which the transformed views do no longer depend on the original sample (this violates the semantic requirement) and the `identity' edge-case in which all transformations reproduce the original sample perfectly but violate the diversity requirement.

We compare what happens to various losses for self-supervised anomaly detection under these edge cases. Specifically, we compare the loss of our method $\mathcal{L}$ (\Cref{eqn:contrastive_loss}) to the transformation prediction loss $\mathcal{L}_P$ (\Cref{eqn:lp}) of \citet{wang2019effective}, and the SimCLR loss $\mathcal{L}_C$ (\Cref{eqn:lc}, \citet{chen2020simple}), which has been used for anomaly detection in \citet{sohn2020learning}  and \citet{tack2020csi}.
In the original sources these losses have been used with fixed transformations (typically image transformations like rotations, cropping, blurring, etc.). Here we consider the same losses but with the learnable transformations parameterized as defined in \Cref{sec:neutralad}. All notation has been defined in \Cref{sec:neutralad}.

\subsection{Proof of Proposition 1}

We first investigate whether we can optimize the transformation prediction loss \Cref{eqn:lp} with respect to the transformation parameters $\theta_k$ and the parameters of $f_\phi$ and obtain learned transformations that fulfill  \Cref{req:sem,req:div}.

The proposition below states that a constant edge-case can achieve a global minimum of \Cref{eqn:lp}, which means that minimizing it can produce transformations the violate \Cref{req:sem}.

\begin{prop}
	\label{ac1}
	The `constant' edge-case $f_\phi(T_k(x)) = Cc_k$, where $c_k$ is a one-hot vector encoding the $k^{th}$ position (i.e. $c_{kk} = 1$) tends towards the minimum of $\mathcal{L}_P$ (\Cref{eqn:lp}) as the constant $C$ goes to infinity.
\end{prop}
\begin{proof}
	As a negative log probability $\mathcal{L}_P\geq 0$ is lower bounded by 0. We want to show that with $f_\phi(T_k(x)) = Cc_k$, (where $c_k$ is a one hot vector and $C$ is a constant,) $\mathcal{L}$ goes to 0 as $C$ goes to infinity. 
	Plugging $f_\phi(T_k(x)) = Cc_k$ into $\mathcal{L}_p$ and taking the limit yields
	\begin{align}
	\lim_{C \to \infty} \mathcal{L}_P &= \lim_{C \to \infty} \mathbb{E}_{x\sim\mathcal{D}}[- \sum_{k=1}^K \log \frac{\exp C}{\exp C + K - 1}] = \lim_{C \to \infty} - K \log \frac{\exp C}{\exp C + K - 1}\nonumber \\&=  \lim_{C \to \infty} - KC + K\log(\exp C + K - 1) = 0 \nonumber
	\end{align}
\end{proof}

\subsection{Proof of Proposition 2}

Next, we ask what would happen if we optimized the SimCLR loss  $\mathcal{L}_C$ (\Cref{eqn:lc}) with respect to the transformation parameters and the encoder. 

The result is, that if we allowed the encoder $f_\phi$ to be as flexible as necessary to achieve a global minimum of $\mathcal{L}_C$, then we can derive another minimum of $\mathcal{L}_C$ that relies only on identity transformations, thereby obtaining a solution to the minimization problem that violates the diversity requirement.

\begin{prop}
	\label{ac2}
	The `identity' edge-case $T_k(x) = x$ with adequate encoder $f_\phi$ is a minimizer of $\mathcal{L}_C$ (\Cref{eqn:lc}).
\end{prop}
\begin{proof}
	$\mathcal{L}_c(\mathcal{M})$ can be separated as the alignment term and the uniformity term.
	\begin{align}
	\mathcal{L}_c(\mathcal{M})&=\underbrace{\sum_{i=1}^N\left[-\log h(x_1^{(i)},x_2^{(i)})-\log h(x_2^{(i)},x_1^{(i)})\right]}_{\mathcal{L}_{\text{alignment}}}\\
	&+\underbrace{\sum_{i=1}^N \left[\log \left[ \sum_{j=1}^N h(x_1^{(i)}, x_2^{(j)})+\sum_{j=1}^N\mathds{1}_{[j\neq i]} h(x_1^{(i)}, x_1^{(j)}) \right]
		+\log \left[ \sum_{j=1}^N h(x_2^{(i)}, x_1^{(j)})+\sum_{j=1}^N\mathds{1}_{[j\neq i]} h(x_2^{(i)}, x_2^{(j)}) \right]\right]}_{\mathcal{L}_{\text{uniformity}}}\,.\nonumber
	\end{align}
	
	A sufficient condition of $\min(\mathcal{L}_c(\mathcal{M}))$ is both $\mathcal{L}_{\text{alignment}}$ and $\mathcal{L}_{\text{uniformity}}$ are minimized.
	\begin{align}
	\min(\mathcal{L}_c(\mathcal{M})) \geq \min(\mathcal{L}_{\text{alignment}})+ \min(\mathcal{L}_{\text{uniformity}})\,.
	\end{align}
	
	Given an adequate encoder $f^*_\phi$, that is flexible enough to minimize both  $\mathcal{L}_{\text{alignment}}$ and $\mathcal{L}_{\text{uniformity}}$ for all transformation pairs $T_1$ and $T_2$, we will show we can construct another solution to the minimization problem that relies only on identity transformations.
	
	The alignment term is only minimized for all $T_1$, $T_2$, if $f^*_{\phi}(T_1(x^{(i)})) = f^*_{\phi}(T_2(x^{(i)}))$  for all $x^{(i)}\sim \mathcal {M}$. So we know for $f^*_{\phi}$ that
	\begin{align}
	f^*_{\phi} = \argmin_{f_\phi}\mathcal{L}_{\text{alignment}} \quad &\iff \quad  \text{sim}(f^*_{\phi}(T_1(x^{(i)})),f^*_{\phi}(T_2(x^{(i)})))  = 1 \quad  \forall \, x^{(i)}\sim \mathcal {M} \\ &\iff \quad f^*_{\phi}(T_1(x^{(i)})) = f^*_{\phi}(T_2(x^{(i)})) \quad \forall \, x^{(i)}\sim \mathcal {M}.
	\end{align}
	Define $\tilde{f}_{\phi} = f^*_{\phi}\circ T_1$. Since $f^*_{\phi}(T_1(x^{(i)})) = f^*_{\phi}(T_2(x^{(i)}))$,
	\begin{align}
	\tilde{f}_{\phi}(\mathbb{I}(x^{(i)})) = f^*_{\phi}(T_1(x^{(i)})) = f^*_{\phi}(T_2(x^{(i)})) \quad  \forall \, x^{(i)}\sim \mathcal {M} .
	\end{align}
	Using only the identity transformation $\mathbb{I}(x) = x$ for $T_1$ and $T_2$, and $\tilde{f}_{\phi}$ as the encoder in $\mathcal{L}_C$ yields the same minimal loss as under $T_1$, $T_2$ and $f^*_{\phi}$.
\end{proof}

\subsection{Proof of Proposition 3}

Finally, we investigate the effect of the edge-cases from \Cref{c1,c2} on our objective \Cref{eqn:contrastive_loss}.
\begin{prop}
	\label{ac3}
	The edge-cases of \Cref{c1,c2} do not minimize $\mathcal{L}$ (\gls{dcl}, \Cref{eqn:contrastive_loss}).
\end{prop}
We divide the proposition and its proof into two parts.
\begin{prt}
	The `constant' edge-case $f_\phi(T_k(x)) = Cc_k$, where $c_k$ is a one-hot vector encoding the $k^{th}$ position (i.e. $c_{kk} = 1$) does not minimize $\mathcal{L}$ (\gls{dcl}, \Cref{eqn:contrastive_loss}) for any $C$, also not as $C$ tends to infinity.
\end{prt}

\begin{proof}
	For simplicity, we define the embeddings obtained by the transformations as $z_k := f_{\phi}(T_k(x))$ for $k=1,\cdots,K$ and $z_0:=f_{\phi}(x)$.
	We prove that the gradient of $\mathcal{L}(x)$ with respect to embeddings $\nabla \mathcal{L}(x)= [\frac{\partial \mathcal{L}(x)}{\partial z_0},\cdots,\frac{\mathcal{L}(x)}{\partial z_K}]^T \neq 0$ at $z_{1:K}=Cc_{1:K}$, where $c_k$ is a one-hot vector encoding the $k^{th}$ position (i.e. $c_{k,k} = 1$). 
	Using the chain rule, the partial derivative $\frac{\partial \mathcal{L}(x)}{\partial z_{n}}\, \forall \,n\in\{0,\cdots,K\}$ can be factorized as 
	\begin{align}
	\frac{\partial \mathcal{L}(x)}{\partial z_{n}} = \sum_{k=1}^K \frac{\sum_{l\in\{1,\cdots,K\}/ \{k\}}h(z_k, z_l)(\frac{\partial \text{sim}(z_k,z_l)}{\partial z_{n}/\|z_{n}\|}-\frac{\partial \text{sim}(z_k,z_0)}{\partial z_{n}/\|z_{n}\|})}{ \sum_{l\in\{0,\cdots,K\}/ \{k\}}h(z_k, z_l)}\frac{I- z_nz_n^T/\|z_{n}\|^2}{\|z_{n}\|}.
	\end{align}
	We define $A_n:= \frac{I- z_nz_n^T/\|z_{n}\|^2}{\|z_{n}\|}$, and plug in $z_{1:K}=Cc_{1:K}$ and $\text{sim}(a,b) = a^T b/ \|a\| \|b\|$.
	\begin{align}
	\frac{\partial \mathcal{L}(x)}{\partial z_{n}} \at[\big]{z_{1:K}=Cc_{1:K}} = \sum_{k\in\{1,\cdots,K\}/ \{n\}}\frac{c_k^TA_n-z_0^TA_n/ \|z_0\|}{h(Cc_{n},z_0)+K-1}+  \frac{c_k^TA_n}{h(Cc_k,z_0)+K-1}.
	\end{align}
	As $C$ is finite, by assuming the $k$th ($k\neq n$) entry of $\frac{\partial \mathcal{L}(x)}{\partial z_{n}}$ equals zero at $z_{1:K}=Cc_{1:K}$ we get the $k$th entry of $z_0$
	\begin{align}
	\label{eqn:conflict_1}
	z_{0,k}^T = \frac{\|z_0\|}{K-1}(c_{k,k}^T+ c_{k,k}^T\frac{h(Cc_n,z_0)+K-1}{h(Cc_k,z_0)+K-1}).
	\end{align}
	Similarly, by assuming $k$th entry of $\frac{\partial \mathcal{L}(x)}{\partial z_{m}}$ equals zero at $z_{1:K}=Cc_{1:K}$ with $m\neq n$ and $m\neq k$ we have
	\begin{align}
	\label{eqn:conflict_2}
	z_{0,k}^T = \frac{\|z_0\|}{K-1}(c_{k,k}^T+ c_{k,k}^T\frac{h(Cc_m,z_0)+K-1}{h(Cc_k,z_0)+K-1}).
	\end{align}
	As \Cref{eqn:conflict_1} and \Cref{eqn:conflict_2} are equal, we have $h(Cc_n,z_0)=h(Cc_m,z_0)$. Since this should be hold by assuming any entry of any partial derivative equals zero, by induction we have 
	\begin{align}
	\label{eqn:z_0}
	h(Cc_k,z_0)=r, \quad z_{0,k}^T/\|z_0\| = \frac{2}{K-1} \quad \forall\, k\in\{1,\cdots,K\}.
	\end{align}
	By plugging in $h(Cc_k,z_0)=r$ and $z_{1:K}=Cc_{1:K}$ to $\frac{\partial \mathcal{L}(x)}{\partial z_{0}}$ we have
	\begin{align}
	\label{eqn:pd_z0}
	\frac{\partial \mathcal{L}(x)}{\partial z_{0}} \at[\big]{h(Cc_k,z_0)=r,z_{1:K}=Cc_{1:K}} =\frac{1-K}{r+K-1}\sum_{k=1}^K c_k^T\frac{I- z_0z_0^T/\|z_{0}\|^2}{\|z_{0}\|}
	\end{align}
	\Cref{eqn:pd_z0} equals zero, if and only if every entry in the resulting vector equals zero. By assuming the $k$th ($k\in\{1,\cdots,K\}$) entry of $\frac{\partial \mathcal{L}(x)}{\partial z_{0}}$ equals zero at \Cref{eqn:z_0} and $z_{1:K}=Cc_{1:K}$, we have  
	\begin{align}
	1-K(\frac{2}{K-1})^2=0\,,
	\end{align}
	which leads to a non-integral value of $K$. Since $K$ is the number of transformations and is defined as an integer, we have \Cref{eqn:pd_z0} is not zero.  
	Therefore, $\nabla \mathcal{L}(x) = [\frac{\partial \mathcal{L}(x)}{\partial z_0},\cdots,\frac{\mathcal{L}(x)}{\partial z_K}]^T \neq 0$ at $z_{1:K}=Cc_{1:K}$.
\end{proof}

\begin{prt}
	The `identity' edge-case $T_k(x) = x$ does not minimize $\mathcal{L}$ (\Cref{eqn:contrastive_loss}) for adequate encoder $f_{\phi}$.
\end{prt}

\begin{proof}
	Plugging $T_k(x) = x$ for all $k$ into \Cref{eqn:contrastive_loss}
	\begin{align}
	\mathcal{L}=&\mathbb{E}_{x \sim \mathcal{D}}[-\sum_{k=1}^K \log \frac{h(x, x)}{K h(x, x)}] = K \log K.
	\end{align}
	This is $K$ times the cross-entropy of the uniform distribution, meaning that using the identity transformation is equivalent to random guessing for the task of the \gls{dcl}, which is to predict which sample is the original given a transformed view. 
	When $f_{\phi}$ is adequate (i.e. flexible enough) we can do better than random on $\mathcal{L}$. This can be seen in the anomaly scores in  \Cref{fig:hist_trained} which are much smaller than $K \log K$ after training (better than random).
	
	We can also construct examples which achieve better than random performance. For example, for $K=2$, taking $z_1 \perp z$, $z_2 \perp z$, and $z_1 =-z_2$, does better than random.  
	The loss values (with $\tau=1$) of the two cases are
	\begin{itemize}
		\item `Identity' edge-case: $z_1=z_2=z$, so $\mathcal{L}(x)=-2\log(0.5)=1.386$.
		\item Counterexample: 
		$z_1 \perp z$, $z_2 \perp z$, and $z_1 =-z_2$, so $\mathcal{L}(x)=-2\log(1/(1+\exp(-1)))=0.627$.
	\end{itemize}
	The counterexample achieves a lower loss value than the `identity' edge-case. So the `identity' edge-case is not the minimum of $\mathcal{L}(x)$.
\end{proof}

\section{Implementation details}
\label{sec:appendix_implement}
\subsection{Implementations of \gls{ntl} on time series datasets}
The networks in the neural transformations used in all experiments consist of one convolutional layer on the bottom, a stack of three residual blocks of 1d convolutional layers with instance normalization layers and ReLU activations, as well as one convolutional layer on the top. All convolutional layers are with the kernel size of 3, and the stride of 1. All bias terms are fixed as zero, and the learnable affine parameters of the instance normalization layers are frozen. The dimension of the residual blocks is the data dimension. The convolutional layer on the top has an output dimension as the data dimension. For the multiplicative parameterization, a sigmoid activation is added to the end.

The encoder used in all experiments consists of residual blocks of 1d convolutional layers with ReLU activations, as well as one 1d convolutional layer on the top of all residual blocks.
The detailed network structure (from bottom to top) in each time series dataset is:
\begin{itemize}[leftmargin=*]
	\item \gls{sad}: (i) one residual block with the kernel size of 3, the stride of 1, and the output dimension of 32. (ii) four residual blocks with the kernel size of 3, the stride of 2, and the output dimensions of 32, 64, 128, 256. (iii) one 1d convolutional layer with the kernel size of 6, the stride of 1, and the output dimension of 32.
	
	\item \gls{natops}: (i) one residual block with the kernel size of 3, the stride of 1, and the output dimension of 32. (ii) four residual blocks with the kernel size of 3, with the stride of 2, and the output dimensions of 32, 64, 128, 256. (iii) one 1d convolutional layer with the kernel size of 4, the stride of 1, and the output dimension of 64.
	
	\item \gls{ct}: (i) one residual block with the kernel size of 3, the stride of 1, and the output dimension of 32. (ii) six residual blocks with the kernel size of 3, the stride of 2, and the output dimensions of 32, 64, 128, 256, 512, 1024. (iii) one 1d convolutional layer with the kernel size of 3, the stride of 1, and the output dimension of 64.
	
	\item \gls{epilepsy}: (i) one residual block with the kernel size of 3, the stride of 1, and the output dimension of 32. (ii) six residual blocks with the kernel size of 3, the stride of 2, and the output dimensions of 32, 64, 128, 256, 512, 1024. (iii) one 1d convolutional layer with the kernel size of 4, the stride of 1, and the output dimension of 128.
	
	\item \gls{rs}: (i) one residual block with the kernel size of 3, the stride of 1, and the output dimension of 32. (ii) three residual blocks with the kernel size of 3, the stride of 2, and the output dimensions of 32, 64, 128. (iii) one 1d convolutional layer with the kernel size of 4, the stride of 1, and the output dimension of 64.
\end{itemize}

\subsection{Implementations of baselines}
\begin{itemize}[leftmargin=*]
	\item Traditional Anomaly Detection Baselines. \gls{ocsvm}, \gls{if}, and \gls{lof} are taken from scikit-learn library with default parameters.
	\item Deep Anomaly Detection Baselines. The implementations of \gls{svdd}, \gls{drocc}, and \gls{dagmm} are adopted from the published codes with a similar encoder as \gls{ntl}. \gls{dagmm} has a hyperparameter of the number of mixture components. We consider the number of components between 4 and 12 and select the best performing one. 
	\item Self-supervised Anomaly Detection Baselines. The implementation of \gls{goad} is taken from the published code. The results of \gls{goad} depend on the choice of the output dimension $r$ of affine transformations. We consider the reduced dimension $r\in \{2^2,2^3,...,2^6\}$, and select the best performing one. We craft specific time series transformations for the designed classification-based baseline. The hand-crafted transformations are the compositions of flipping along the time axis (true/false), flipping along the channel axis (true/false), and shifting along the time axis by 0.25 of its time length (forward/backward/none). By taking all possible compositions, we obtain a total of $2*2*3=12$ transformations.
	\item Anomaly Detection Baselines for Time Series. The \gls{rnn} is parameterized by two layers of recurrent neural networks, e.g. GRU, and a stack of two linear layers with ReLU activation on the top of it which outputs the mean and variance at each time step. The implementation of \gls{lstm} is taken from the web.
\end{itemize}

\section{Tabular datasets}
\label{sec:appendix_tab_data}
The four used tabular datasets are:
\begin{itemize}[leftmargin=*]
	\item Arrhythmia: A cardiology dataset from the UCI repository contains 274 continuous attributes and 5 categorical attributes. Following the data preparation of previous works, only 274 continuous attributes are considered. The abnormal classes include 3, 4, 5, 7, 8, 9, 14, and 15. The rest classes are considered as normal.
	\item Thyroid: A medical dataset from the UCI repository contains attributes related to hyperthyroid diagnosis. Following the data preparation of previous works, only 6 continuous attributes are considered. The hyperfunction class is treated as abnormal, and the rest 2 classes are considered as normal.
	\item KDDCUP: The KDDCUP99 10 percent dataset from the UCI repository contains 34 continuous attributes and 7 categorical attributes. Following the data preparation of previous works, 7 categorical attributes are represented by one-hot vectors. Eventually, the data has 120 dimensions. The attack samples are considered as normal, and the non-attack samples are considered as abnormal.
	\item KDDCUP-Rev: It is derived from the KDDCUP99 10 percent dataset. The non-attack samples are considered as normal, and attack samples are considered as abnormal. Following the data preparation of previous works, attack data is sub-sampled to consist of $25\%$ of the number of non-attack samples.
\end{itemize}



\section{Additional qualitative results}
\label{sec:appendix_qualitative}
\subsection{Results for time series data}

In \Cref{fig:sad_Ts}, we show the learned transformations (parameterized as $T(x) = M(x)\odot x$ and $K=4$) on spoken Arabic digits. The learned transformations given one normal example are shown in the first row. The learned transformations given one example of each abnormal class are shown in the following rows.
\input{figures/many_spectograms}

In \Cref{fig:na_Ts}, we show the learned transformations (parameterized as $T(x) = M(x)$ and $K=4$) on NATOPS. The learned transformations given one normal example are shown in the first row. The learned transformations given one example of each abnormal class are shown in the following rows.
\input{figures/many_lines}

\subsection{Results for tabular data}
\input{figures/explain}
The learned transformations of thyroid, which are visualized in \Cref{fig:tab_Ts}, offer us the possible explanations of why a data instance is an anomaly. We illustrate one normal example and three anomalies in the first row. Since the anomaly score \Cref{eqn:anomaly_score} is the sum of terms caused by each transformation. Each score shown in the last row has four bars indicating the terms caused by each transformation. The score of the normal example is very low, and its bars are invisible from the plot. The scores of three anomalies are mainly contributed by different terms (colored with orange). 
The four learned masks are colored blue and listed in four rows. $M_4$ focuses on checking the value of the fourth attribute and contributes high values to the scores of all listed anomalies. In comparison, $M_2$ is less useful for anomaly detection. \gls{ntl} is able to learn diverse transformations but is not guaranteed to learn transformations that are useful for anomaly detection, since no label is included in the training. 

We project the score terms of test data contributed by $T_1$, $T_3$, and $T_4$ to a simplex to visualize which transformation dominates the anomaly score in \Cref{fig:tab_score}. From the left subplot, we can see, the scores of normal data (blue) are not dominated by any single transformation, while the scores of anomalies are mainly dominated by $T_3$ and $T_4$. In the right subplot, we visualize the magnitudes of scores via transparency. We can see, the score magnitudes of normal data are clearly lower than the score magnitudes of anomalies.
\input{figures/score_simplex}

%% file: figures/many_spectograms.tex
\begin{figure}[ht]
	\captionsetup[subfigure]{labelformat=empty}
	\centering
	\begin{tabular}{@{}c@{}c@{}c@{}c@{}c@{}}
	$x$&$T_1(x)$&$T_2(x)$&$T_3(x)$&$T_4(x)$\\
			\begin{subfigure}[b]{0.18\linewidth}
	\includegraphics[width=\linewidth]{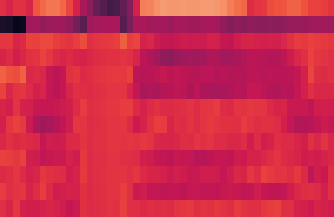}
	\end{subfigure}&
	\begin{subfigure}[b]{0.18\linewidth}
	\includegraphics[width=\linewidth]{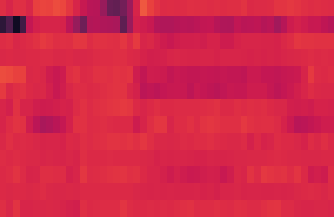}
	\end{subfigure}&
	\begin{subfigure}[b]{0.18\linewidth}
	\includegraphics[width=\linewidth]{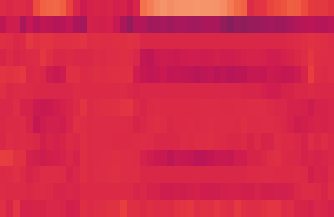}
	\end{subfigure}&
	\begin{subfigure}[b]{0.18\linewidth}
	\includegraphics[width=\linewidth]{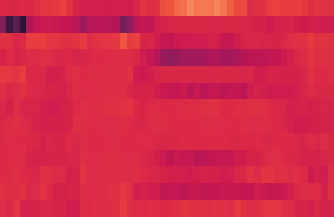}
	\end{subfigure}&
	\begin{subfigure}[b]{0.18\linewidth}
	\includegraphics[width=\linewidth]{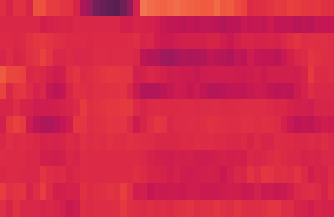}
	\end{subfigure}\\
	\hline
	\begin{subfigure}[b]{0.18\linewidth}
	\includegraphics[width=\linewidth]{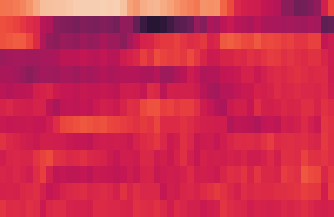}
	\end{subfigure}&
	\begin{subfigure}[b]{0.18\linewidth}
	\includegraphics[width=\linewidth]{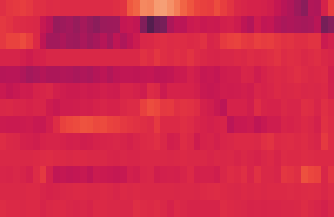}
	\end{subfigure}&
	\begin{subfigure}[b]{0.18\linewidth}
	\includegraphics[width=\linewidth]{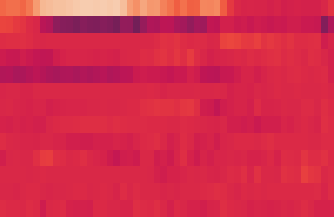}
	\end{subfigure}&
	\begin{subfigure}[b]{0.18\linewidth}
	\includegraphics[width=\linewidth]{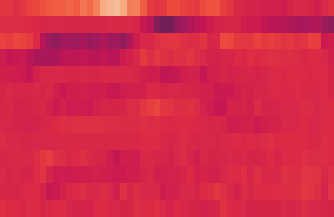}
	\end{subfigure}&
	\begin{subfigure}[b]{0.18\linewidth}
	\includegraphics[width=\linewidth]{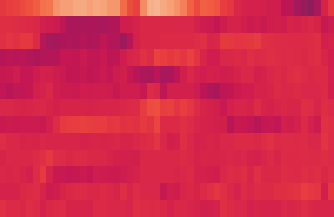}
	\end{subfigure}
\\
	\begin{subfigure}[b]{0.18\linewidth}
	\includegraphics[width=\linewidth]{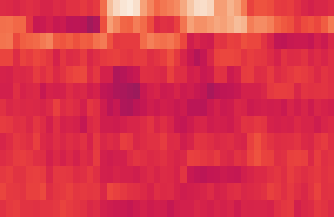}
	\end{subfigure}&
	\begin{subfigure}[b]{0.18\linewidth}
	\includegraphics[width=\linewidth]{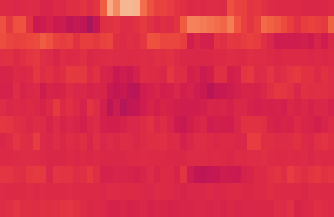}
	\end{subfigure}&
	\begin{subfigure}[b]{0.18\linewidth}
	\includegraphics[width=\linewidth]{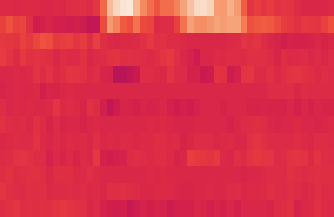}
	\end{subfigure}&
	\begin{subfigure}[b]{0.18\linewidth}
	\includegraphics[width=\linewidth]{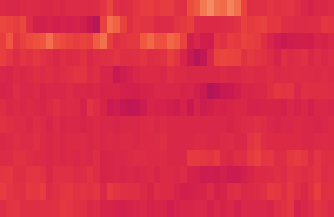}
	\end{subfigure}&
	\begin{subfigure}[b]{0.18\linewidth}
	\includegraphics[width=\linewidth]{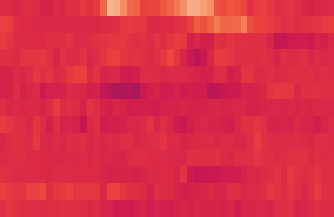}
	\end{subfigure}\\
	\begin{subfigure}[b]{0.18\linewidth}
	\includegraphics[width=\linewidth]{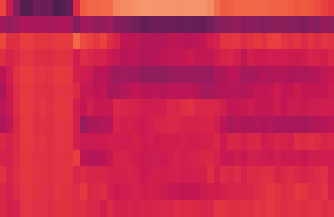}
	\end{subfigure}&
	\begin{subfigure}[b]{0.18\linewidth}
	\includegraphics[width=\linewidth]{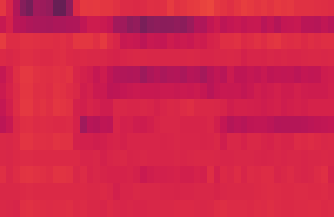}
	\end{subfigure}&
	\begin{subfigure}[b]{0.18\linewidth}
	\includegraphics[width=\linewidth]{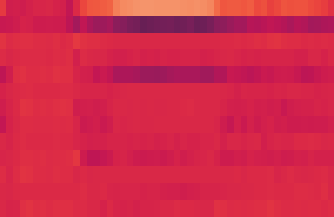}
	\end{subfigure}&
	\begin{subfigure}[b]{0.18\linewidth}
	\includegraphics[width=\linewidth]{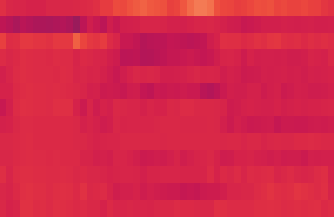}
	\end{subfigure}&
	\begin{subfigure}[b]{0.18\linewidth}
	\includegraphics[width=\linewidth]{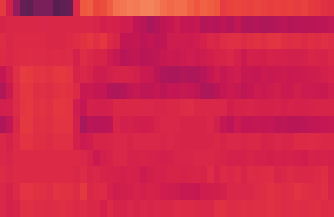}
	\end{subfigure}\\
	\begin{subfigure}[b]{0.18\linewidth}
	\includegraphics[width=\linewidth]{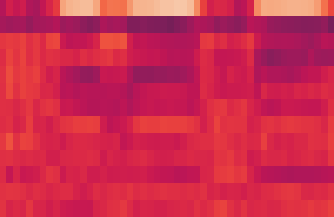}
	\end{subfigure}&
	\begin{subfigure}[b]{0.18\linewidth}
	\includegraphics[width=\linewidth]{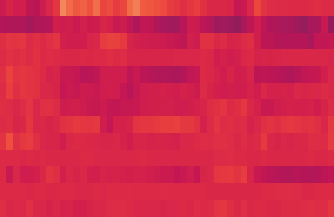}
	\end{subfigure}&
	\begin{subfigure}[b]{0.18\linewidth}
	\includegraphics[width=\linewidth]{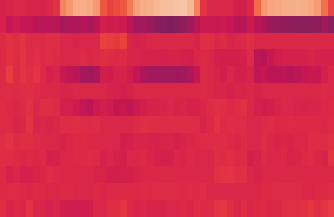}
	\end{subfigure}&
	\begin{subfigure}[b]{0.18\linewidth}
	\includegraphics[width=\linewidth]{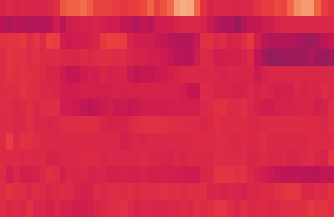}
	\end{subfigure}&
	\begin{subfigure}[b]{0.18\linewidth}
	\includegraphics[width=\linewidth]{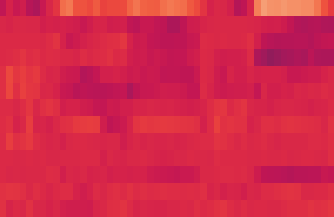}
	\end{subfigure}\\	
	\begin{subfigure}[b]{0.18\linewidth}
	\includegraphics[width=\linewidth]{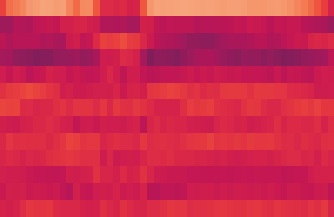}
	\end{subfigure}&
	\begin{subfigure}[b]{0.18\linewidth}
	\includegraphics[width=\linewidth]{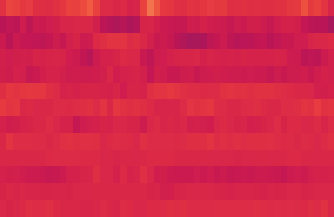}
	\end{subfigure}&
	\begin{subfigure}[b]{0.18\linewidth}
	\includegraphics[width=\linewidth]{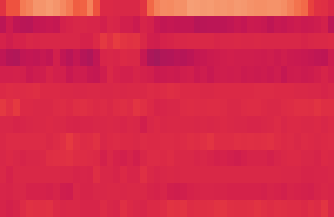}
	\end{subfigure}&
	\begin{subfigure}[b]{0.18\linewidth}
	\includegraphics[width=\linewidth]{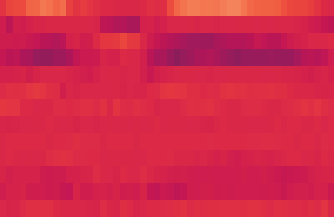}
	\end{subfigure}&
	\begin{subfigure}[b]{0.18\linewidth}
	\includegraphics[width=\linewidth]{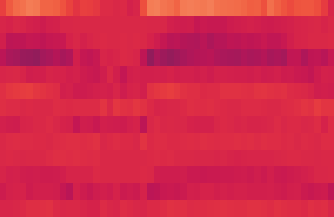}
	\end{subfigure}\\
		\begin{subfigure}[b]{0.18\linewidth}
	\includegraphics[width=\linewidth]{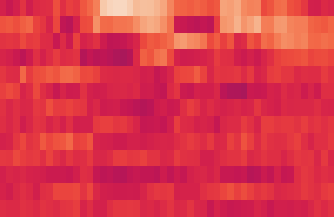}
	\end{subfigure}&
	\begin{subfigure}[b]{0.18\linewidth}
	\includegraphics[width=\linewidth]{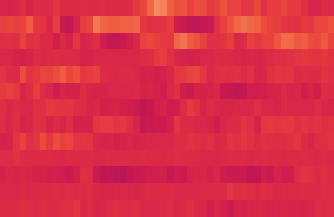}
	\end{subfigure}&
	\begin{subfigure}[b]{0.18\linewidth}
	\includegraphics[width=\linewidth]{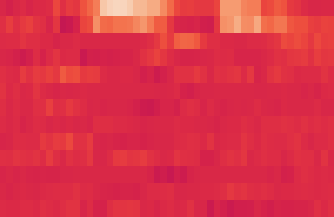}
	\end{subfigure}&
	\begin{subfigure}[b]{0.18\linewidth}
	\includegraphics[width=\linewidth]{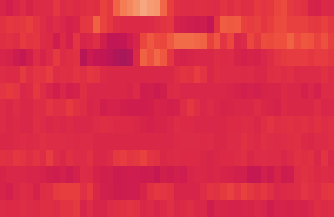}
	\end{subfigure}&
	\begin{subfigure}[b]{0.18\linewidth}
	\includegraphics[width=\linewidth]{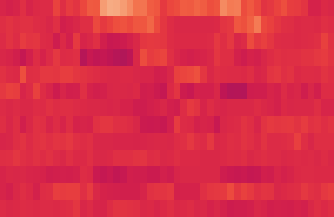}
	\end{subfigure}\\
			\begin{subfigure}[b]{0.18\linewidth}
	\includegraphics[width=\linewidth]{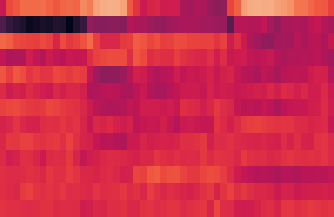}
	\end{subfigure}&
	\begin{subfigure}[b]{0.18\linewidth}
	\includegraphics[width=\linewidth]{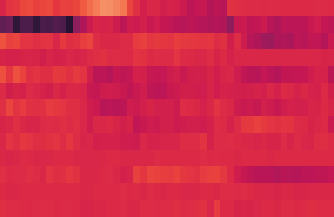}
	\end{subfigure}&
	\begin{subfigure}[b]{0.18\linewidth}
	\includegraphics[width=\linewidth]{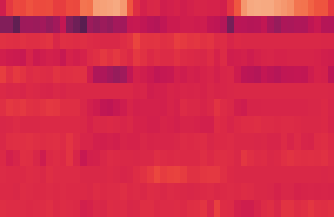}
	\end{subfigure}&
	\begin{subfigure}[b]{0.18\linewidth}
	\includegraphics[width=\linewidth]{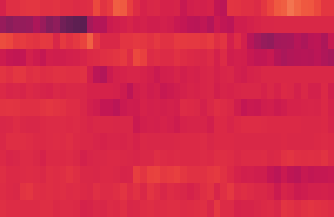}
	\end{subfigure}&
	\begin{subfigure}[b]{0.18\linewidth}
	\includegraphics[width=\linewidth]{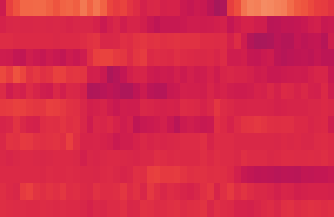}
	\end{subfigure}\\
			\begin{subfigure}[b]{0.18\linewidth}
	\includegraphics[width=\linewidth]{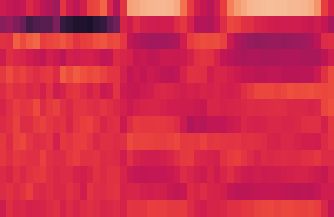}
	\end{subfigure}&
	\begin{subfigure}[b]{0.18\linewidth}
	\includegraphics[width=\linewidth]{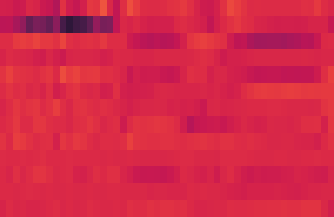}
	\end{subfigure}&
	\begin{subfigure}[b]{0.18\linewidth}
	\includegraphics[width=\linewidth]{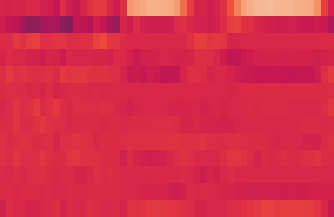}
	\end{subfigure}&
	\begin{subfigure}[b]{0.18\linewidth}
	\includegraphics[width=\linewidth]{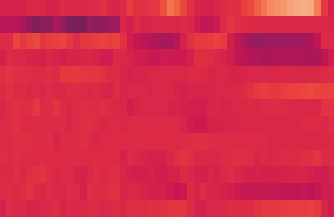}
	\end{subfigure}&
	\begin{subfigure}[b]{0.18\linewidth}
	\includegraphics[width=\linewidth]{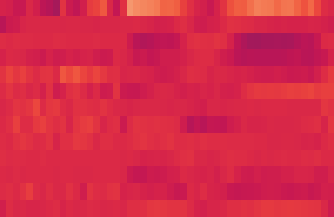}
	\end{subfigure}\\
			\begin{subfigure}[b]{0.18\linewidth}
	\includegraphics[width=\linewidth]{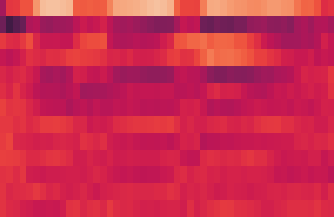}
	\end{subfigure}&
	\begin{subfigure}[b]{0.18\linewidth}
	\includegraphics[width=\linewidth]{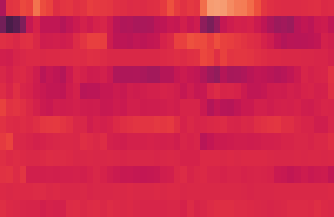}
	\end{subfigure}&
	\begin{subfigure}[b]{0.18\linewidth}
	\includegraphics[width=\linewidth]{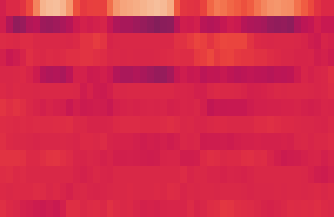}
	\end{subfigure}&
	\begin{subfigure}[b]{0.18\linewidth}
	\includegraphics[width=\linewidth]{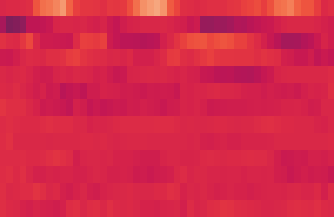}
	\end{subfigure}&
	\begin{subfigure}[b]{0.18\linewidth}
	\includegraphics[width=\linewidth]{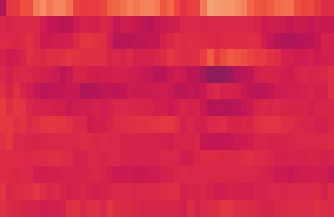}
	\end{subfigure}\\
	\end{tabular}
	\caption{The learned transformations ($T(x) = M(x)\odot x$) on \gls{sad}. The first row: the learned transformations of one given example of the normal class. The rest rows: the learned transformations of one given example of each abnormal class. }
		\label{fig:sad_Ts}
\end{figure}

%% file: figures/many_lines.tex
\begin{figure}[ht]
	\captionsetup[subfigure]{labelformat=empty}
	\centering
	\begin{tabular}{@{}c@{}c@{}c@{}c@{}c@{}}
	$x$&$T_1(x)$&$T_2(x)$&$T_3(x)$&$T_4(x)$\\
			\begin{subfigure}[b]{0.18\linewidth}
	\includegraphics[width=\linewidth]{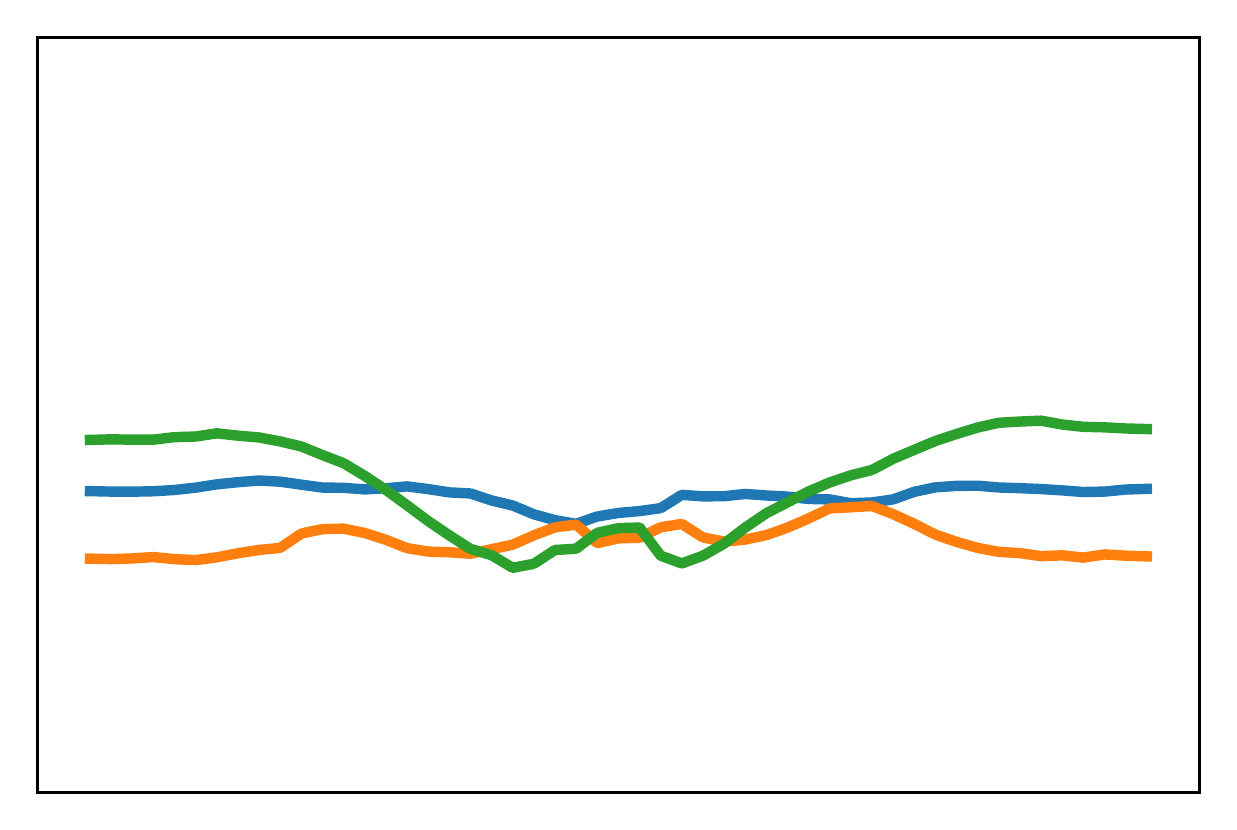}
	\end{subfigure}&
	\begin{subfigure}[b]{0.18\linewidth}
	\includegraphics[width=\linewidth]{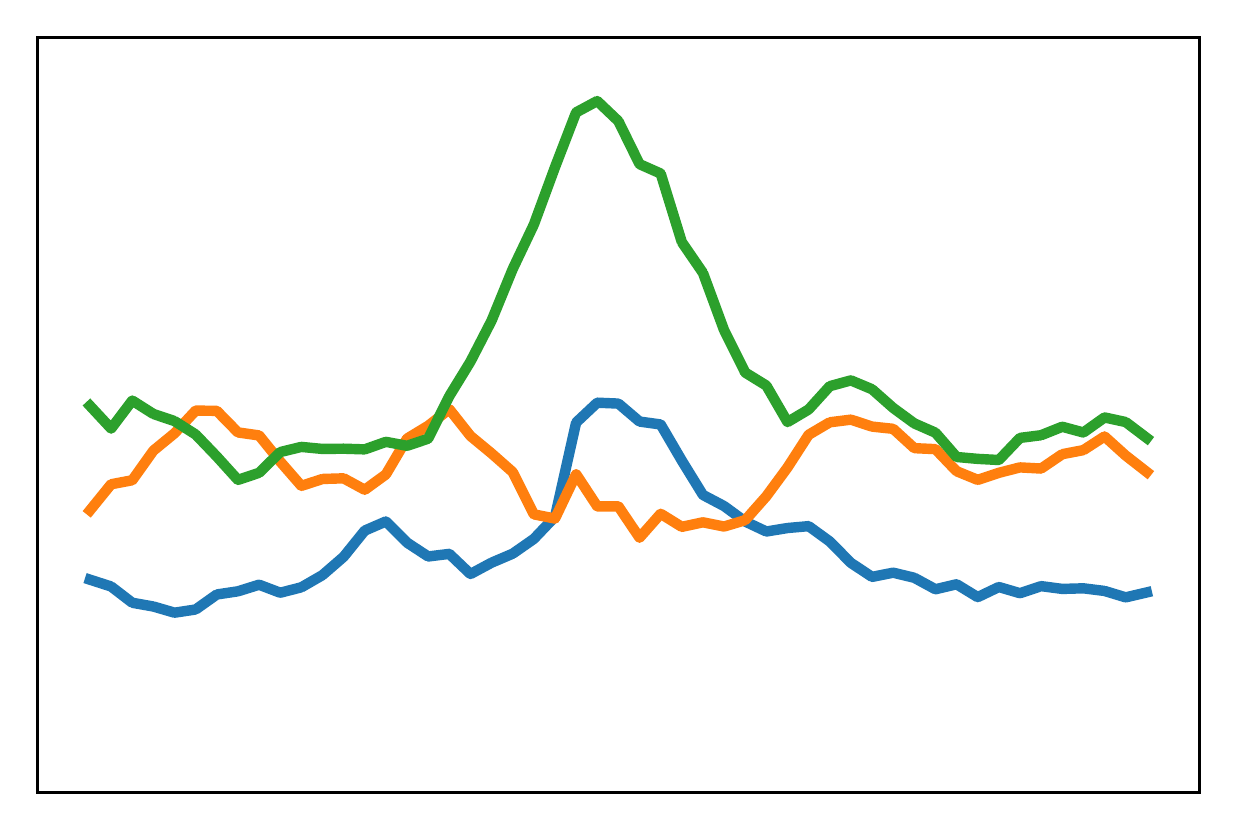}
	\end{subfigure}&
	\begin{subfigure}[b]{0.18\linewidth}
	\includegraphics[width=\linewidth]{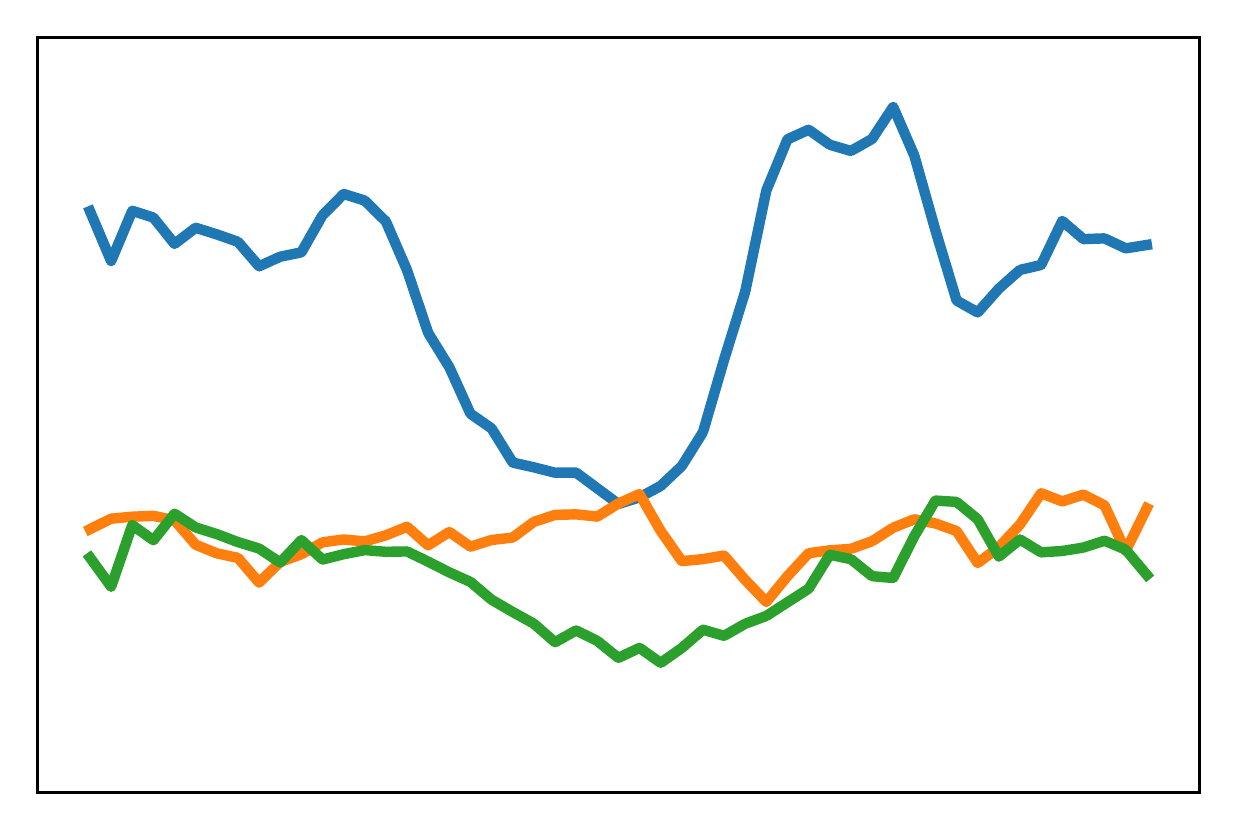}
	\end{subfigure}&
	\begin{subfigure}[b]{0.18\linewidth}
	\includegraphics[width=\linewidth]{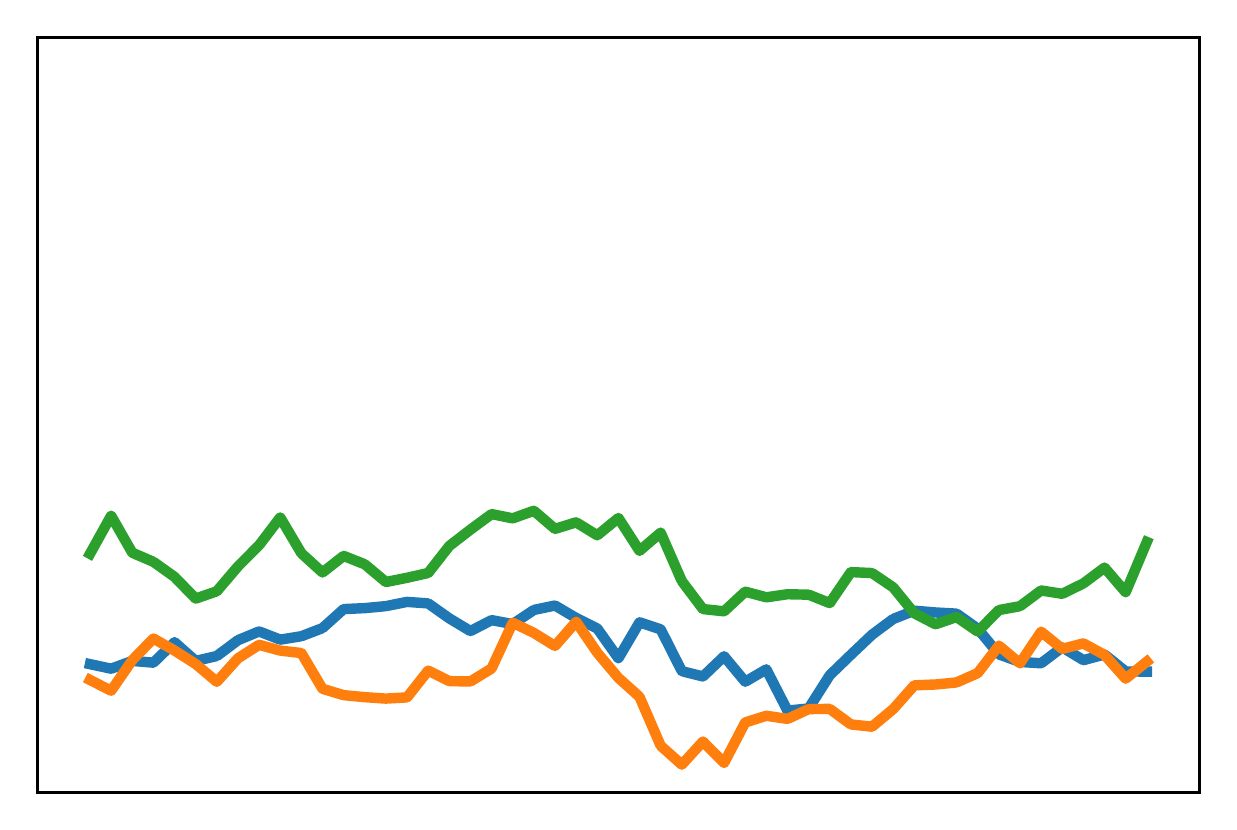}
	\end{subfigure}&
	\begin{subfigure}[b]{0.18\linewidth}
	\includegraphics[width=\linewidth]{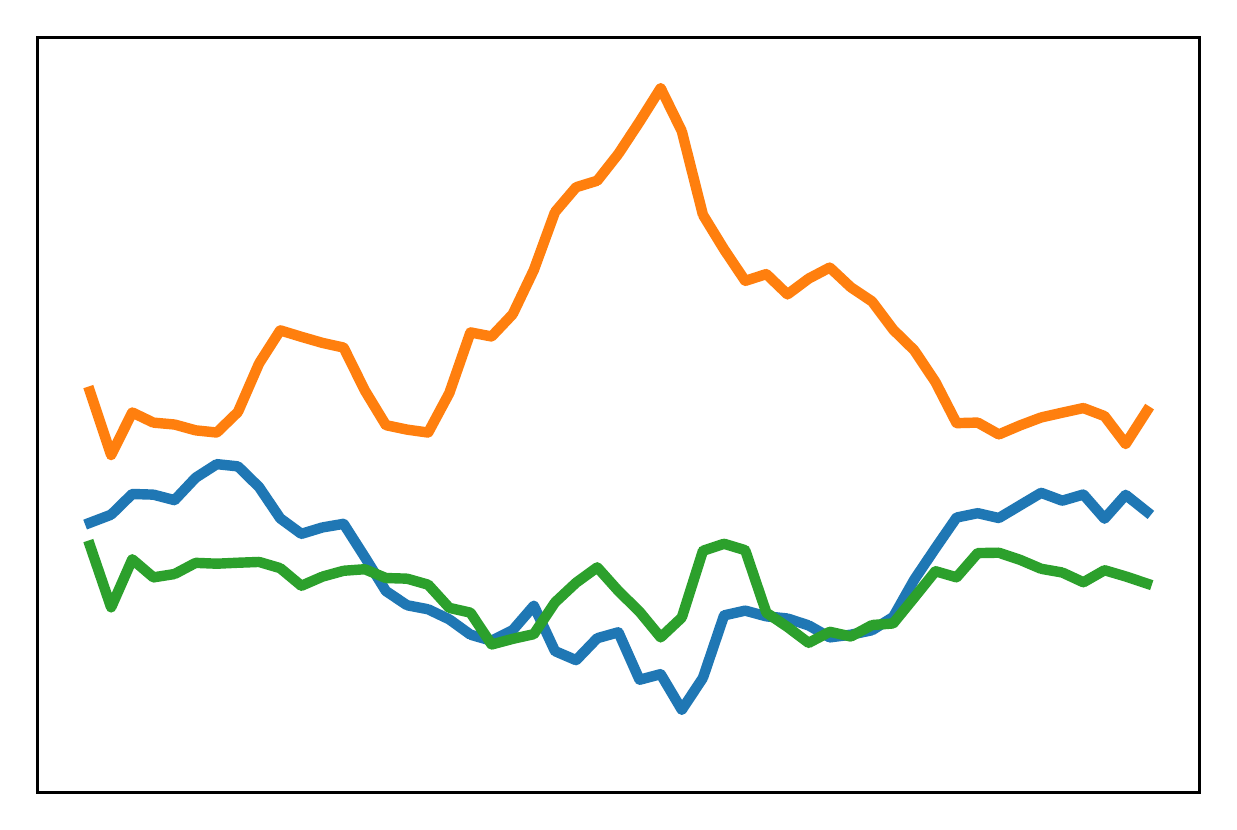}
	\end{subfigure}\\
	\hline
	\begin{subfigure}[b]{0.18\linewidth}
	\includegraphics[width=\linewidth]{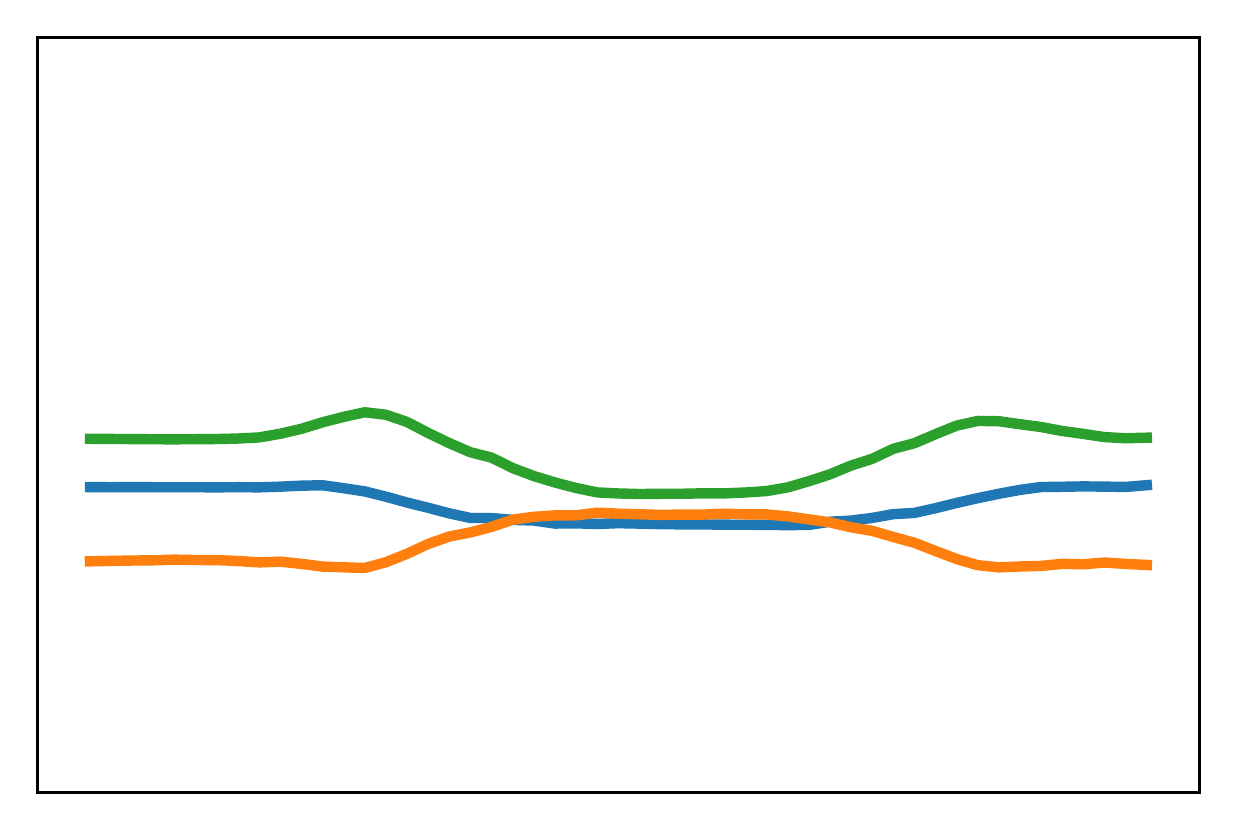}
	\end{subfigure}&
	\begin{subfigure}[b]{0.18\linewidth}
	\includegraphics[width=\linewidth]{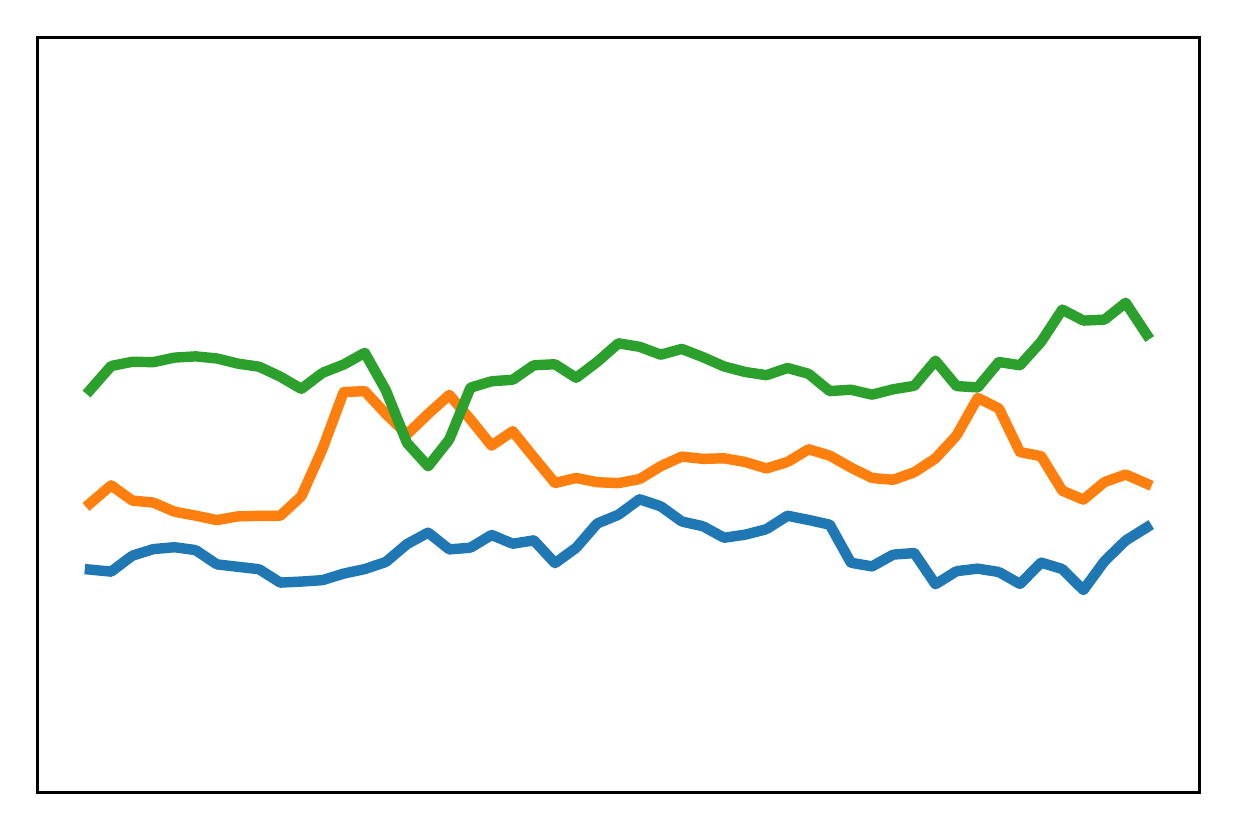}
	\end{subfigure}&
	\begin{subfigure}[b]{0.18\linewidth}
	\includegraphics[width=\linewidth]{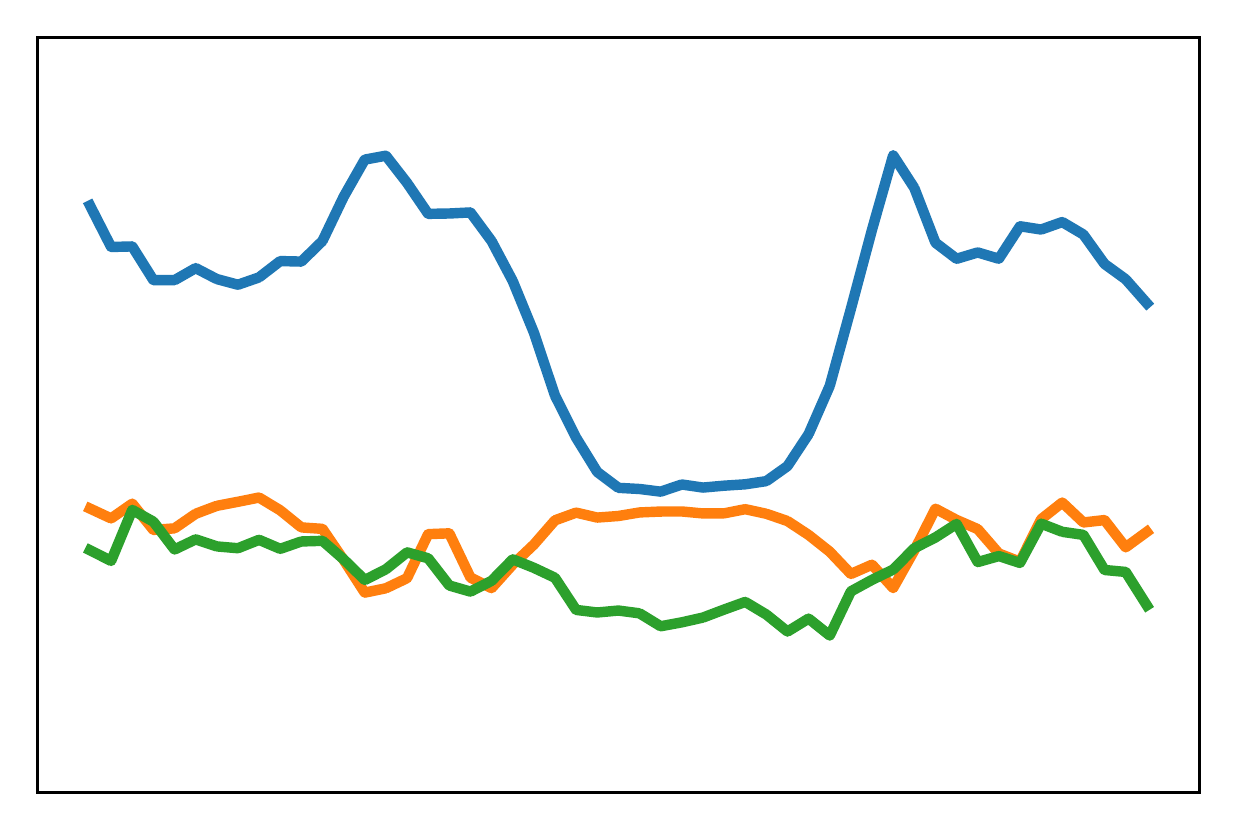}
	\end{subfigure}&
	\begin{subfigure}[b]{0.18\linewidth}
	\includegraphics[width=\linewidth]{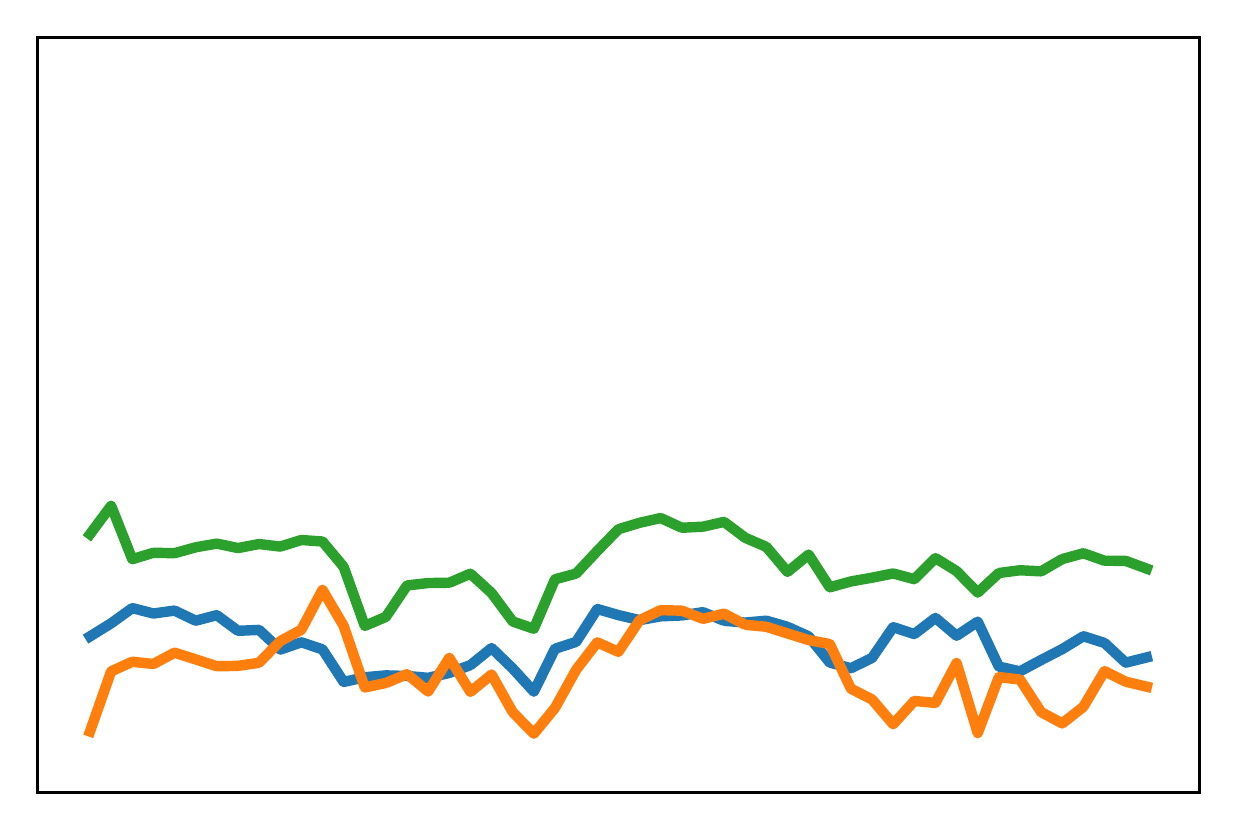}
	\end{subfigure}&
	\begin{subfigure}[b]{0.18\linewidth}
	\includegraphics[width=\linewidth]{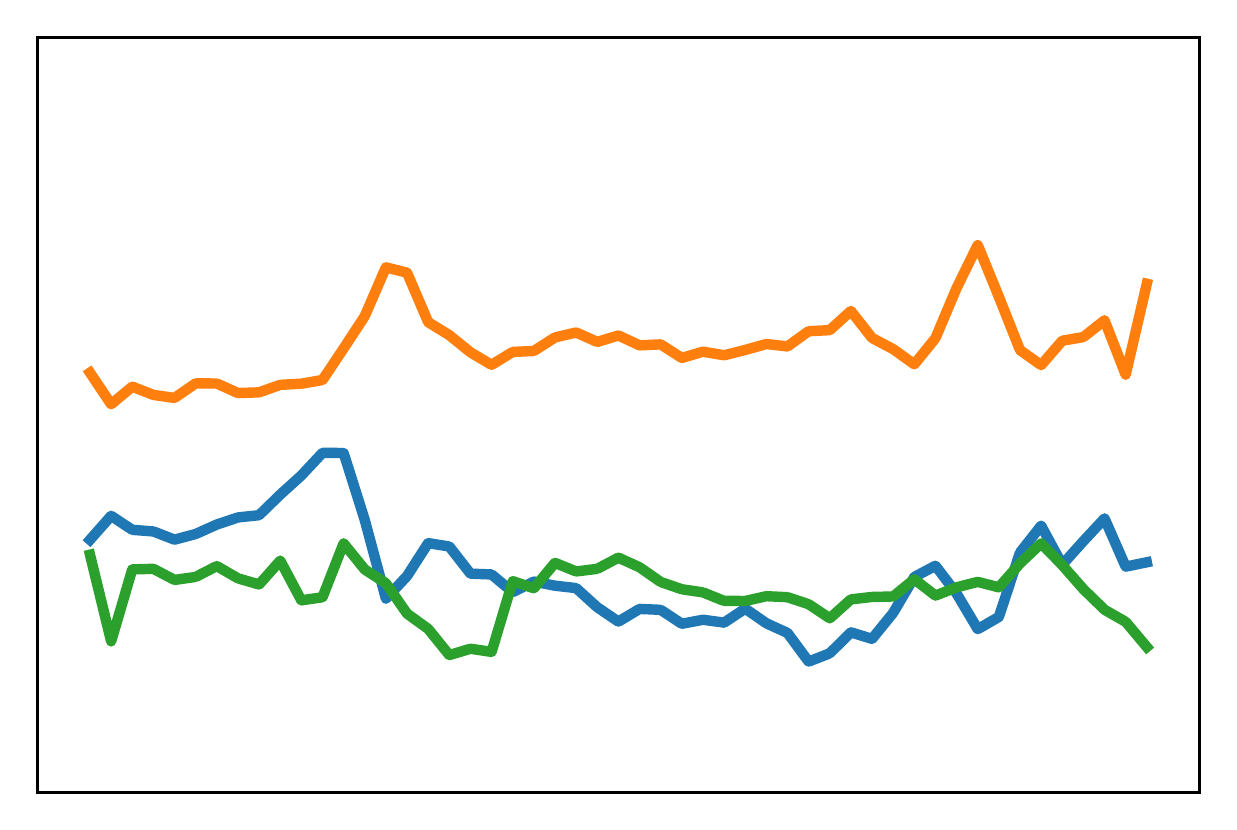}
	\end{subfigure}
\\
	\begin{subfigure}[b]{0.18\linewidth}
	\includegraphics[width=\linewidth]{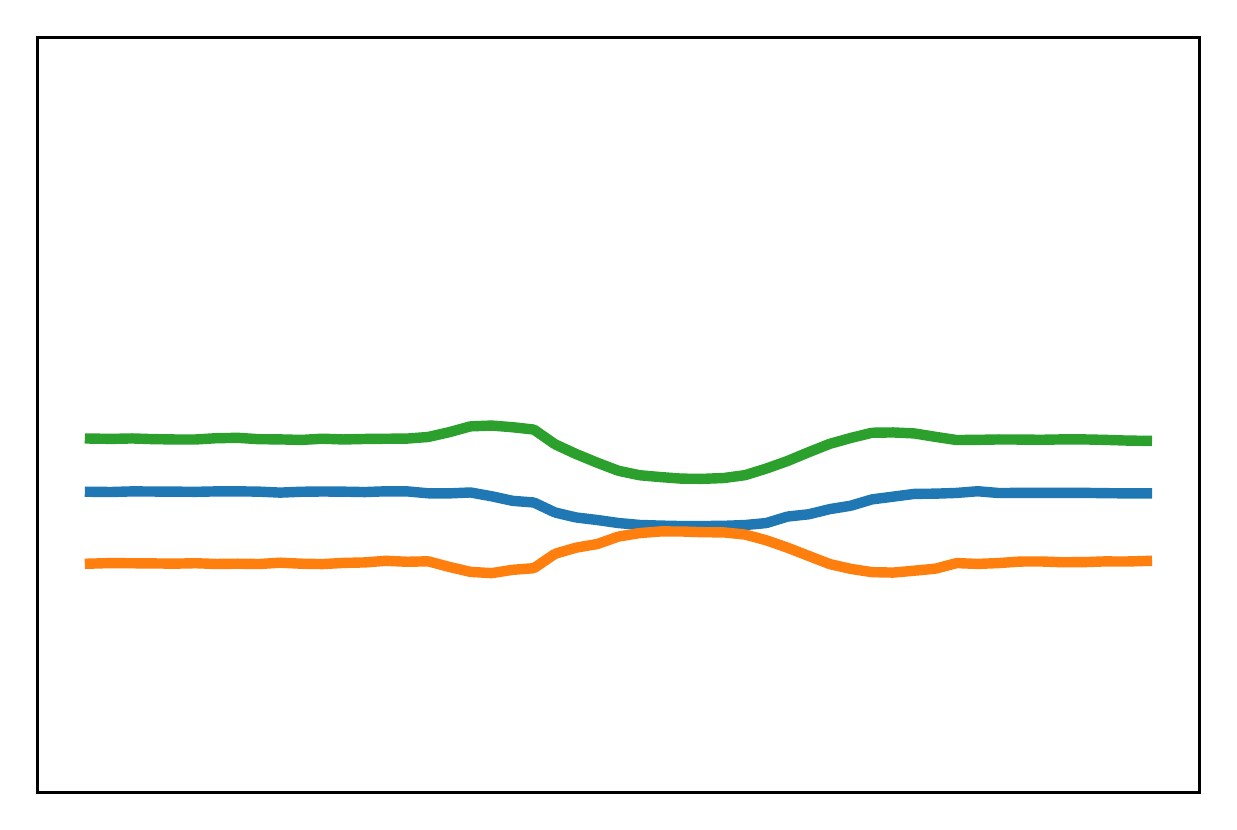}
	\end{subfigure}&
	\begin{subfigure}[b]{0.18\linewidth}
	\includegraphics[width=\linewidth]{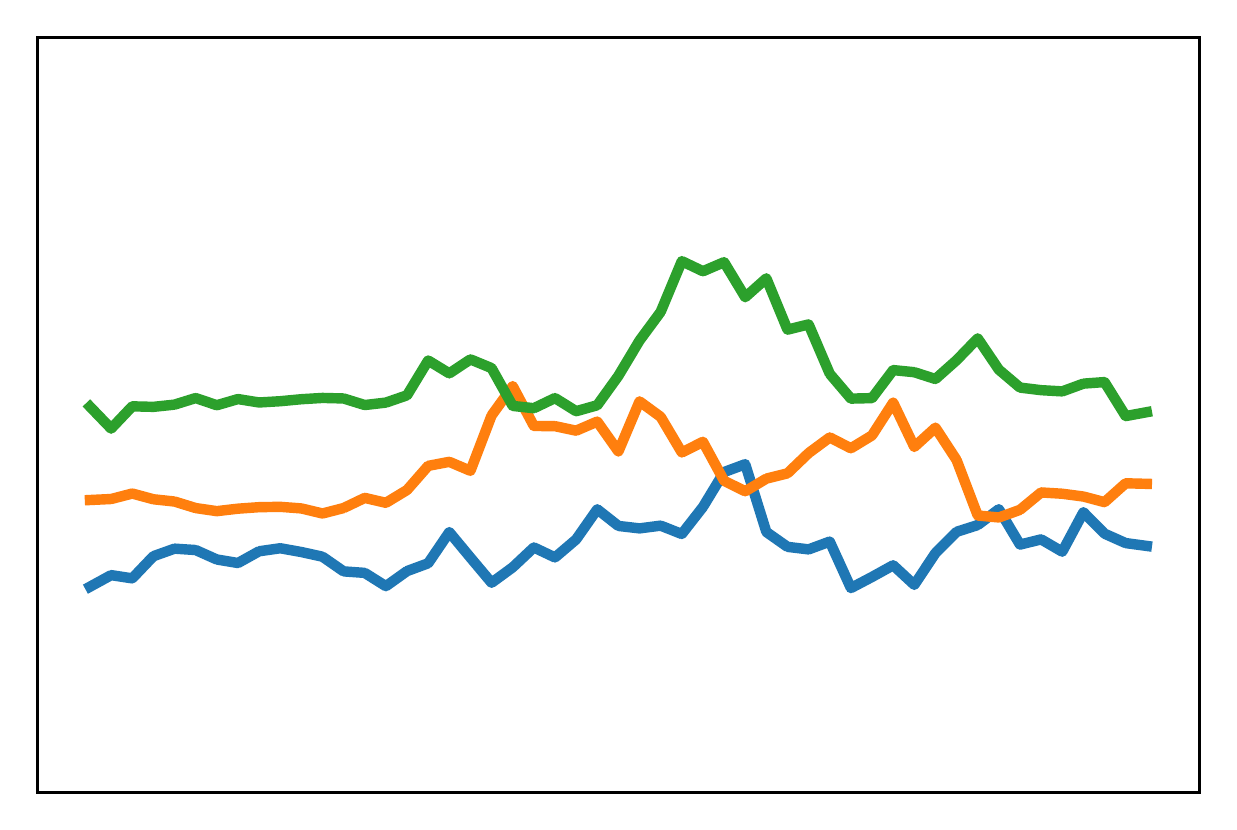}
	\end{subfigure}&
	\begin{subfigure}[b]{0.18\linewidth}
	\includegraphics[width=\linewidth]{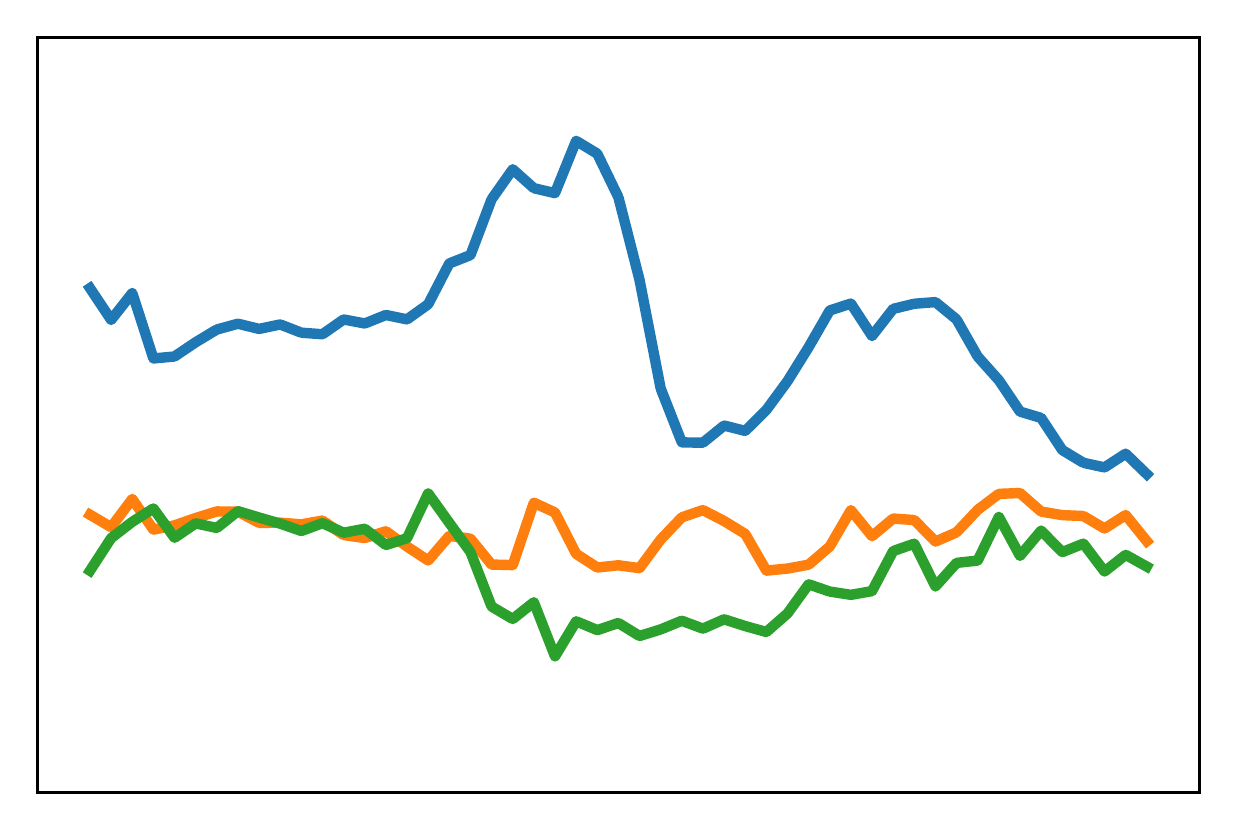}
	\end{subfigure}&
	\begin{subfigure}[b]{0.18\linewidth}
	\includegraphics[width=\linewidth]{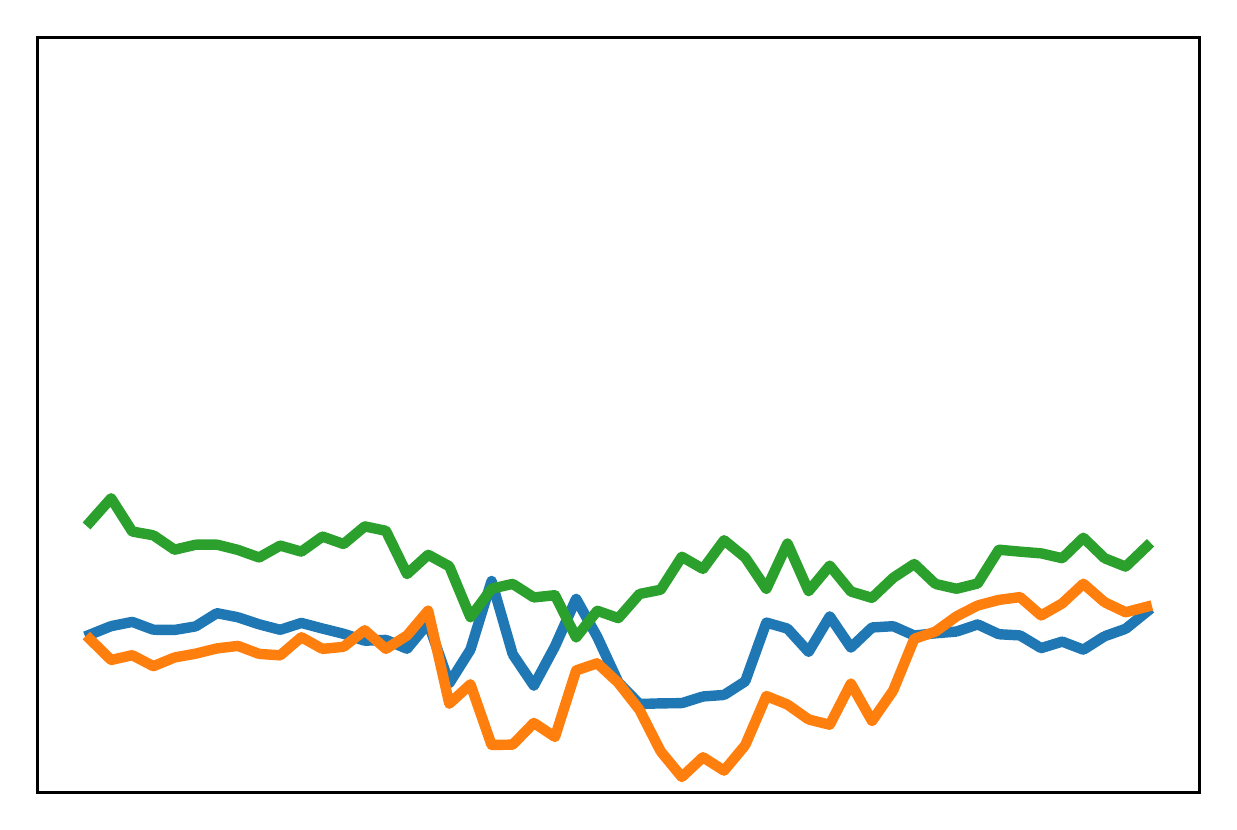}
	\end{subfigure}&
	\begin{subfigure}[b]{0.18\linewidth}
	\includegraphics[width=\linewidth]{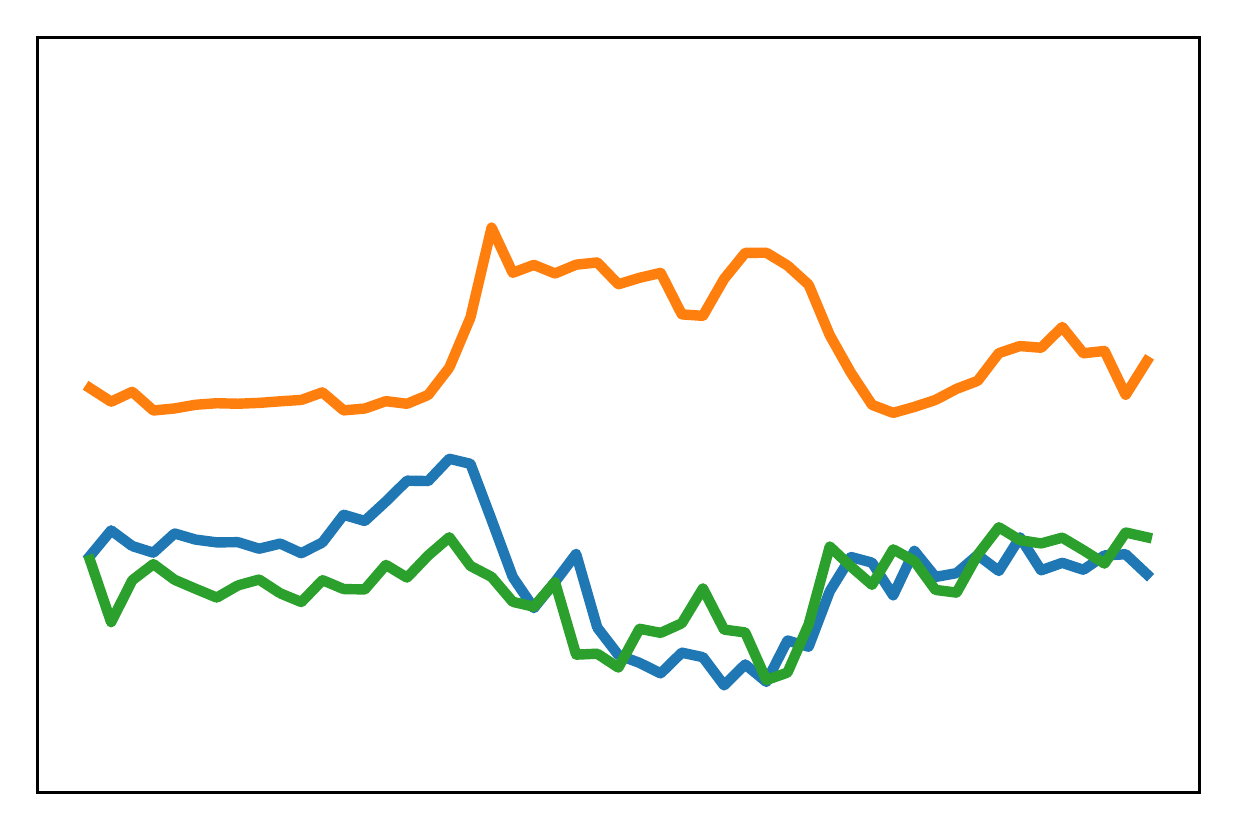}
	\end{subfigure}\\
	\begin{subfigure}[b]{0.18\linewidth}
	\includegraphics[width=\linewidth]{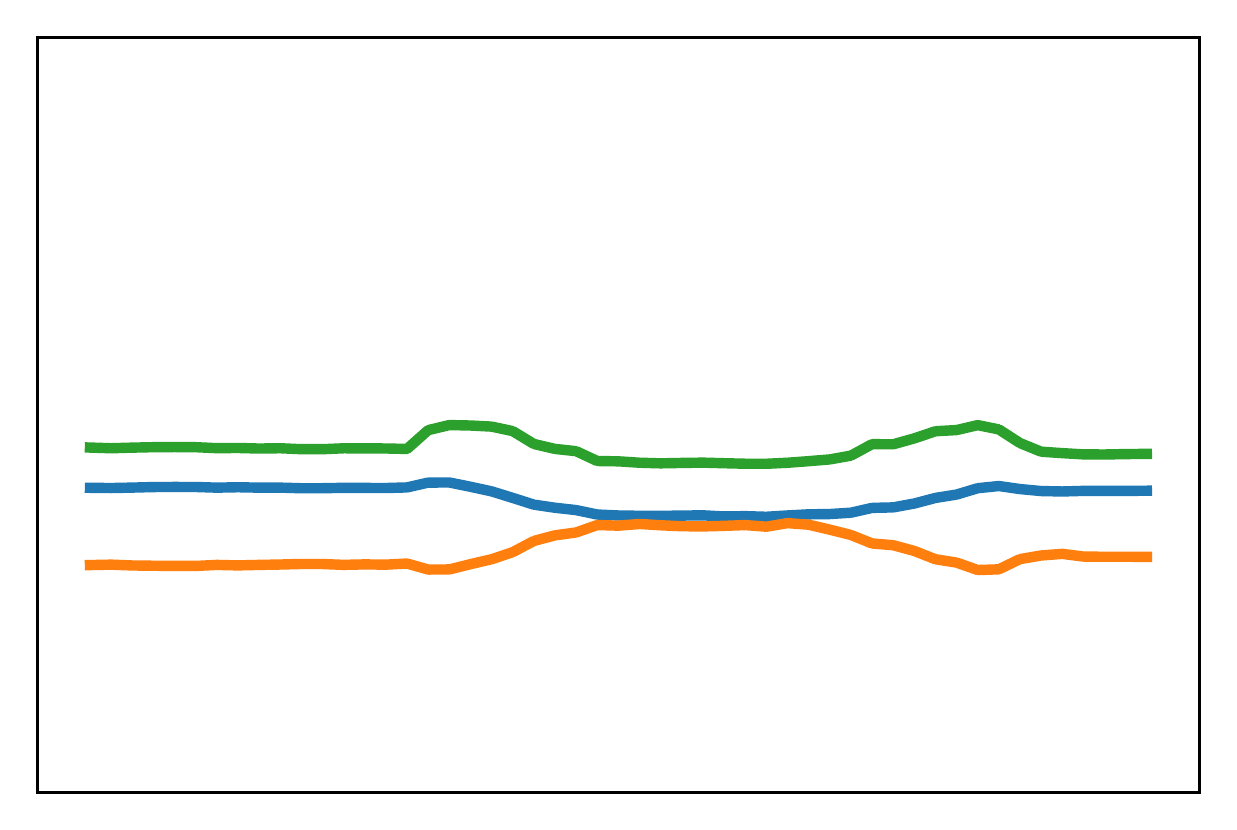}
	\end{subfigure}&
	\begin{subfigure}[b]{0.18\linewidth}
	\includegraphics[width=\linewidth]{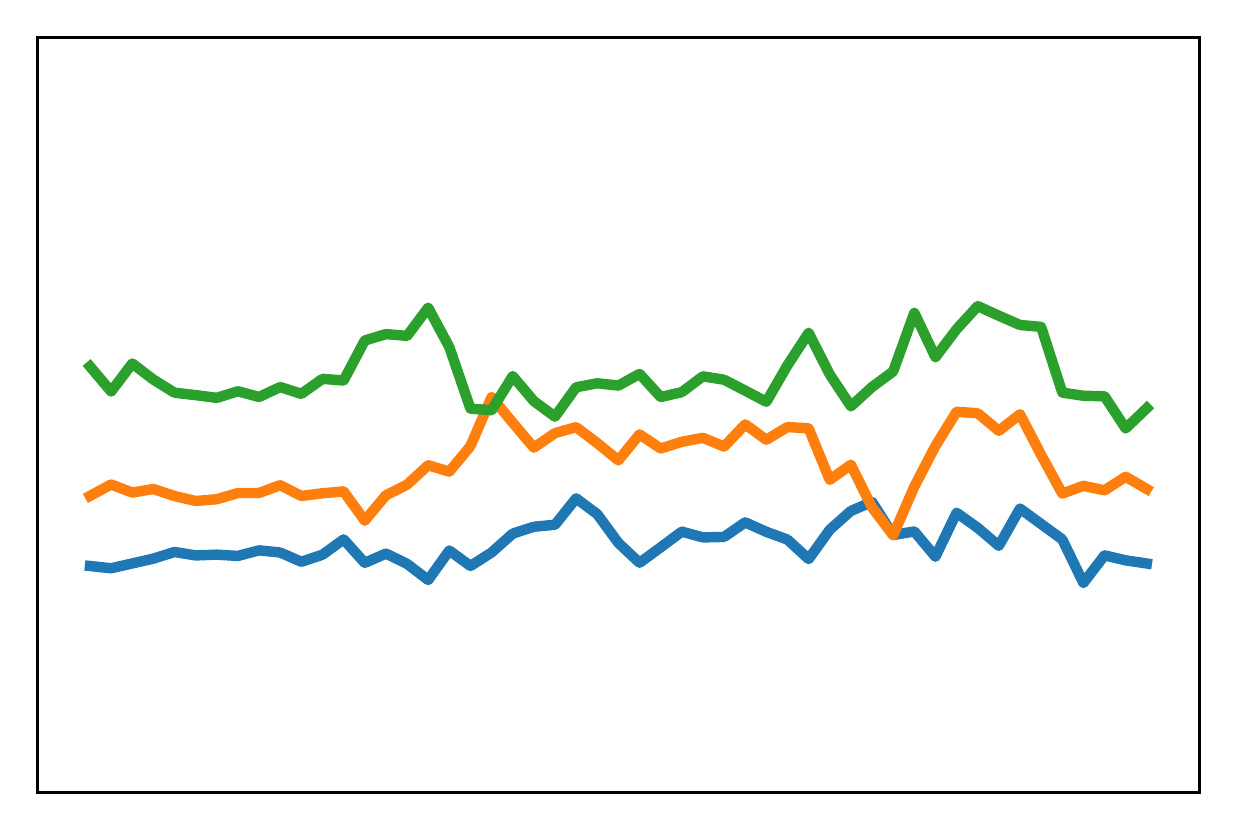}
	\end{subfigure}&
	\begin{subfigure}[b]{0.18\linewidth}
	\includegraphics[width=\linewidth]{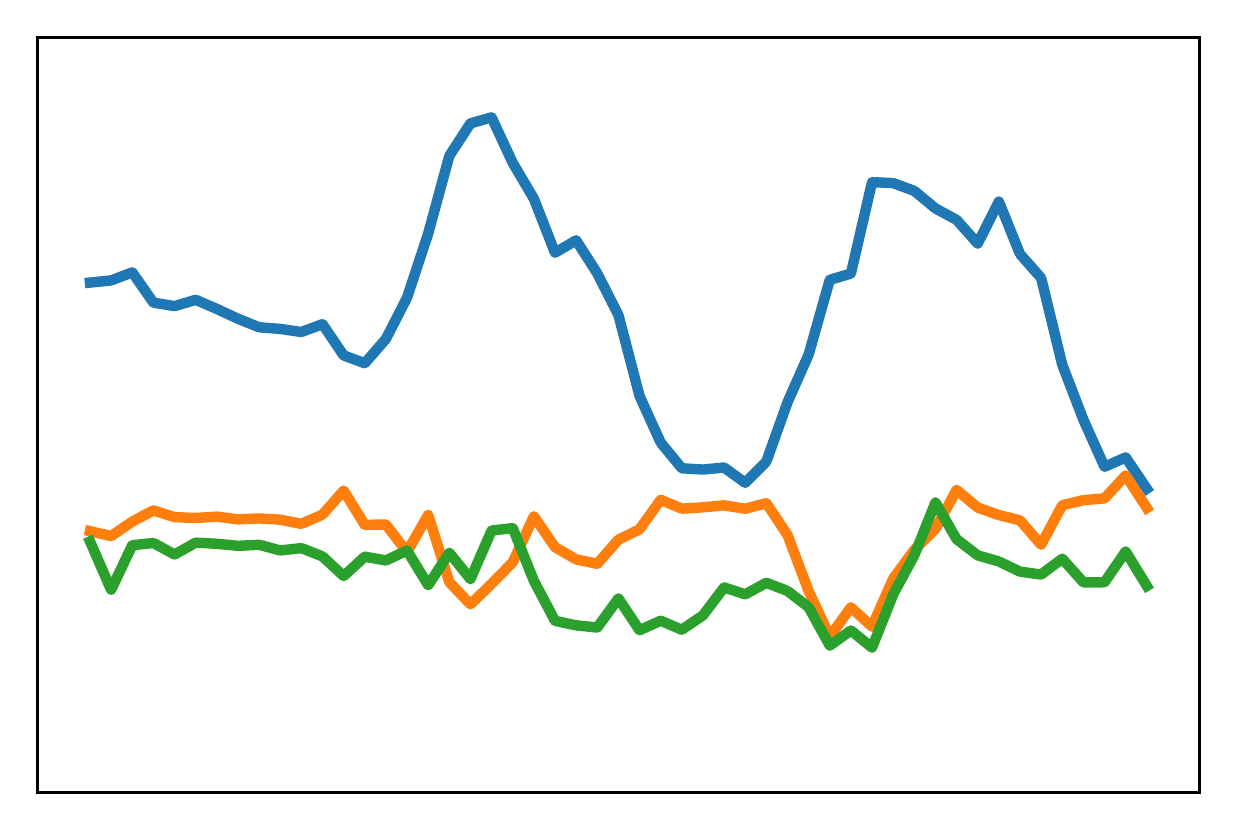}
	\end{subfigure}&
	\begin{subfigure}[b]{0.18\linewidth}
	\includegraphics[width=\linewidth]{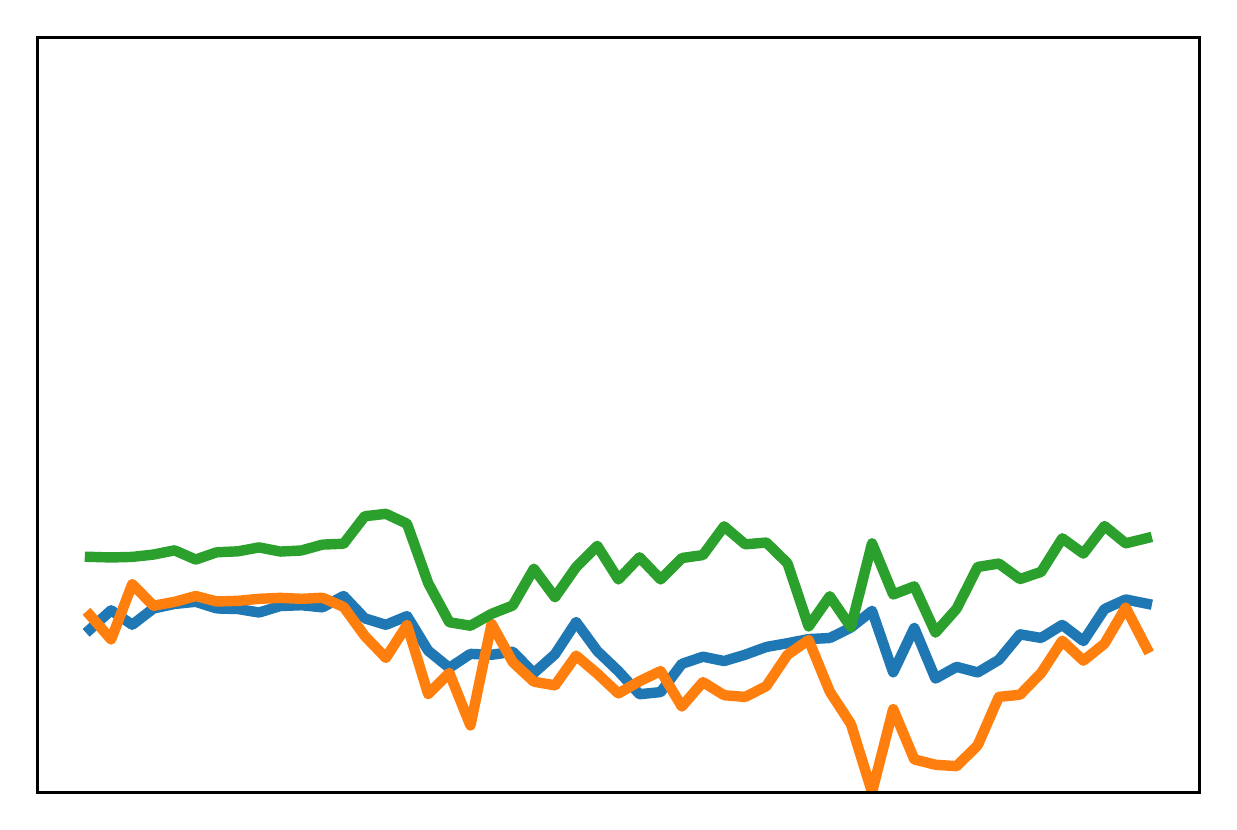}
	\end{subfigure}&
	\begin{subfigure}[b]{0.18\linewidth}
	\includegraphics[width=\linewidth]{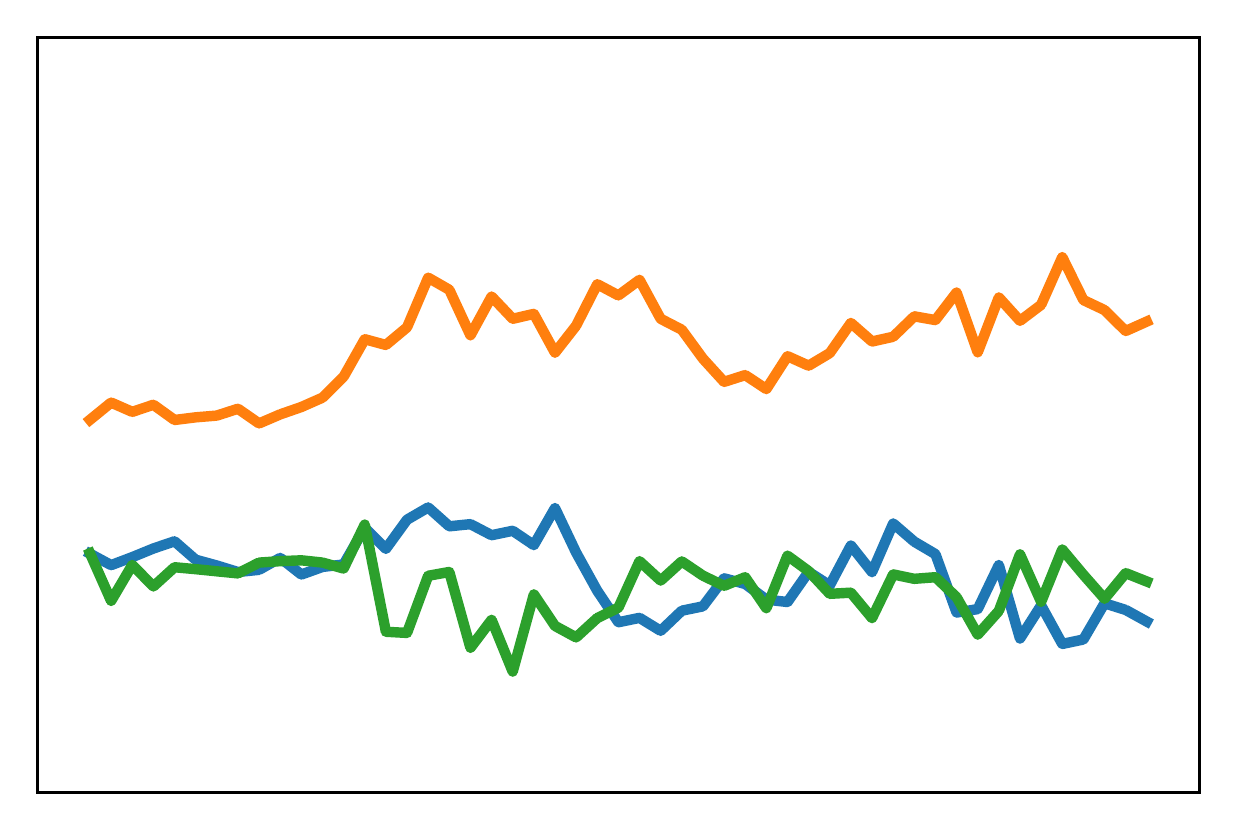}
	\end{subfigure}\\
	\begin{subfigure}[b]{0.18\linewidth}
	\includegraphics[width=\linewidth]{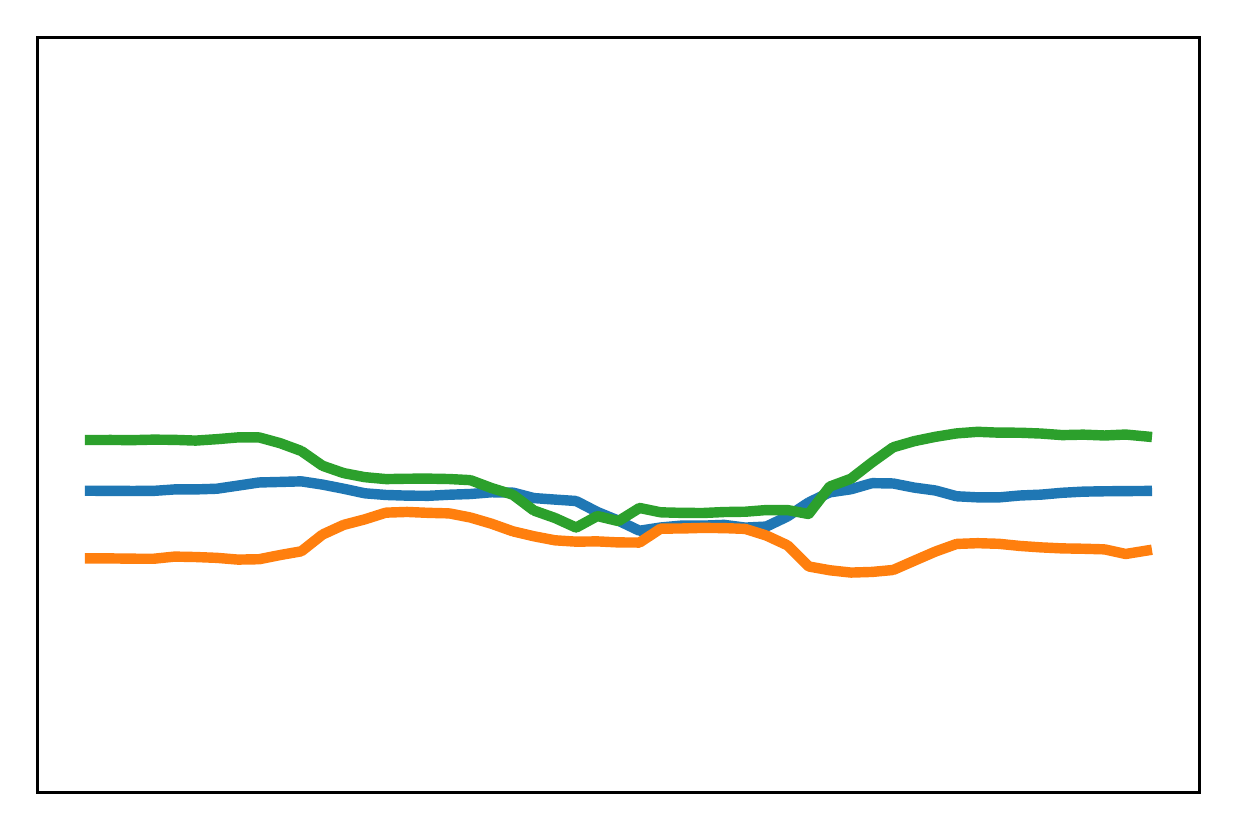}
	\end{subfigure}&
	\begin{subfigure}[b]{0.18\linewidth}
	\includegraphics[width=\linewidth]{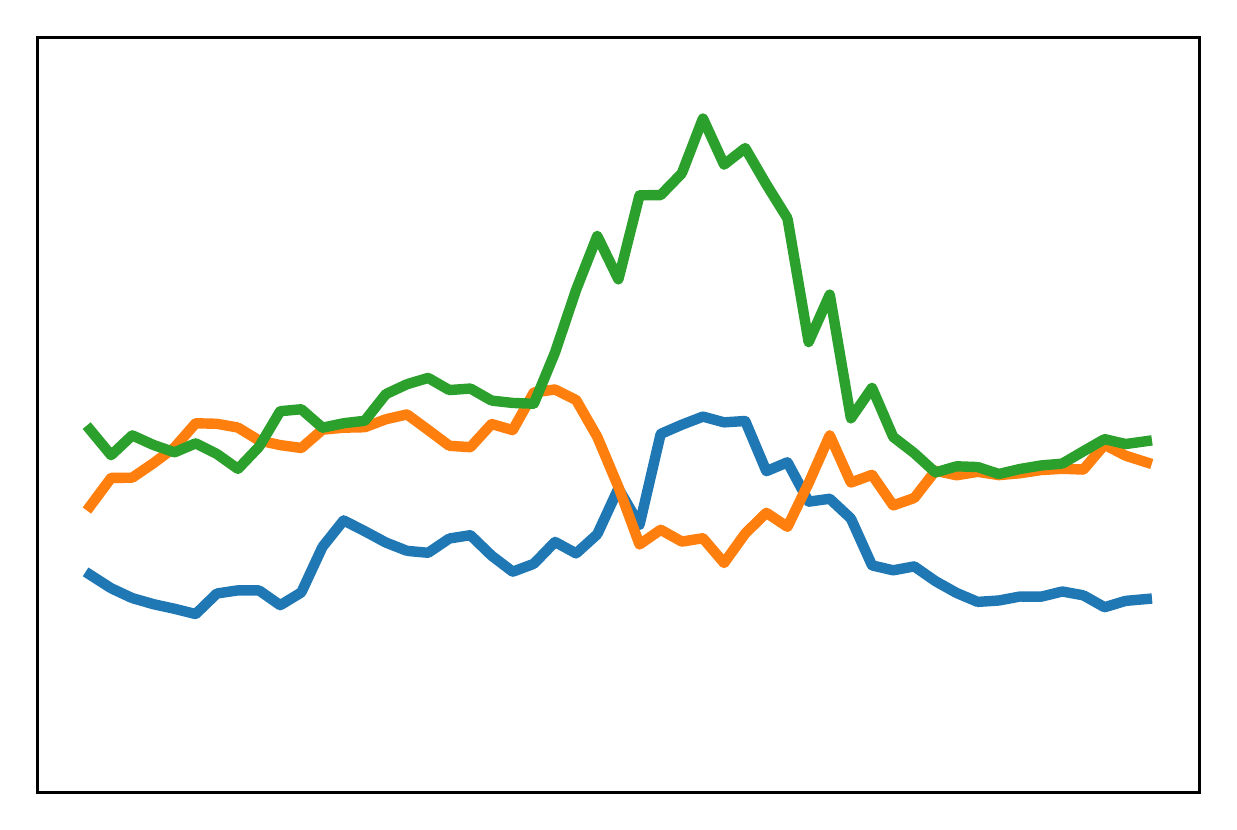}
	\end{subfigure}&
	\begin{subfigure}[b]{0.18\linewidth}
	\includegraphics[width=\linewidth]{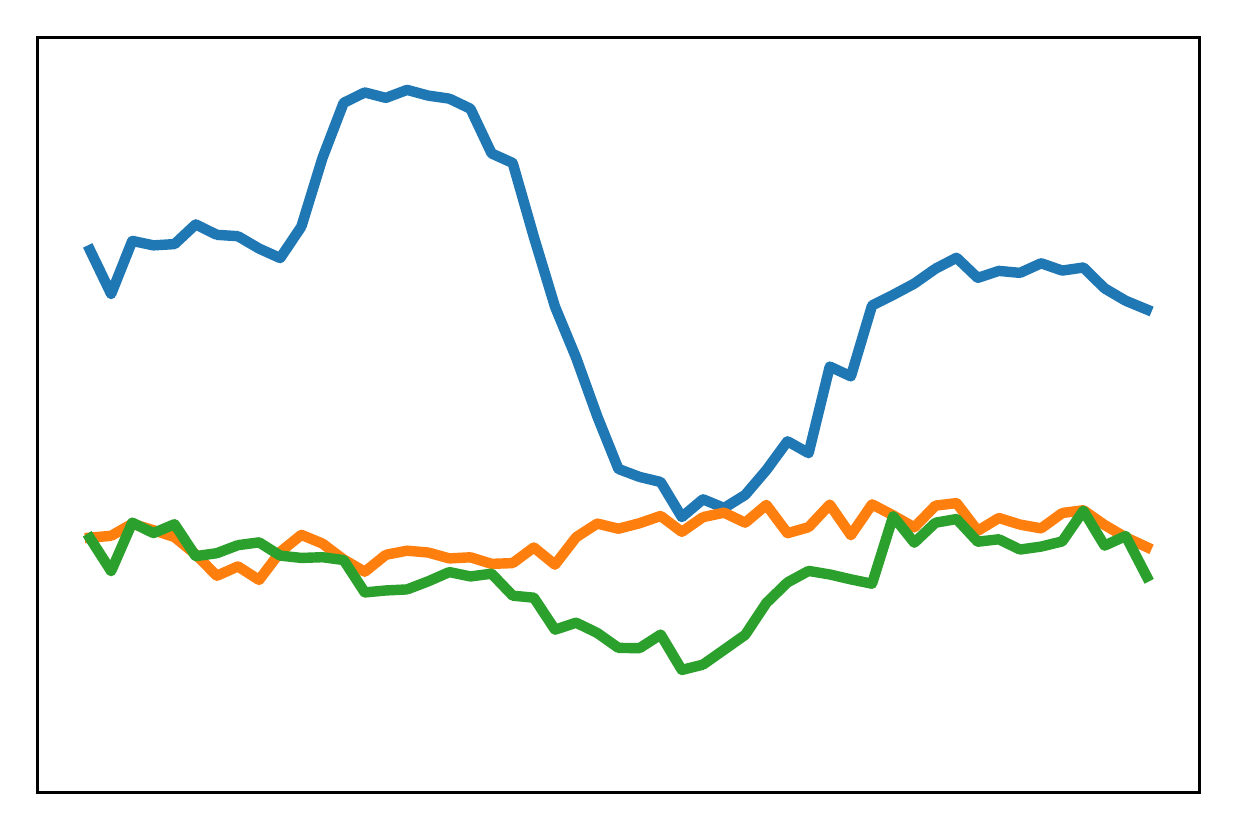}
	\end{subfigure}&
	\begin{subfigure}[b]{0.18\linewidth}
	\includegraphics[width=\linewidth]{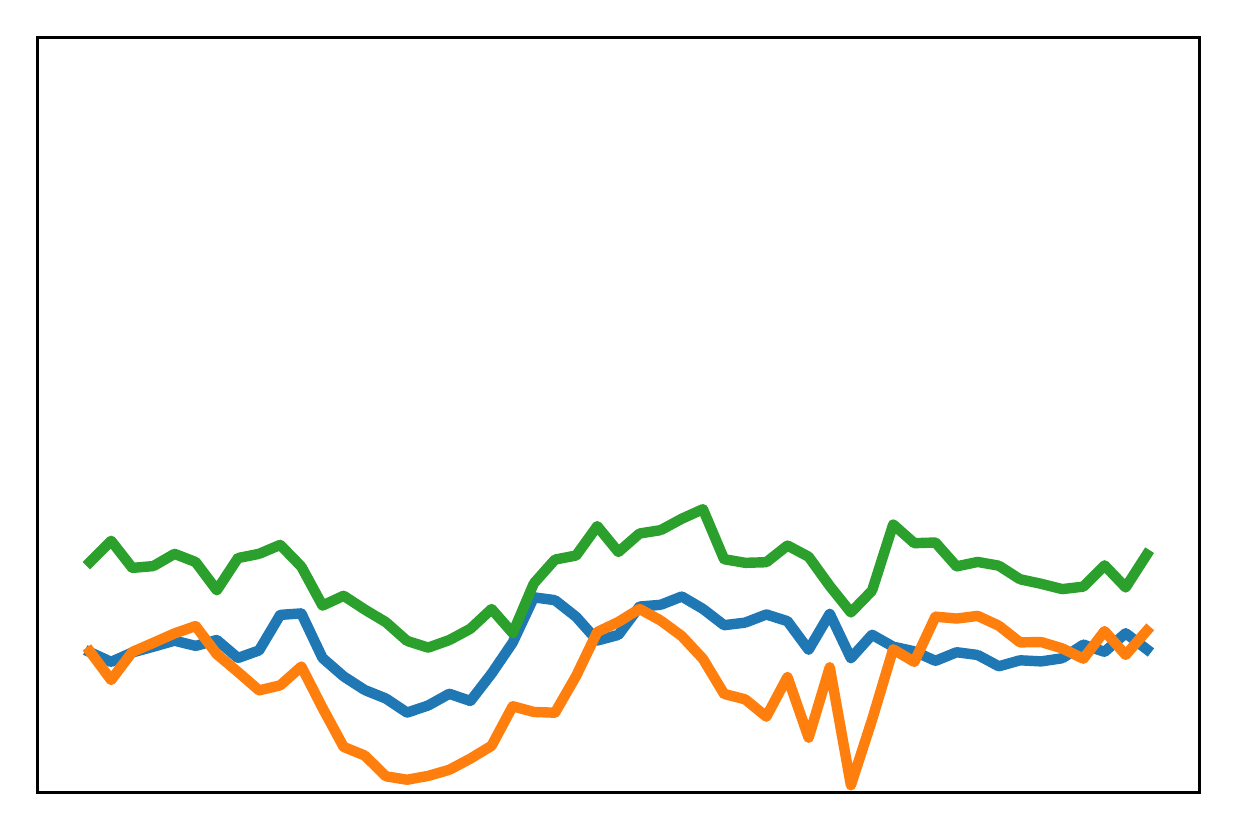}
	\end{subfigure}&
	\begin{subfigure}[b]{0.18\linewidth}
	\includegraphics[width=\linewidth]{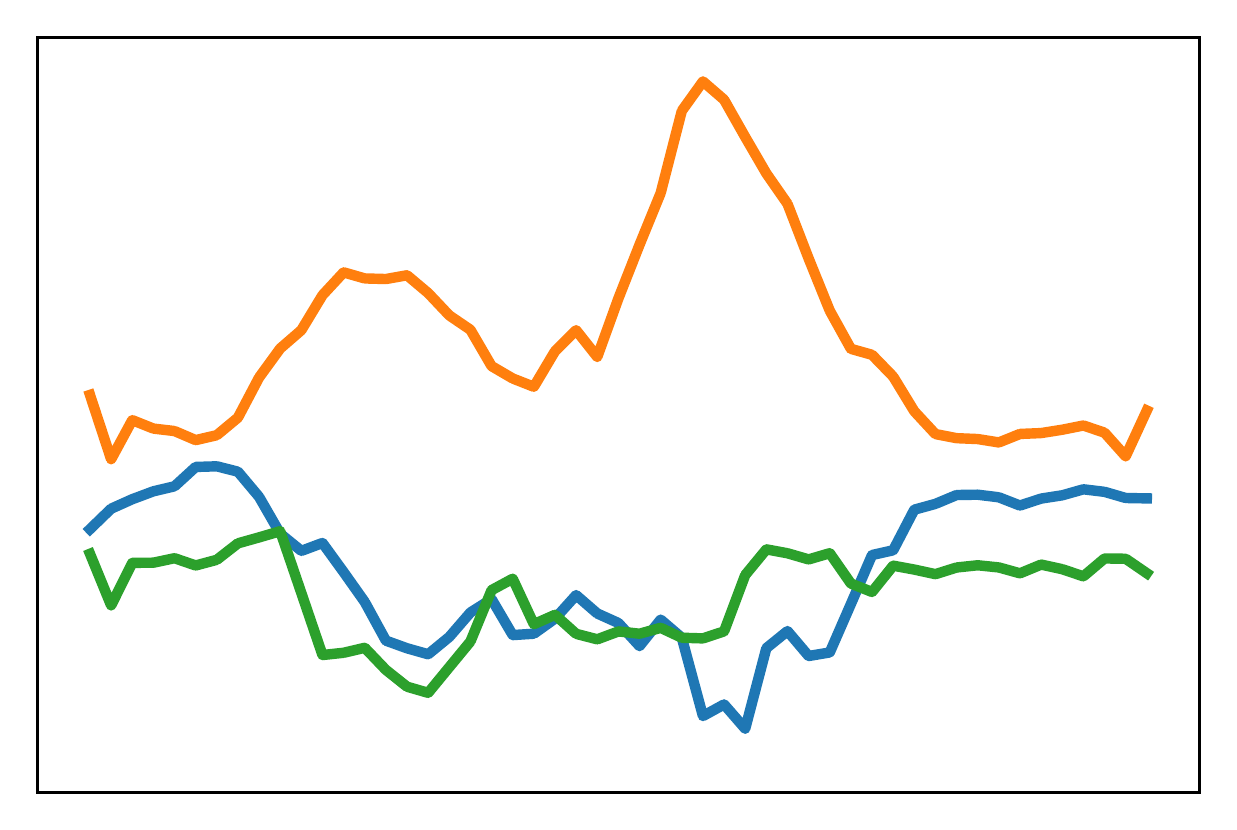}
	\end{subfigure}\\	
	\begin{subfigure}[b]{0.18\linewidth}
	\includegraphics[width=\linewidth]{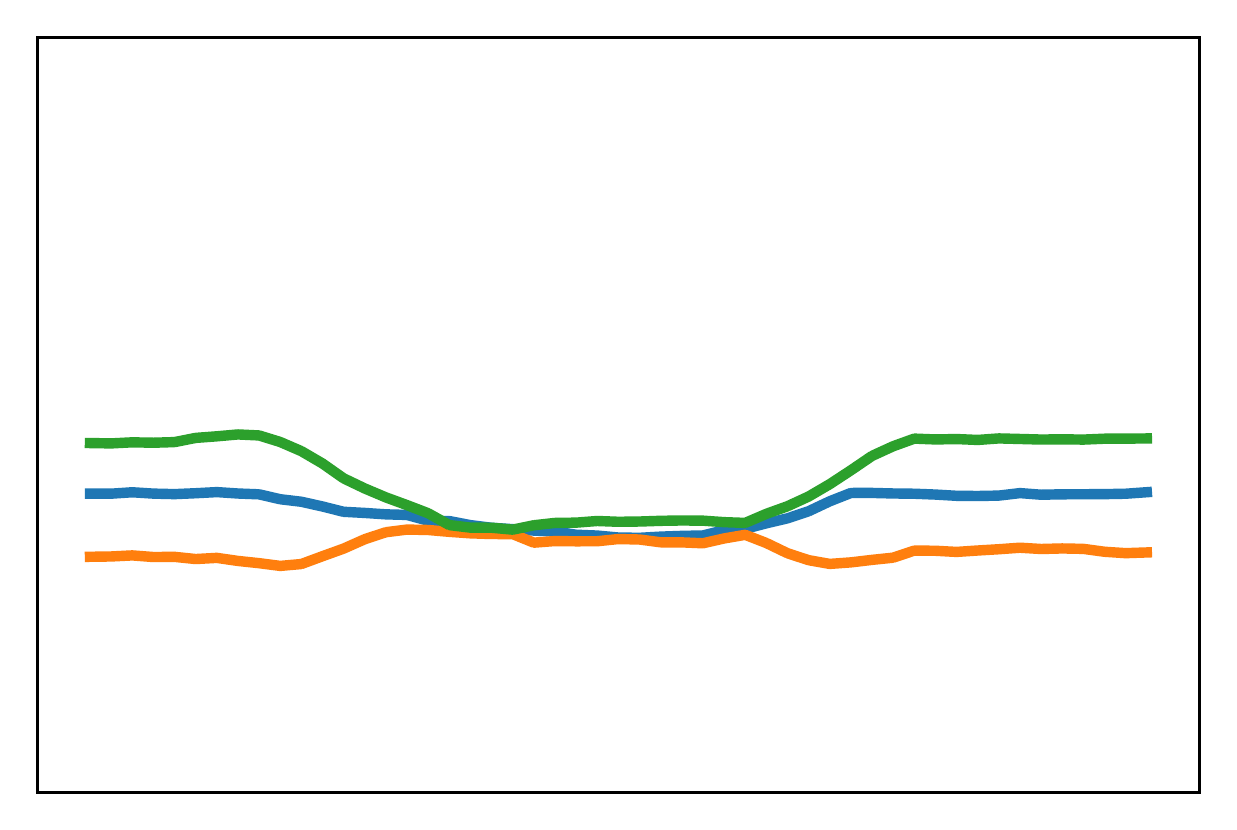}
	\end{subfigure}&
	\begin{subfigure}[b]{0.18\linewidth}
	\includegraphics[width=\linewidth]{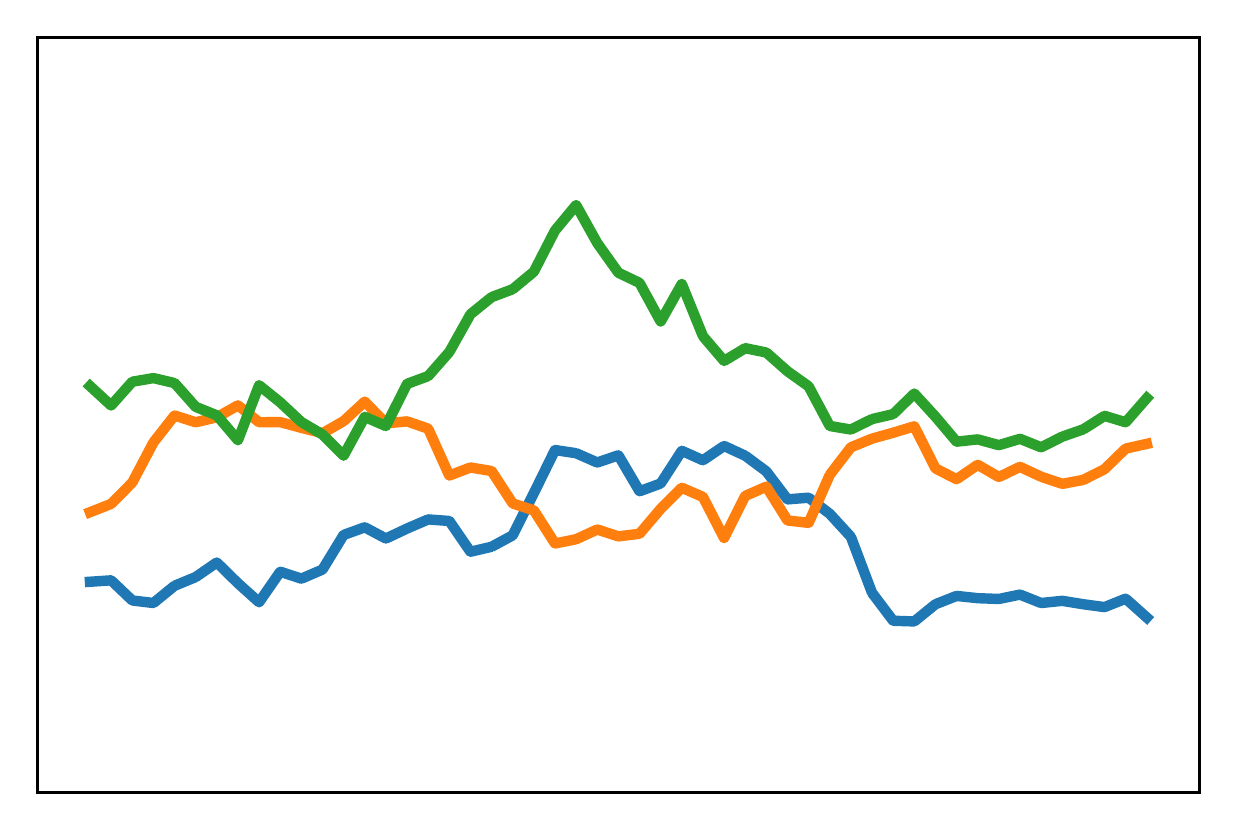}
	\end{subfigure}&
	\begin{subfigure}[b]{0.18\linewidth}
	\includegraphics[width=\linewidth]{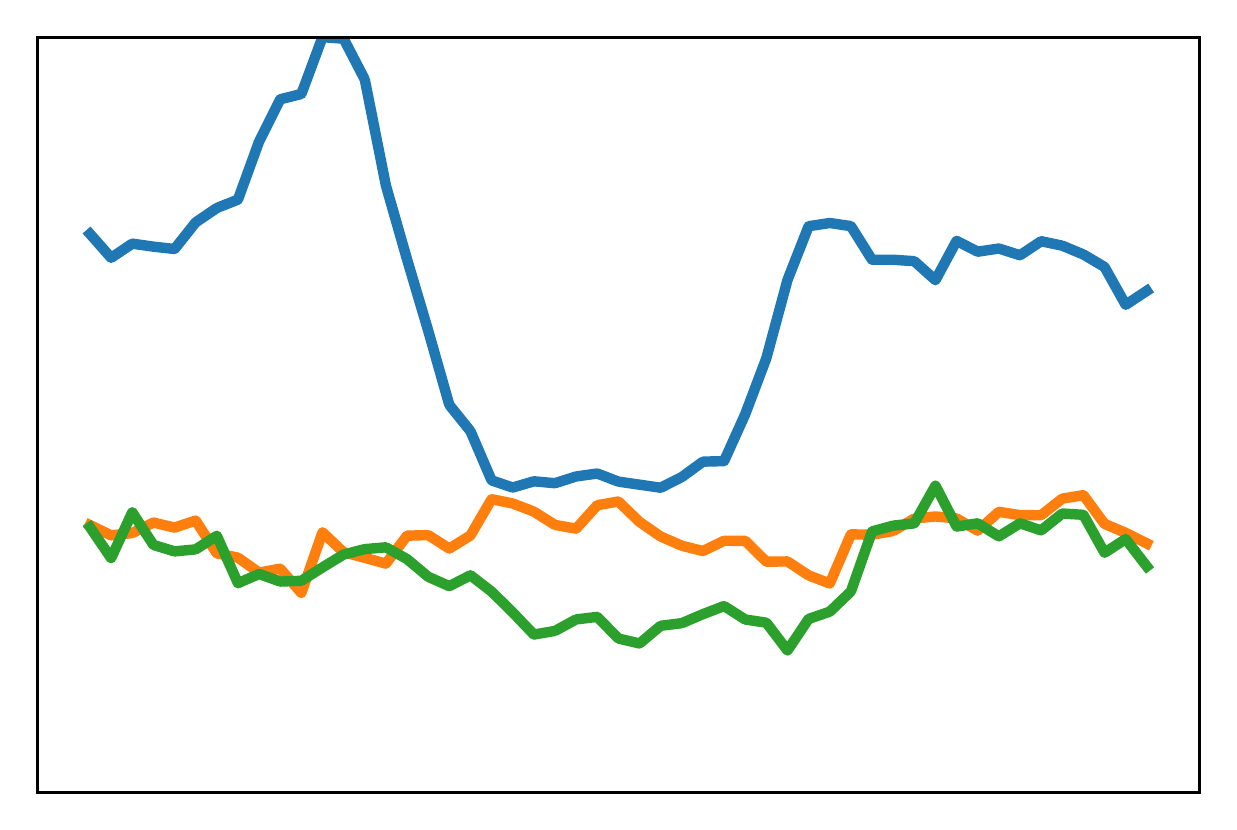}
	\end{subfigure}&
	\begin{subfigure}[b]{0.18\linewidth}
	\includegraphics[width=\linewidth]{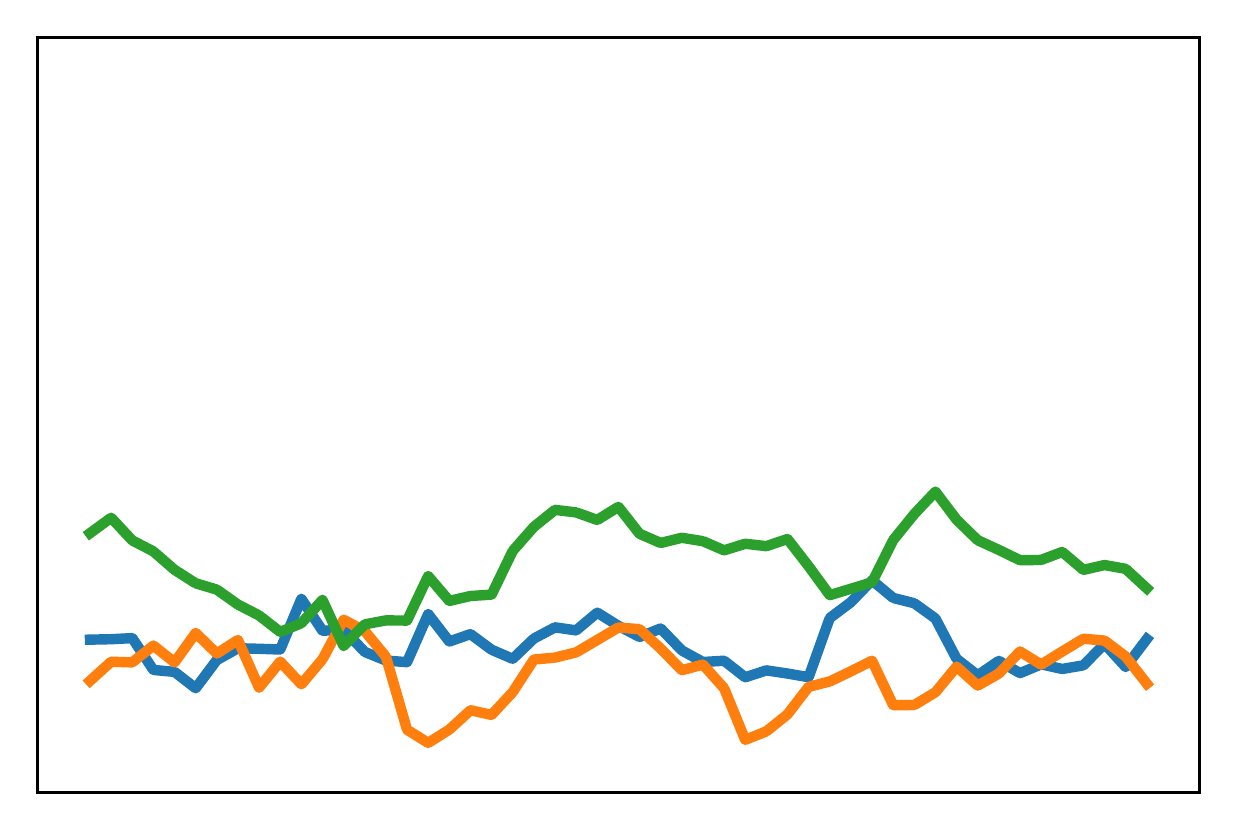}
	\end{subfigure}&
	\begin{subfigure}[b]{0.18\linewidth}
	\includegraphics[width=\linewidth]{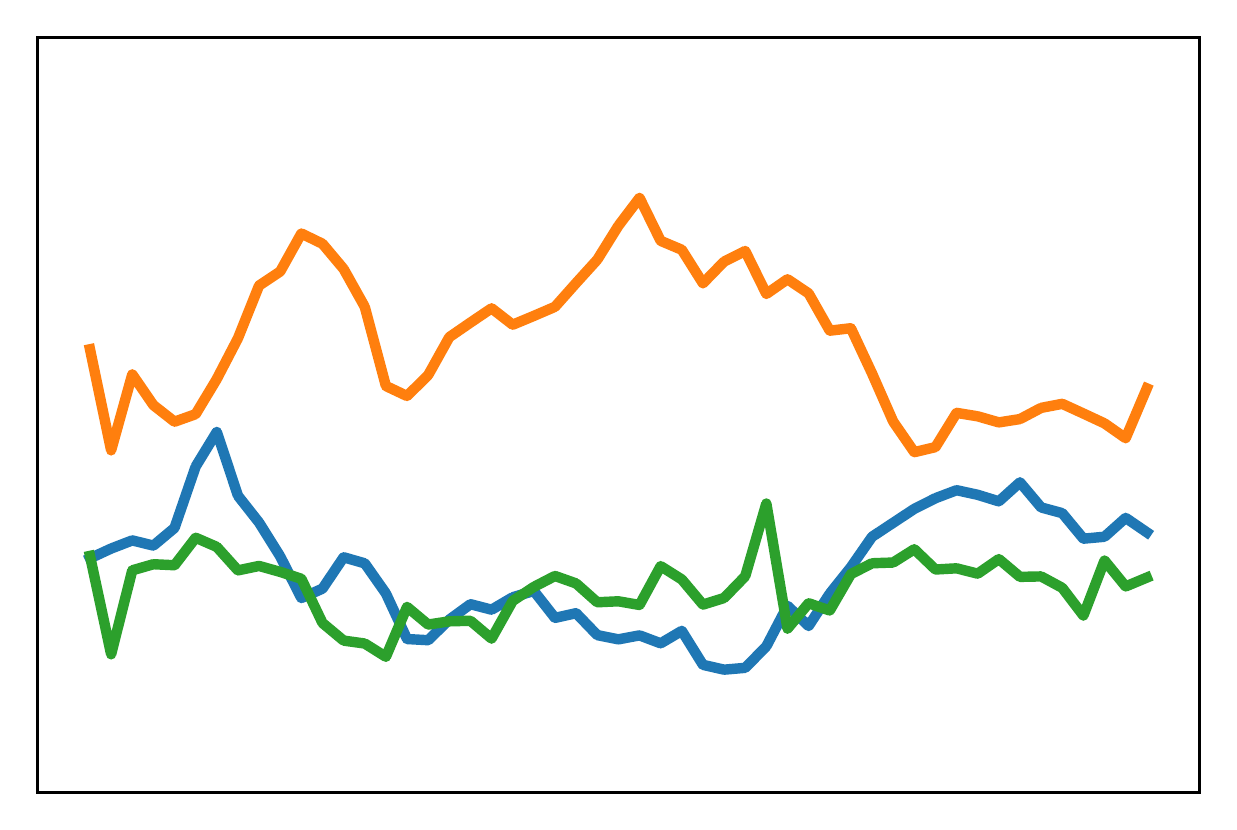}
	\end{subfigure}
	\end{tabular}
	\caption{The learned transformations ($T(x) = M(x)$) on \gls{natops}. The first row: the learned transformations of one given example of the normal class. The rest rows: the learned transformations of one given example of each abnormal class. }
		\label{fig:na_Ts}
\end{figure}

%% file: figures/explain.tex
\begin{figure}[ht]
	\captionsetup[subfigure]{labelformat=empty}
	\centering
	\begin{tabular}{@{}c@{}c@{}c@{}c@{}}
	normal data&anomaly 1&anomaly 2&anomaly 3\\
	\begin{subfigure}[b]{0.2\linewidth}
	\begin{overpic}[width=\linewidth]{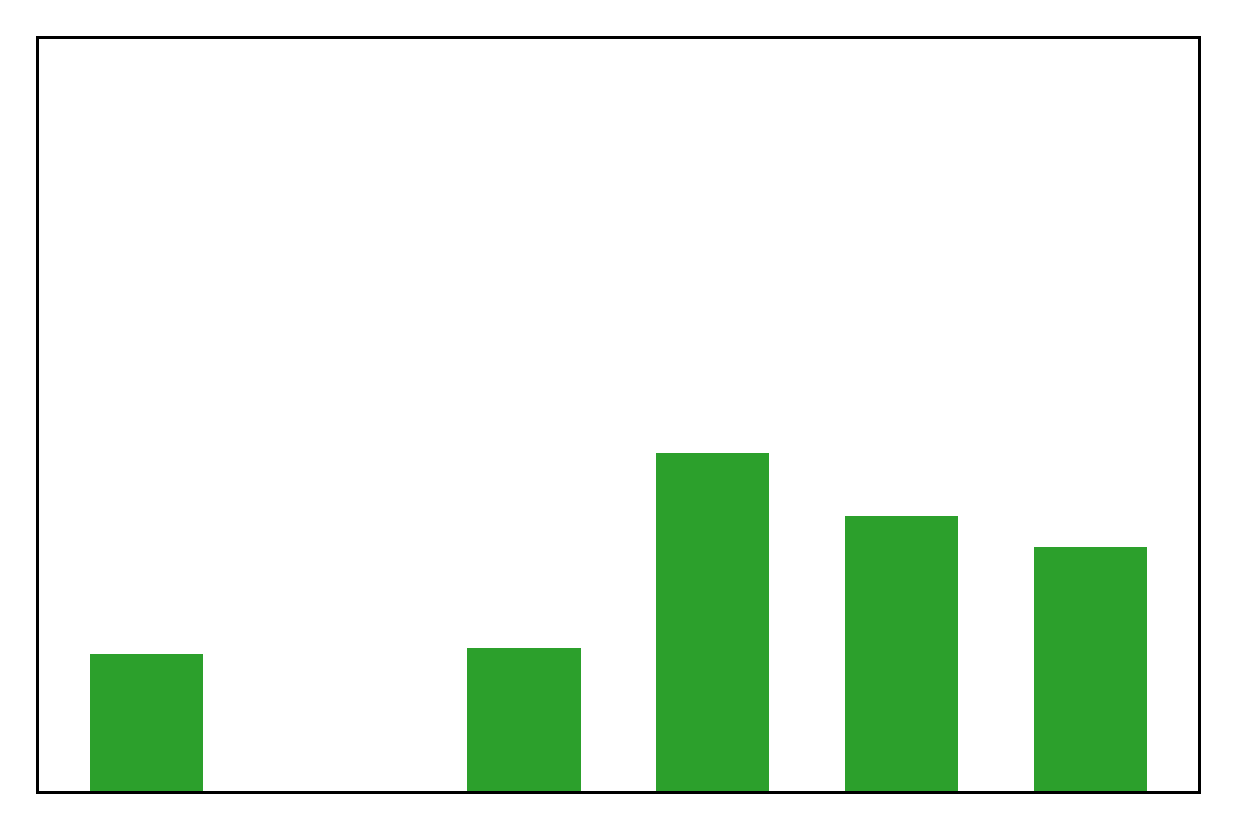}
 \put (48,50) {$\displaystyle x$}
\end{overpic}
	\end{subfigure}&
	\begin{subfigure}[b]{0.2\linewidth}
	\includegraphics[width=\linewidth]{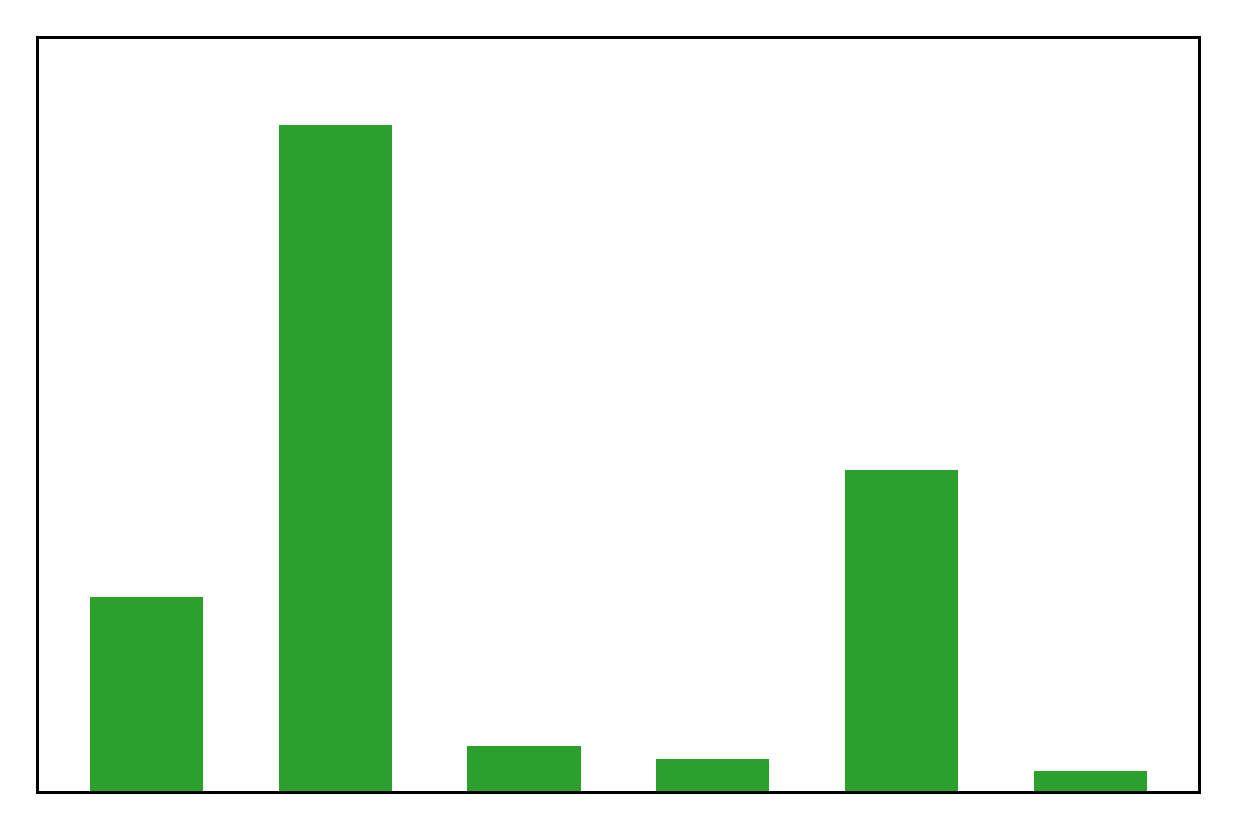}
	\end{subfigure}&
	\begin{subfigure}[b]{0.2\linewidth}
	\includegraphics[width=\linewidth]{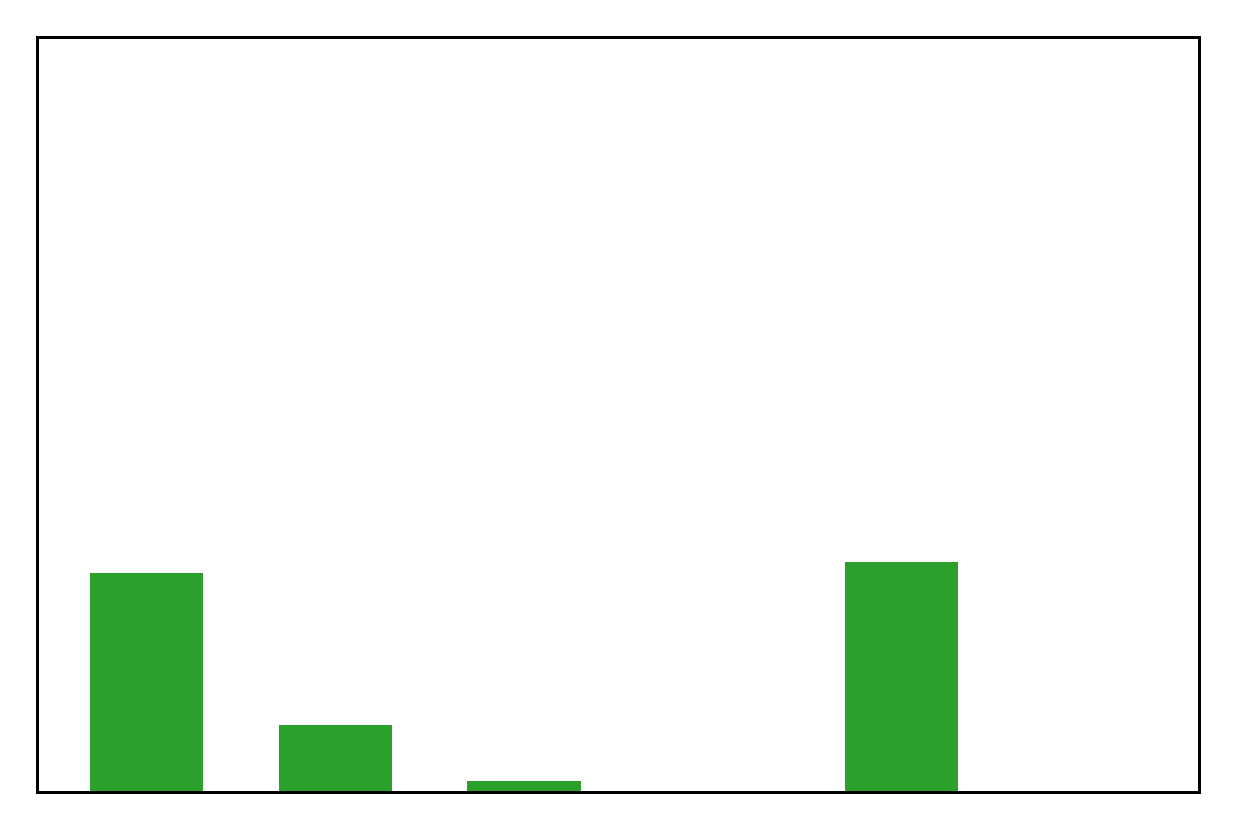}
	\end{subfigure}&
	\begin{subfigure}[b]{0.2\linewidth}
	\includegraphics[width=\linewidth]{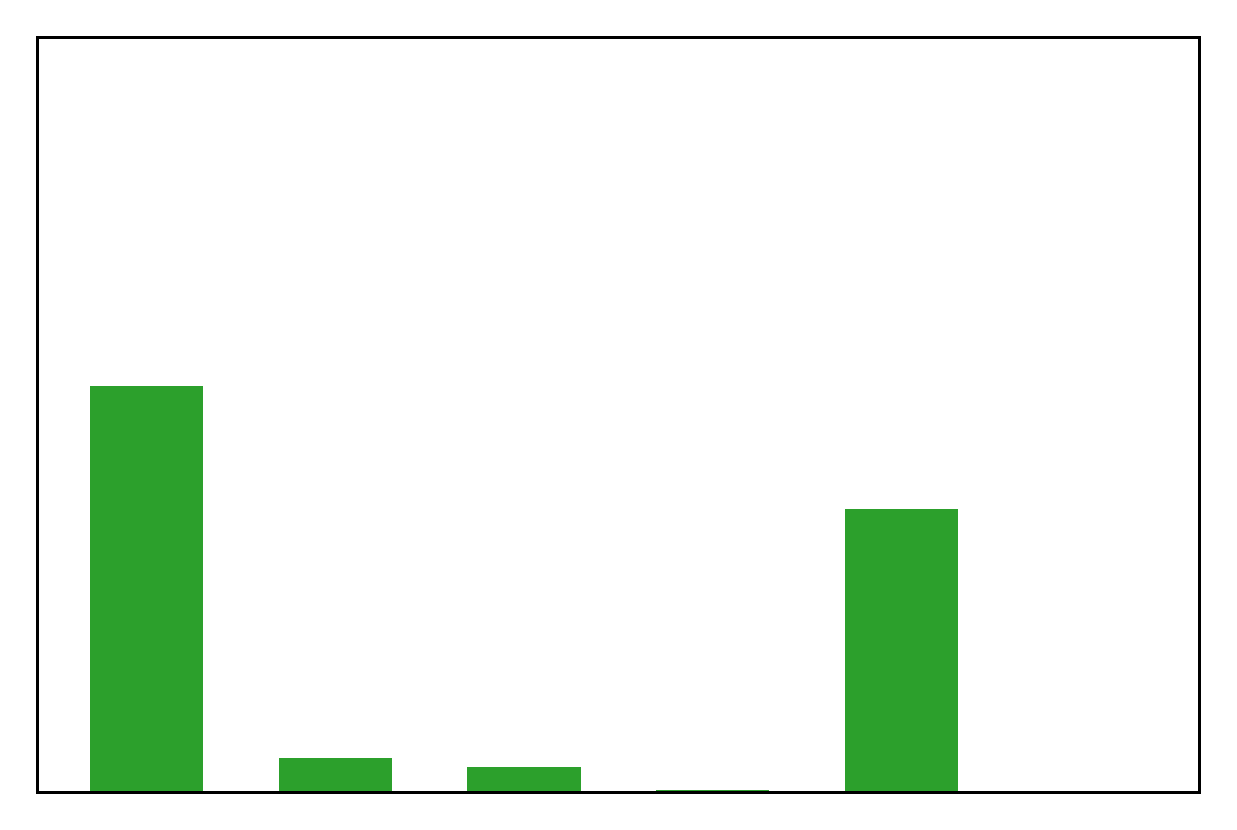}
	\end{subfigure}\\

	\begin{subfigure}[b]{0.2\linewidth}
\begin{overpic}[width=\linewidth]{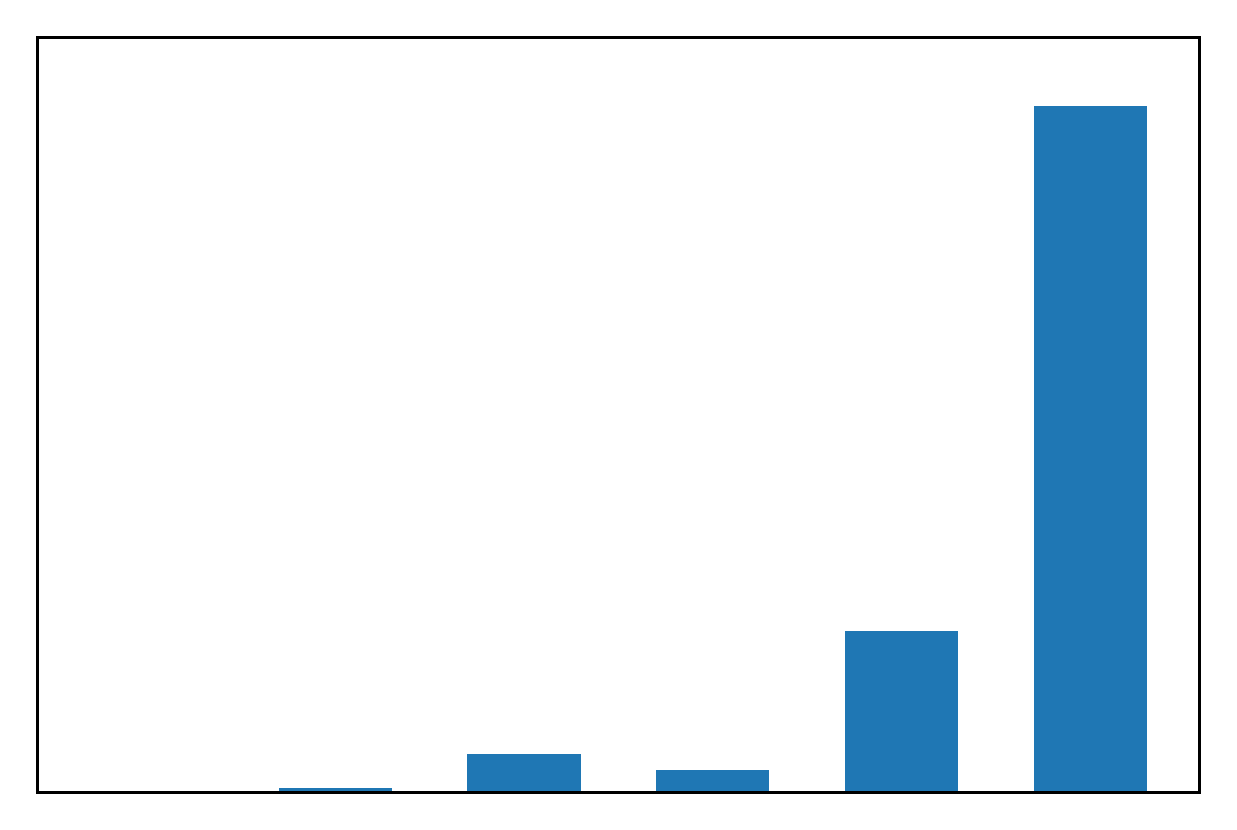}
 \put (35,50) {$\displaystyle M_1(x)$}
\end{overpic}
	\end{subfigure}&
	\begin{subfigure}[b]{0.2\linewidth}
	\includegraphics[width=\linewidth]{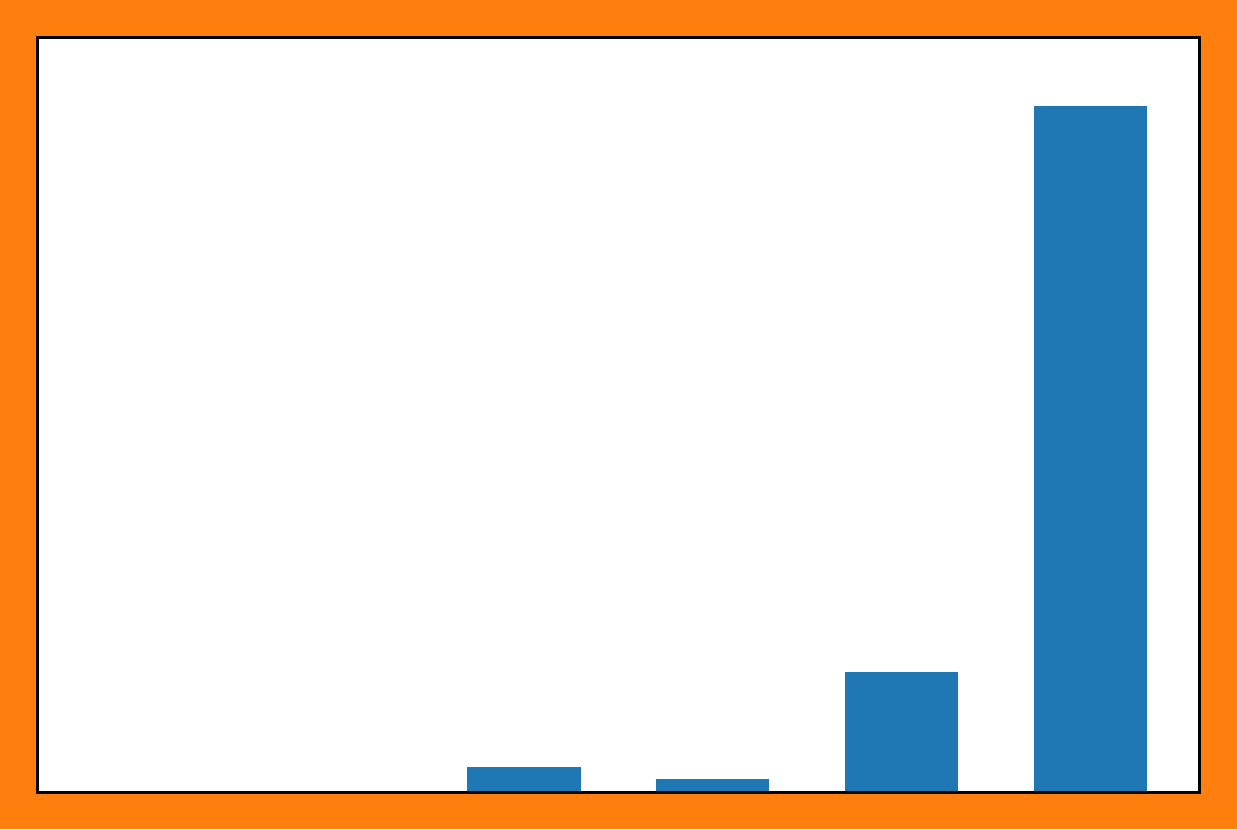}
	\end{subfigure}&
	\begin{subfigure}[b]{0.2\linewidth}
	\includegraphics[width=\linewidth]{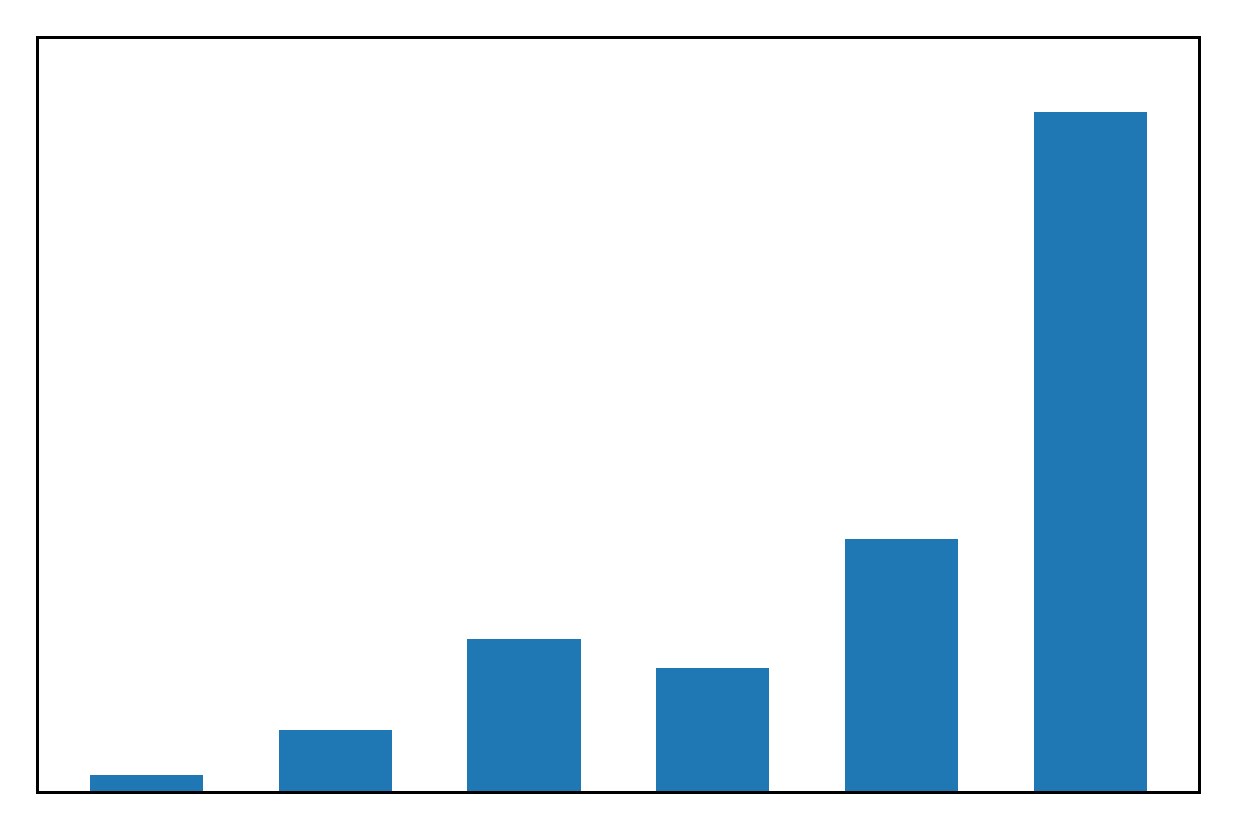}
	\end{subfigure}&
	\begin{subfigure}[b]{0.2\linewidth}
	\includegraphics[width=\linewidth]{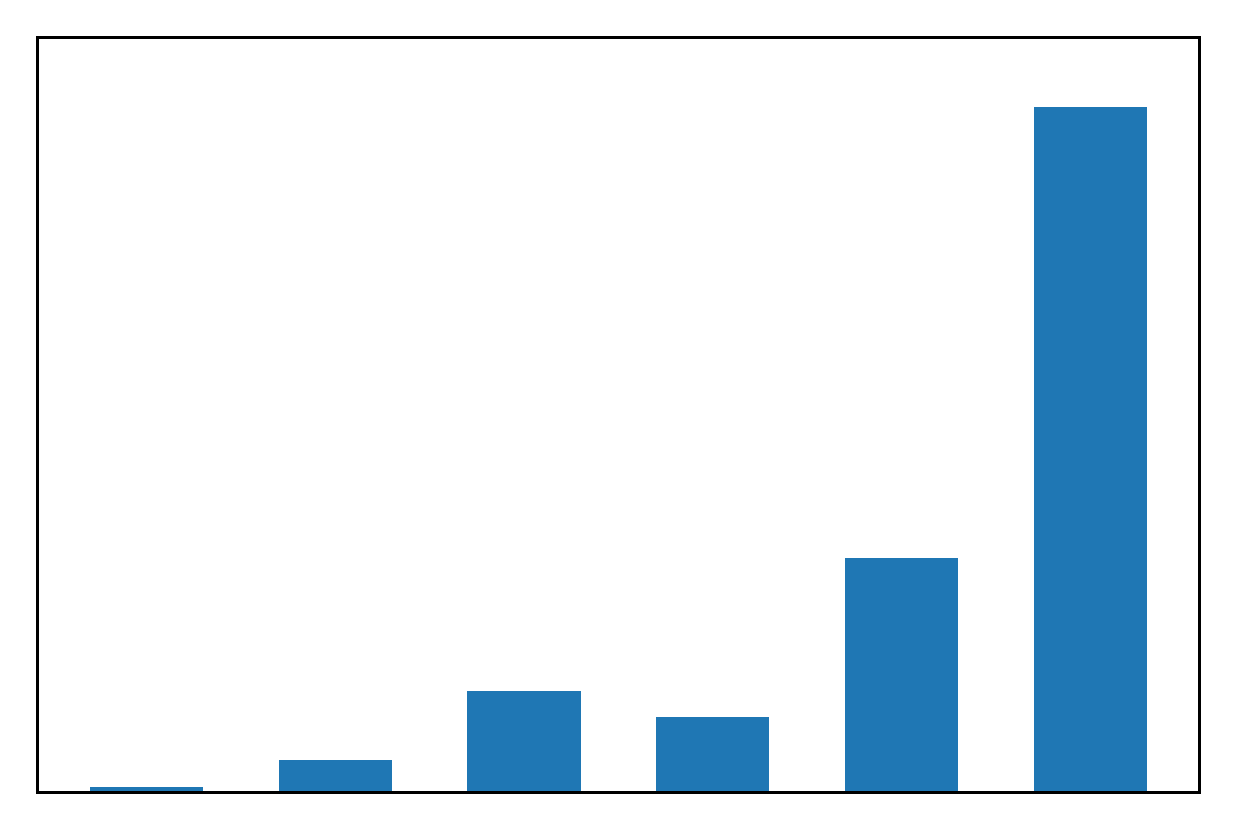}
	\end{subfigure}\\
	\begin{subfigure}[b]{0.2\linewidth}
	\begin{overpic}[width=\linewidth]{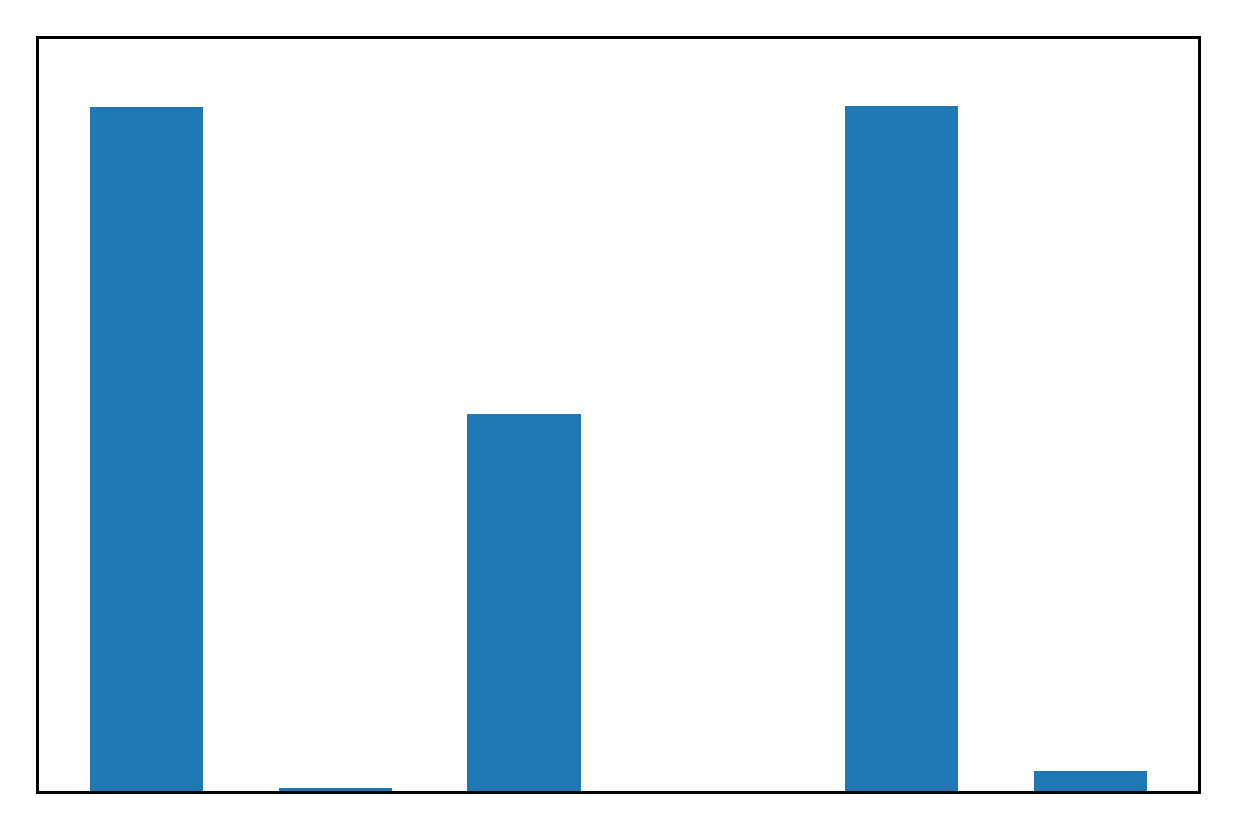}
 \put (35,50) {$\displaystyle M_2(x)$}
\end{overpic}
	\end{subfigure}&
	\begin{subfigure}[b]{0.2\linewidth}
	\includegraphics[width=\linewidth]{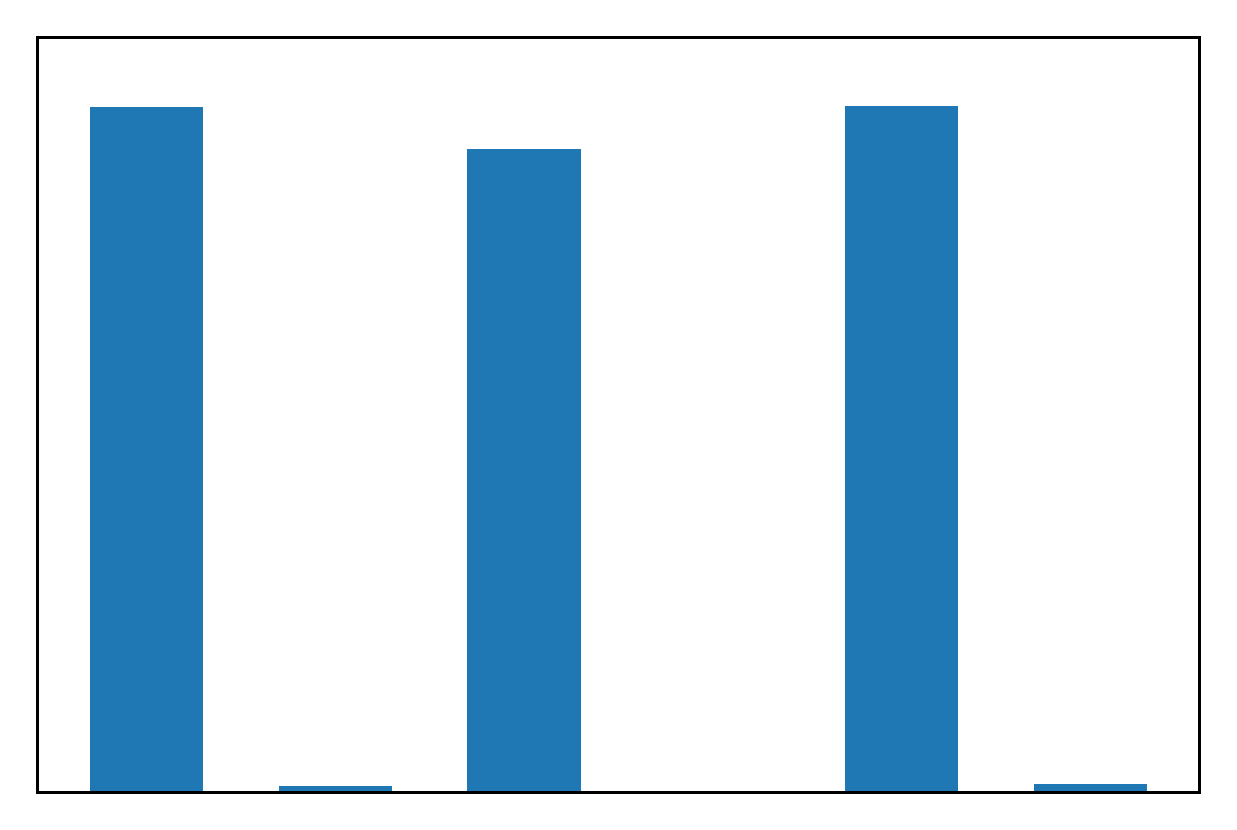}
	\end{subfigure}&
	\begin{subfigure}[b]{0.2\linewidth}
	\includegraphics[width=\linewidth]{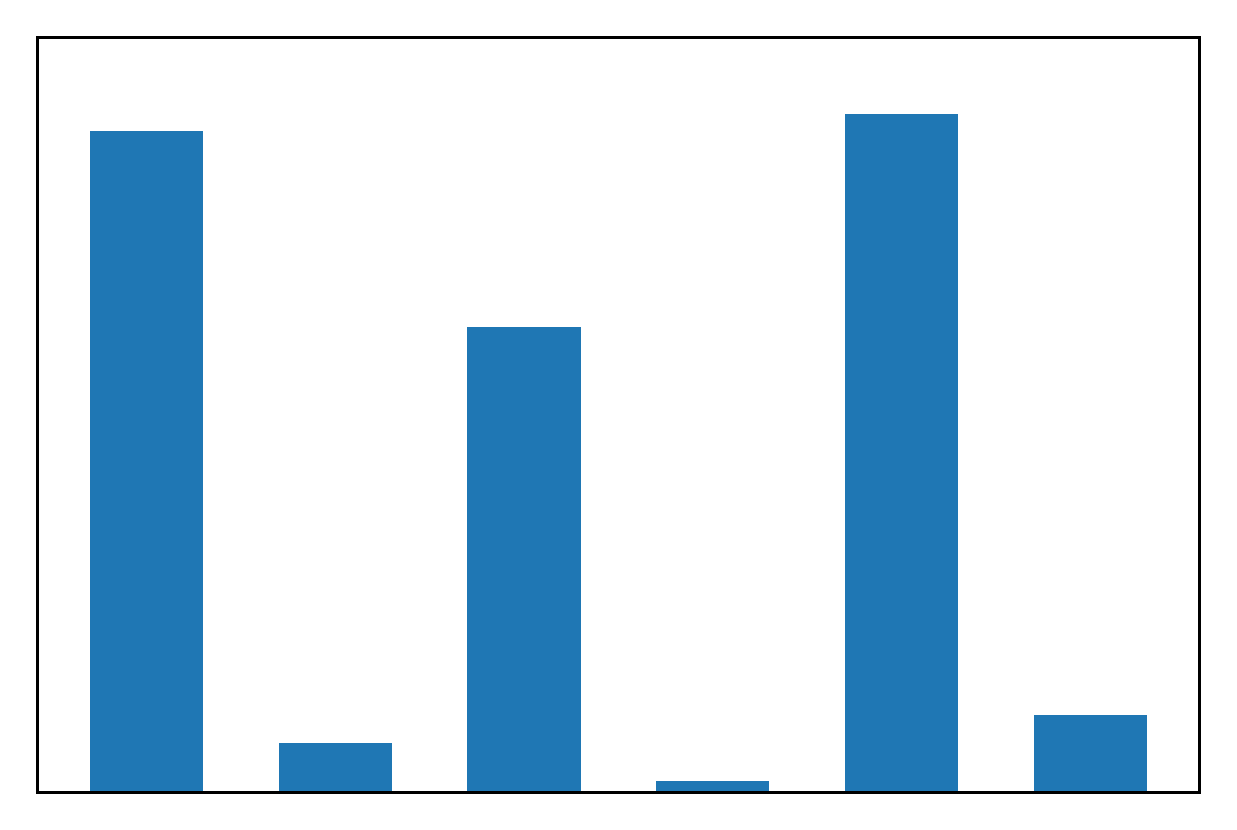}
	\end{subfigure}&
	\begin{subfigure}[b]{0.2\linewidth}
	\includegraphics[width=\linewidth]{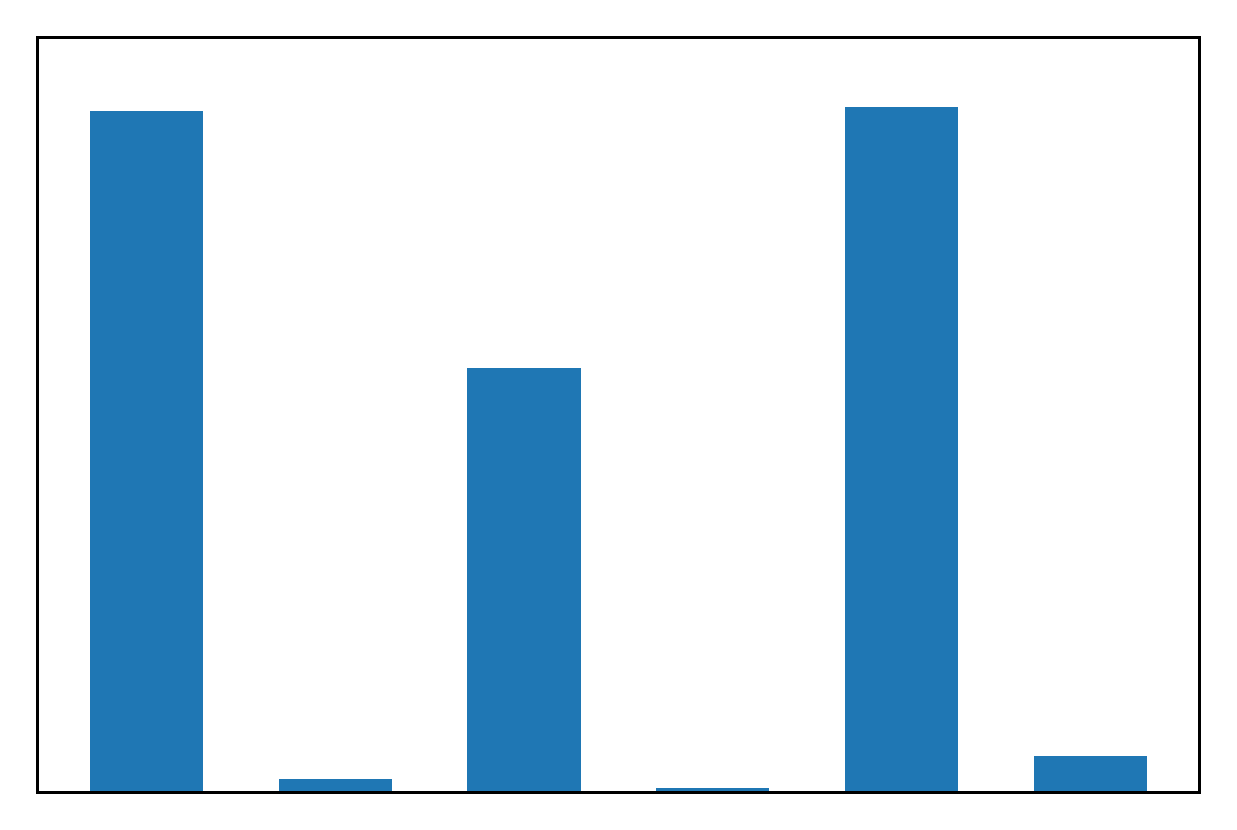}
	\end{subfigure}\\
	\begin{subfigure}[b]{0.2\linewidth}
	\begin{overpic}[width=\linewidth]{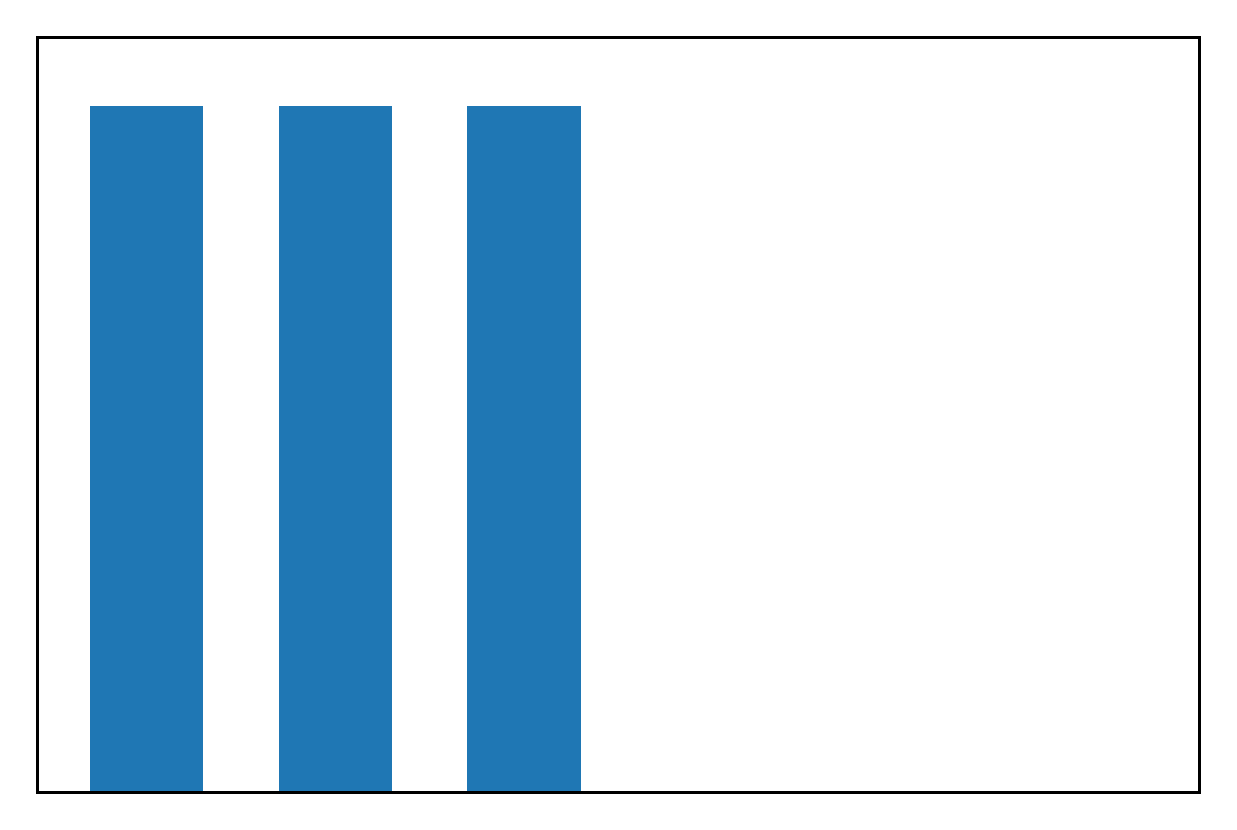}
 \put (65,50) {$\displaystyle M_3(x)$}
\end{overpic}
	\end{subfigure}&
	\begin{subfigure}[b]{0.2\linewidth}
	\includegraphics[width=\linewidth]{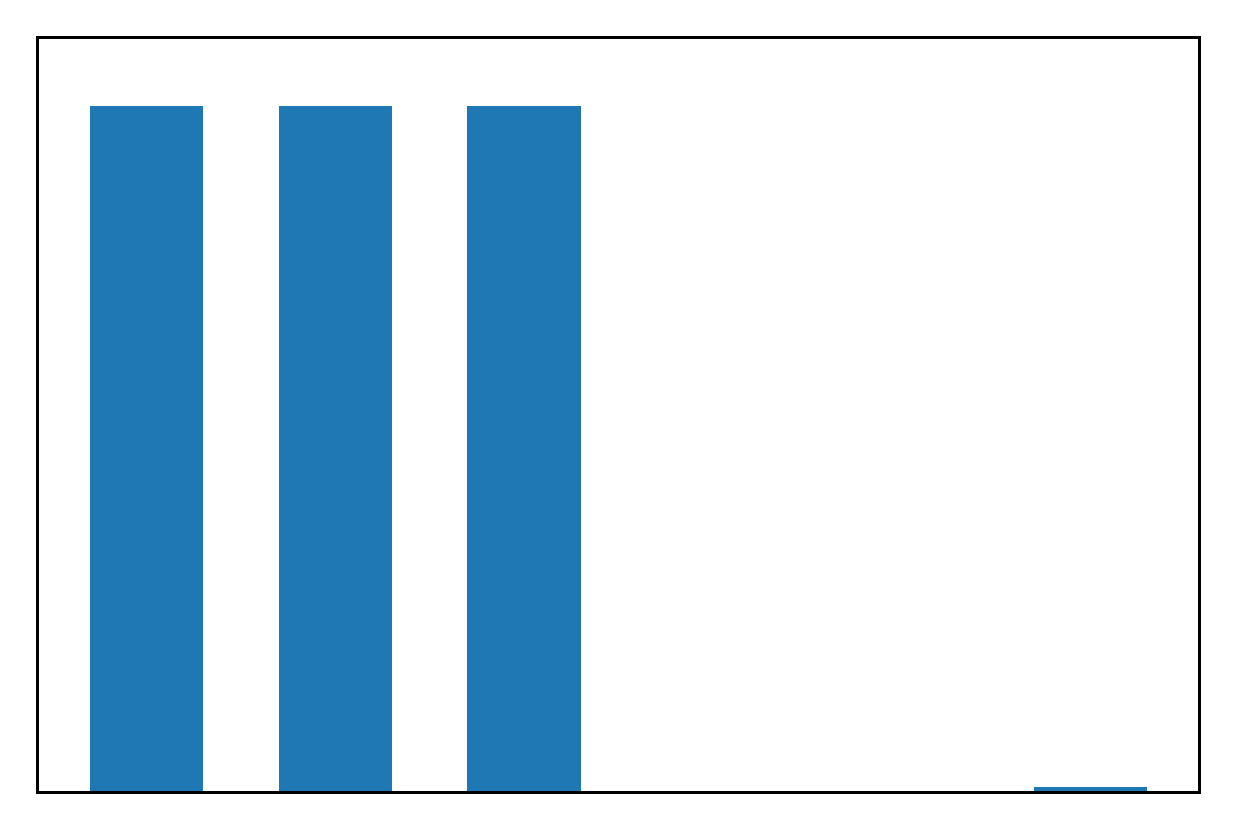}
	\end{subfigure}&
	\begin{subfigure}[b]{0.2\linewidth}
	\includegraphics[width=\linewidth]{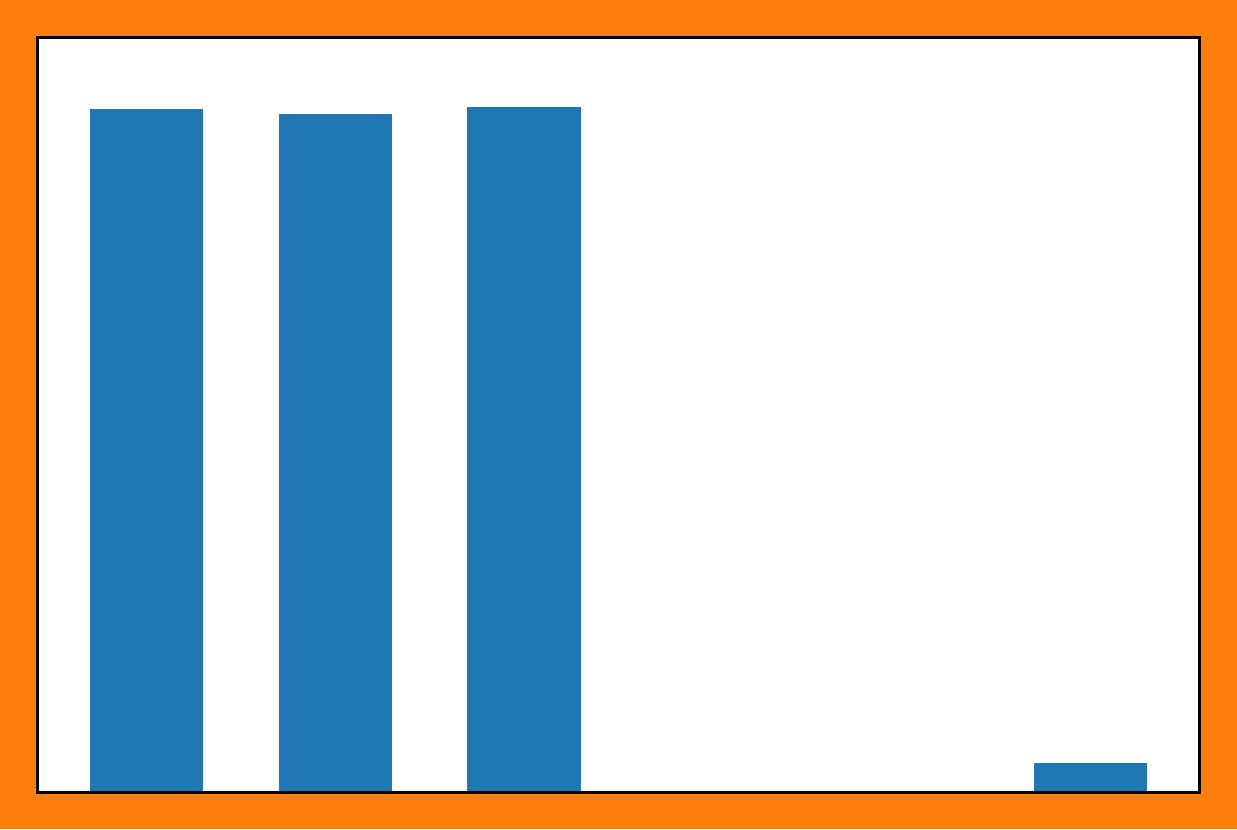}
	\end{subfigure}&
	\begin{subfigure}[b]{0.2\linewidth}
	\includegraphics[width=\linewidth]{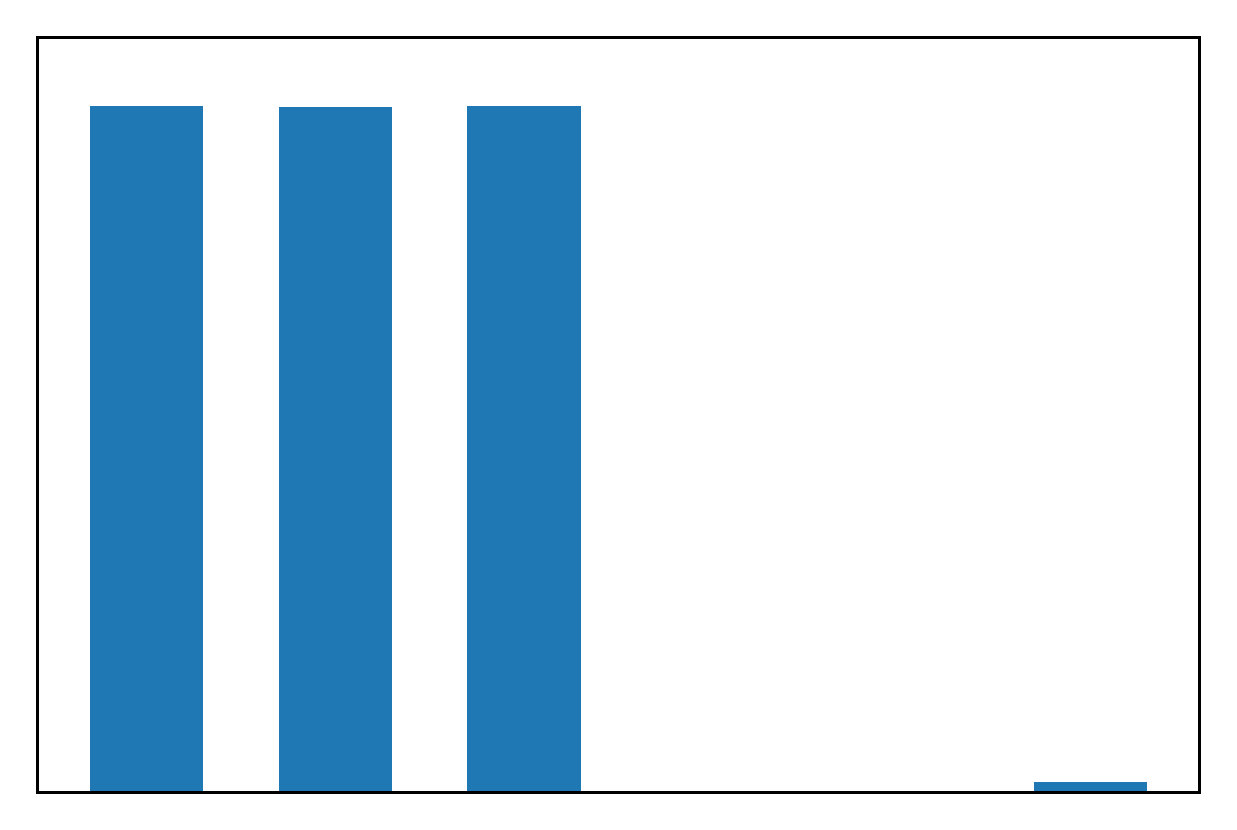}
	\end{subfigure}\\
		\begin{subfigure}[b]{0.2\linewidth}
		\begin{overpic}[width=\linewidth]{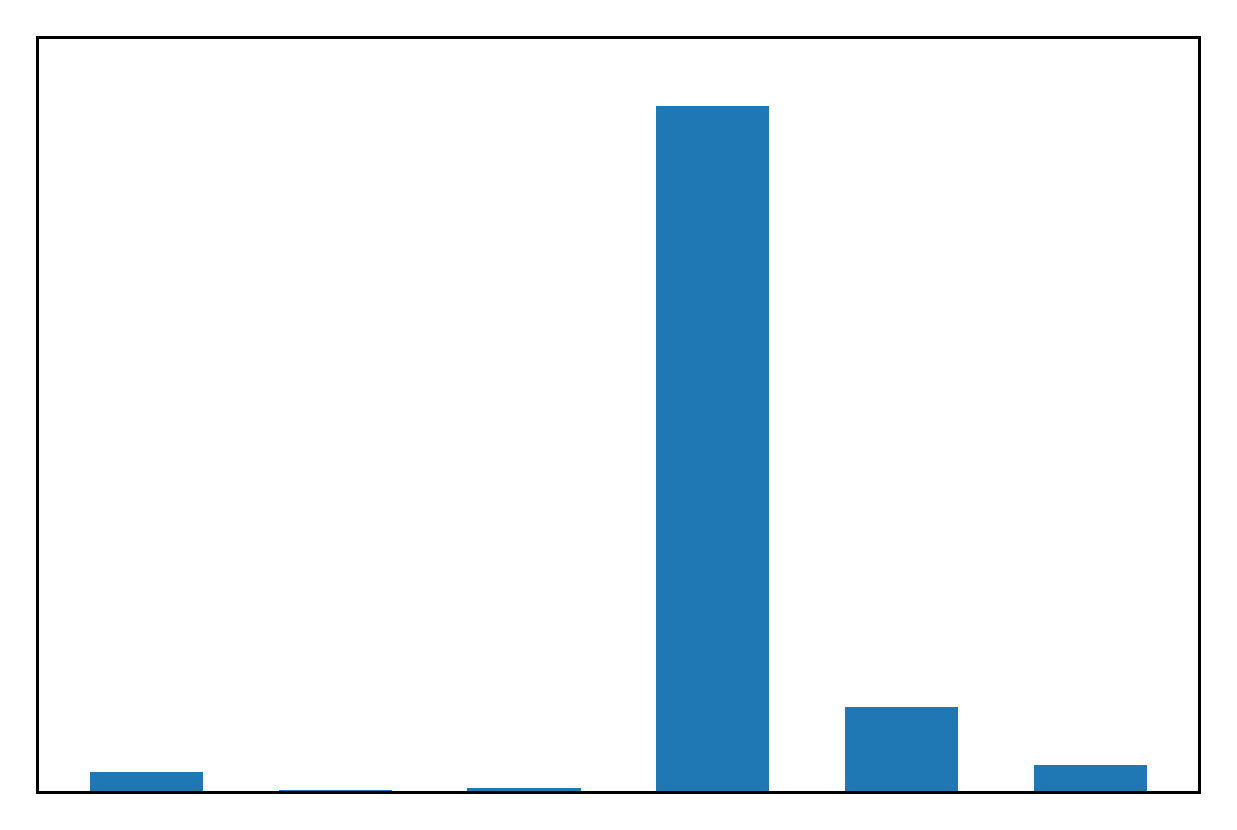}
 \put (65,50) {$\displaystyle M_4(x)$}
\end{overpic}
	\end{subfigure}&
	\begin{subfigure}[b]{0.2\linewidth}
	\includegraphics[width=\linewidth]{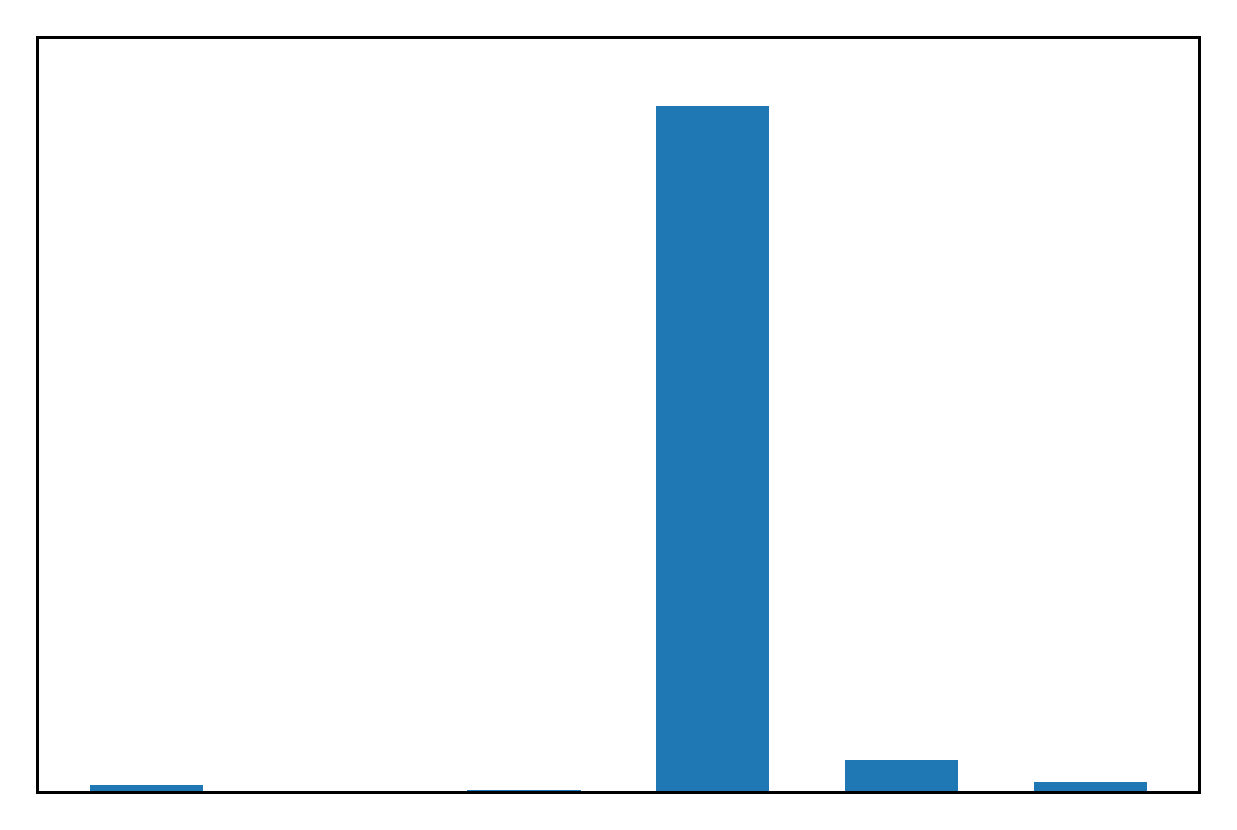}
	\end{subfigure}&
	\begin{subfigure}[b]{0.2\linewidth}
	\includegraphics[width=\linewidth]{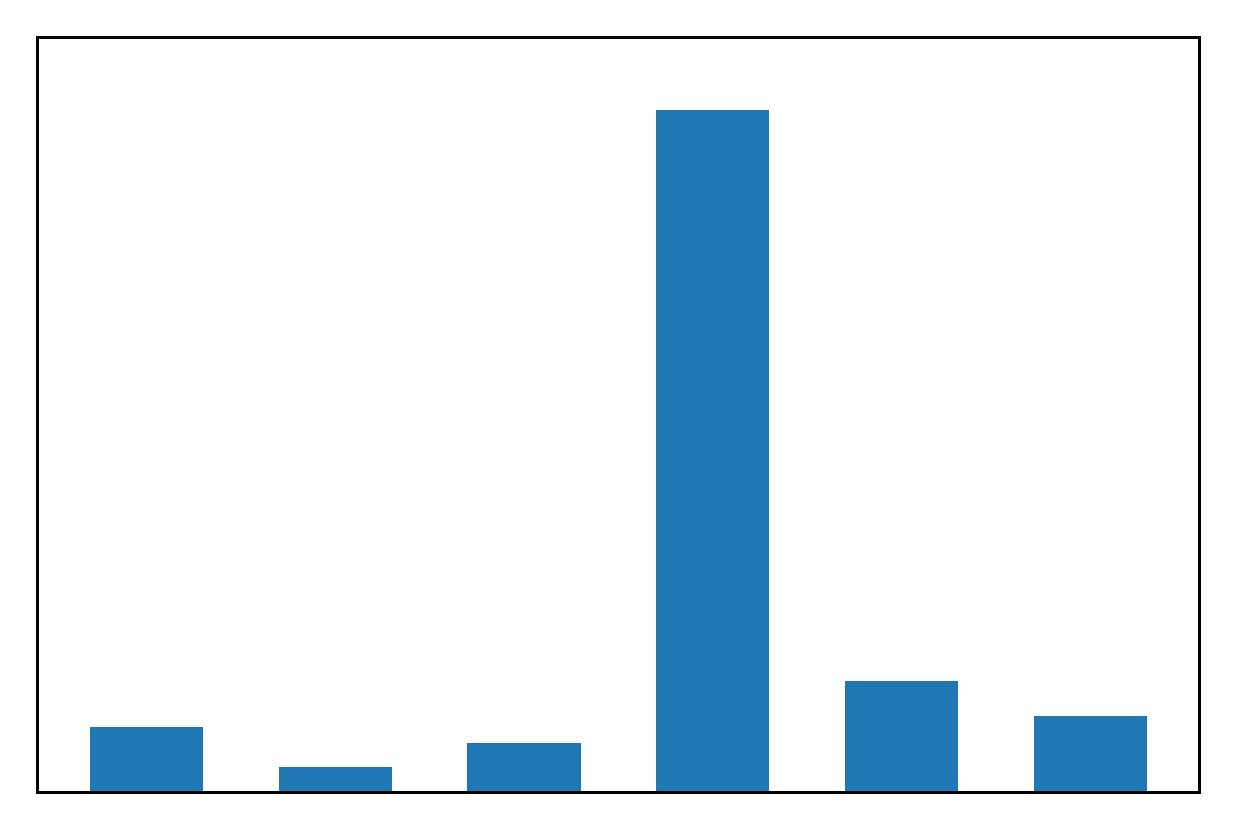}
	\end{subfigure}&
	\begin{subfigure}[b]{0.2\linewidth}
	\includegraphics[width=\linewidth]{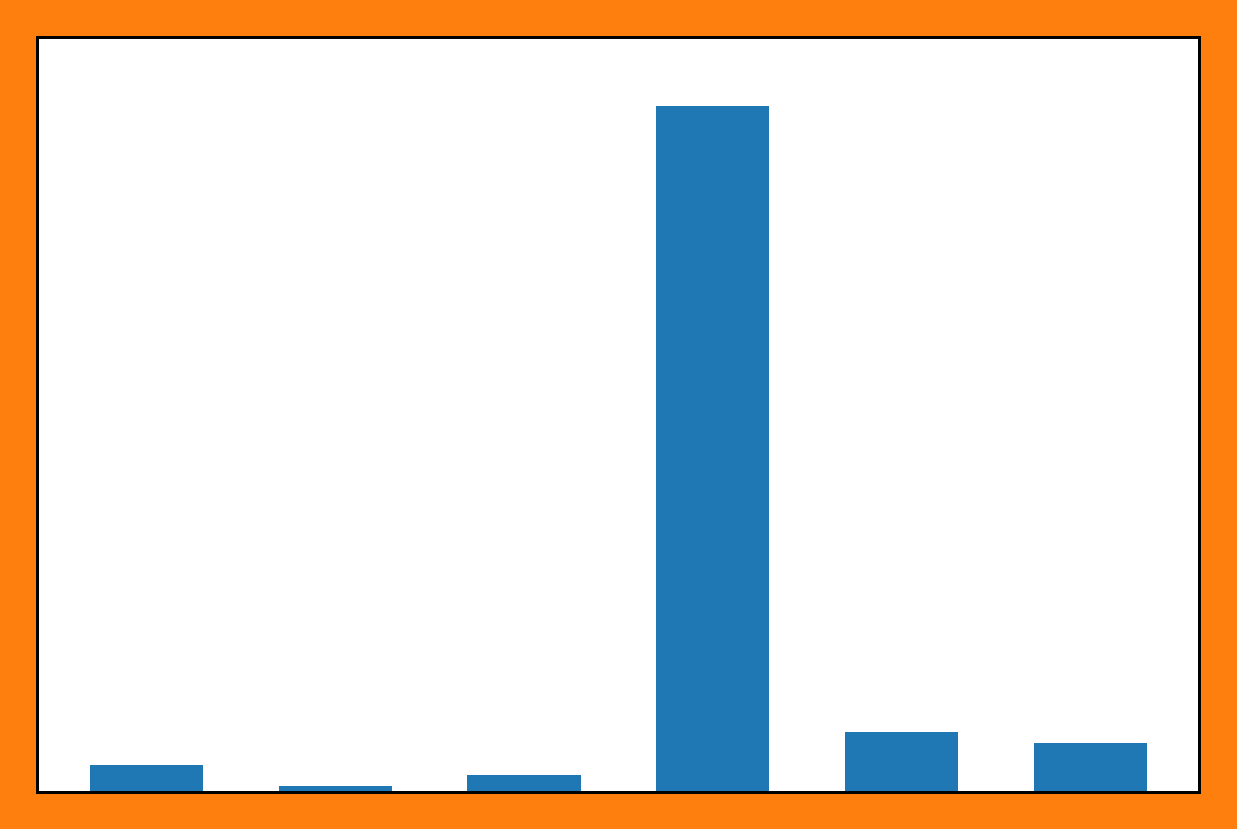}
	\end{subfigure}\\
		\begin{subfigure}[b]{0.2\linewidth}
\begin{overpic}[width=\linewidth]{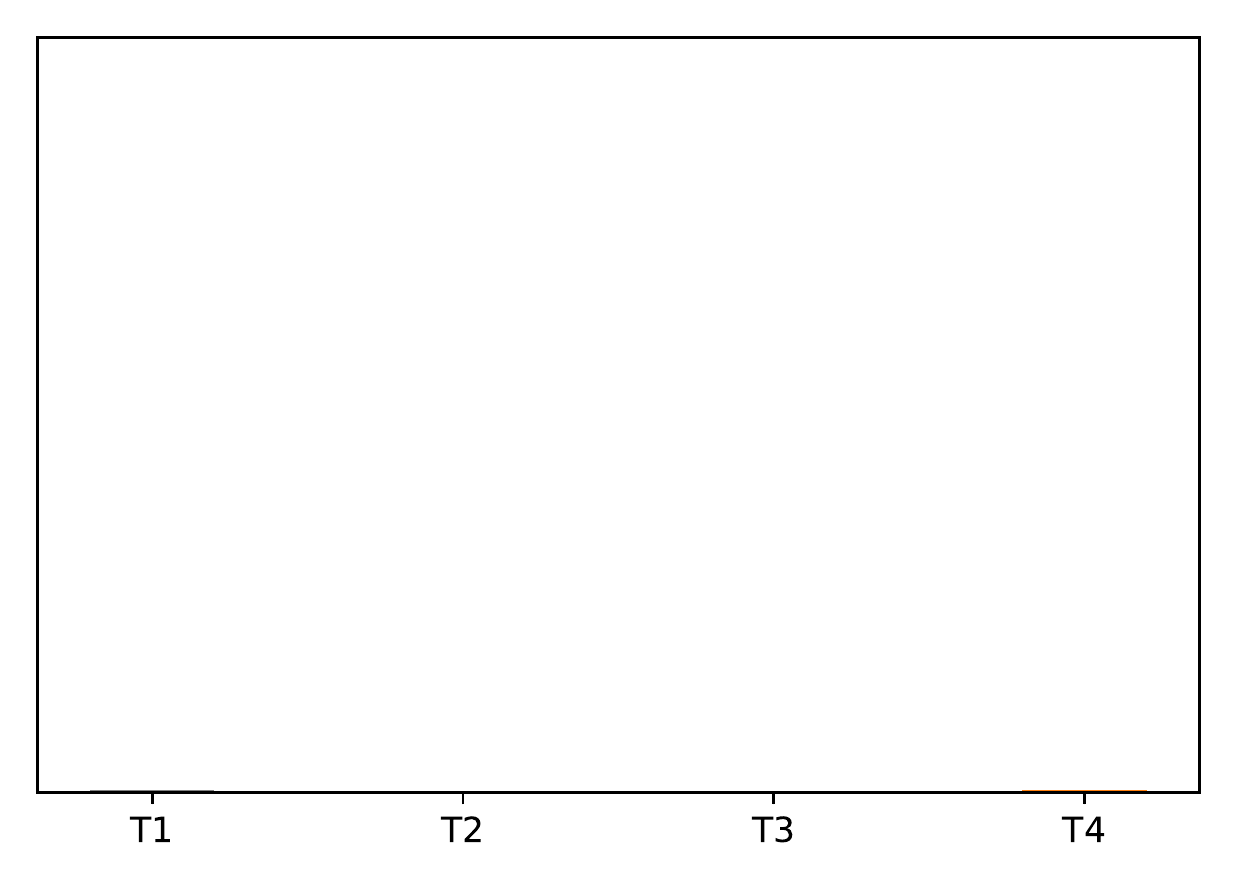}
 \put (38,50) {$\displaystyle S(x)$}
\end{overpic}
	\end{subfigure}&
	\begin{subfigure}[b]{0.2\linewidth}
	\includegraphics[width=\linewidth]{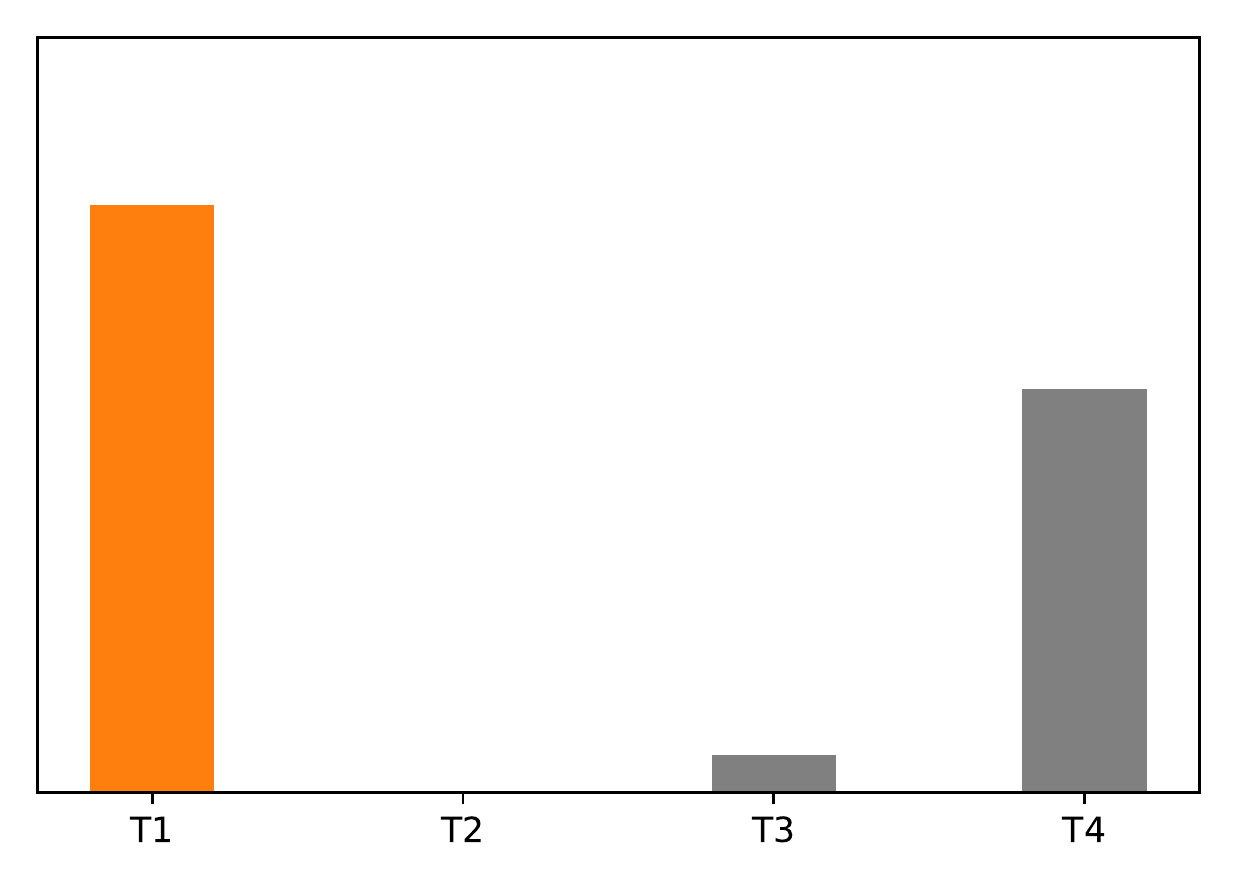}
	\end{subfigure}&
	\begin{subfigure}[b]{0.2\linewidth}
	\includegraphics[width=\linewidth]{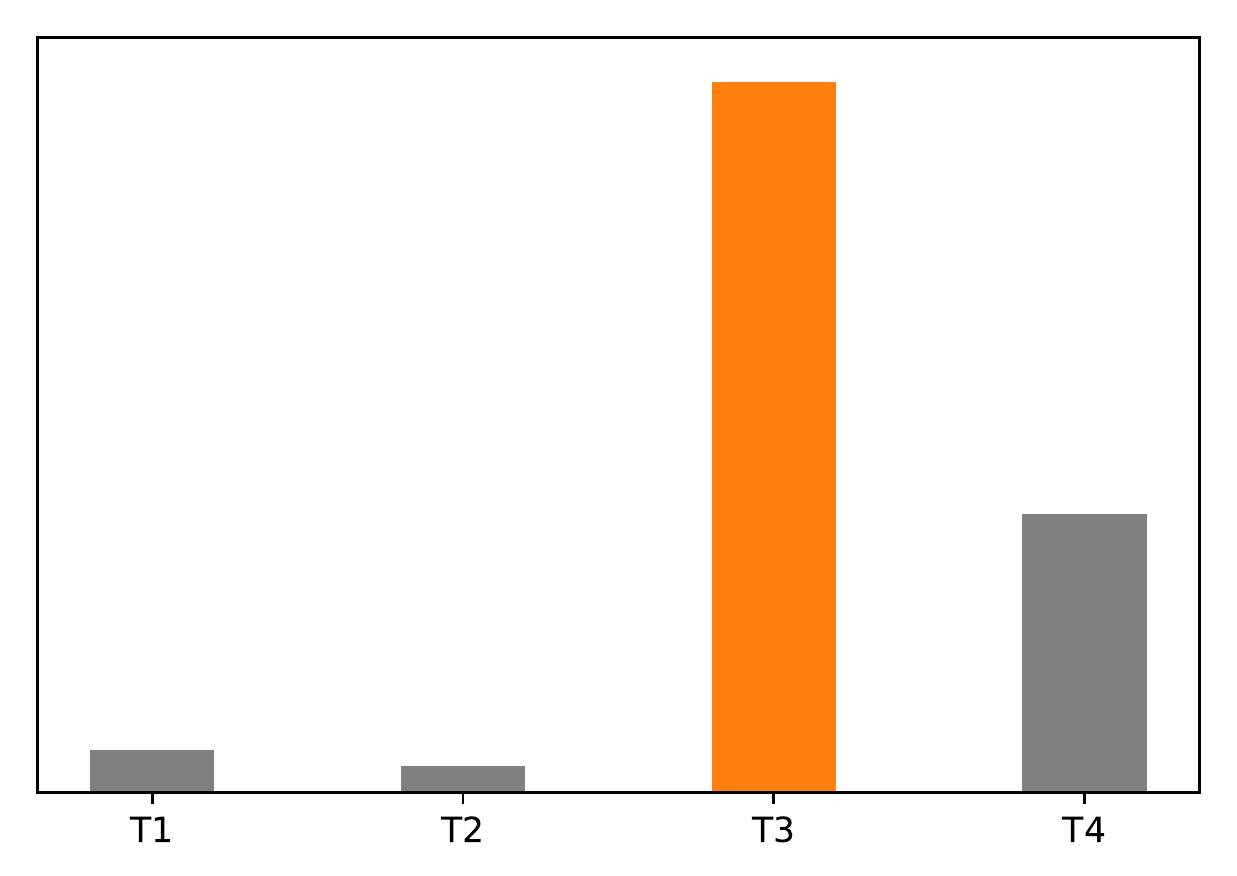}
	\end{subfigure}&
	\begin{subfigure}[b]{0.2\linewidth}
	\includegraphics[width=\linewidth]{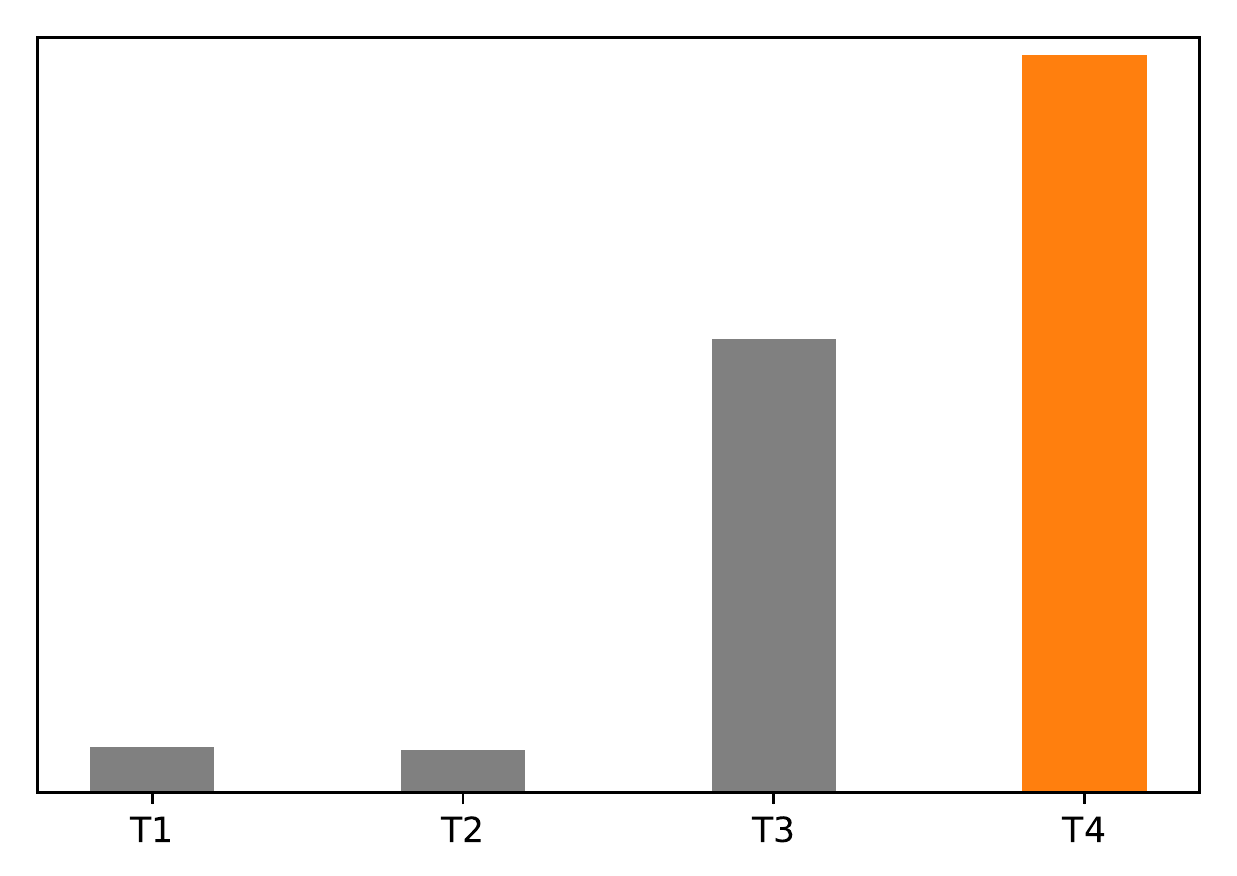}
	\end{subfigure}
	\end{tabular}
	\caption{Visualizations of thyroid. The first row: one normal example and three abnormal examples. The second to the fourth rows: the learned four masks of them. The fifth row: the scores contributed by each transformation of them, where the highest term is colored by orange. The plots on each row are under the same scale.}
	\label{fig:tab_Ts}
\end{figure}

%% file: figures/score_simplex.tex
\begin{figure}[ht]
	\centering
	\begin{subfigure}[b]{0.4\linewidth}
		\includegraphics[width=\linewidth]{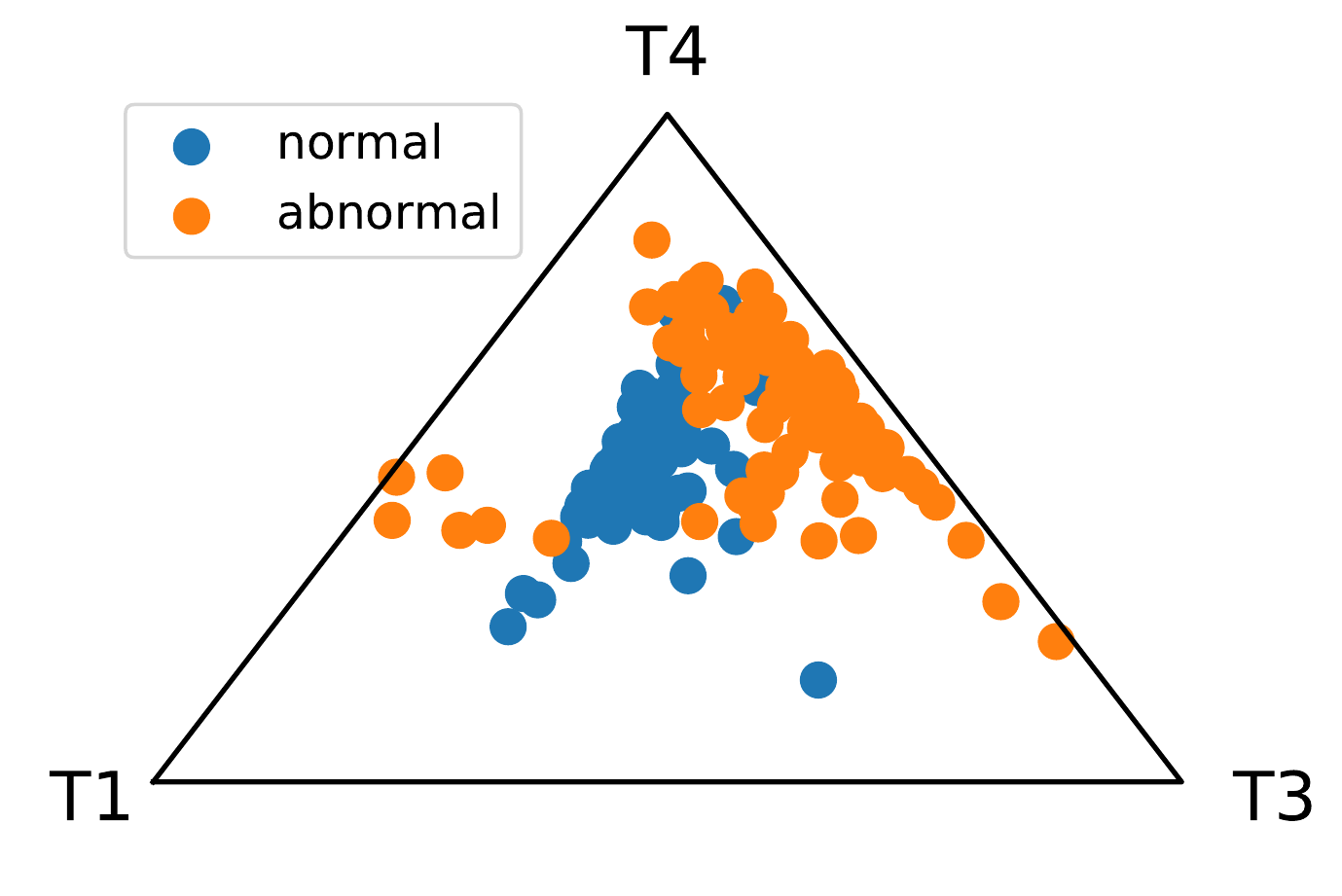}

	\end{subfigure}
	\begin{subfigure}[b]{0.4\linewidth}
		\includegraphics[width=\linewidth]{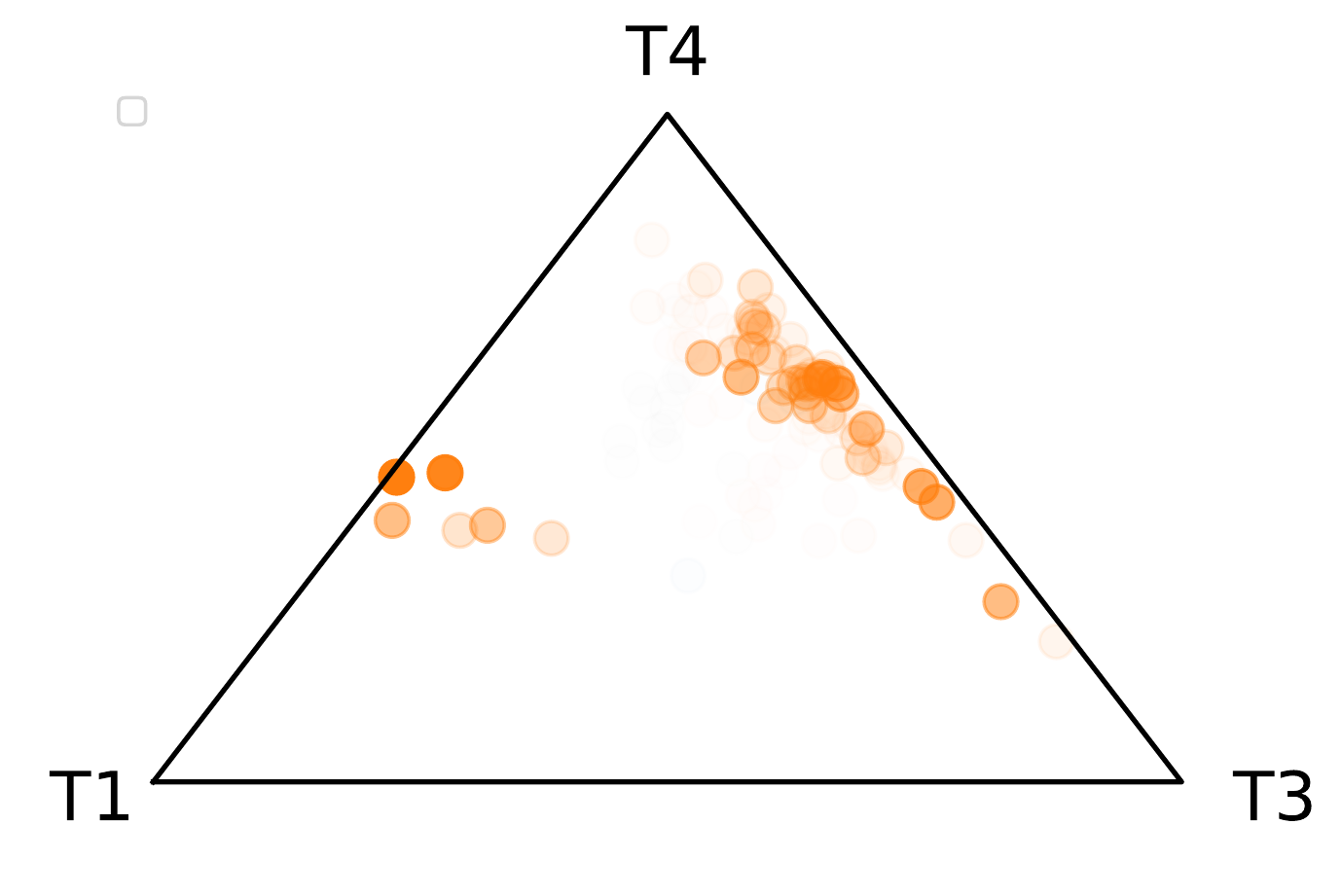}

	\end{subfigure}
	\caption{Visualizations of thyroid. We project the scores contributed by $T_1$, $T_3$, and $T_4$ to a simplex. The subplot on the left visualizes which transformation dominates the score. The subplot on the right visualizes the scores via transparency.}
	\label{fig:tab_score}
\end{figure}